\PassOptionsToPackage{square,comma,numbers,sort&compress}{natbib}
\documentclass{article}

\usepackage[preprint]{neurips_2021}

\usepackage[utf8]{inputenc} % allow utf-8 input
\usepackage[T1]{fontenc}    % use 8-bit T1 fonts
\usepackage[colorlinks,bookmarks=false]{hyperref}%pagebackref=true,
\usepackage{url}            % simple URL typesetting
\usepackage{booktabs}       % professional-quality tables
\usepackage{amsfonts}       % blackboard math symbols
\usepackage{nicefrac}       % compact symbols for 1/2, etc.
\usepackage{microtype}      % microtypography
\usepackage{wrapfig}
\usepackage{caption}
\usepackage{subcaption}

\usepackage{amsfonts}       % blackboard math symbols
\usepackage{algorithm}
\usepackage{algorithmic}
\usepackage{amsmath}
\usepackage{graphicx}
\usepackage{multirow}
\usepackage{bm}
\usepackage{color}
\usepackage{mcode}
\usepackage{mcodescri}
\usepackage{amsthm}
\usepackage{tabularx}
\usepackage{wrapfig}
\usepackage{footnote}
\usepackage{threeparttable}
\usepackage{tablefootnote}
\makesavenoteenv{tabular}
\makesavenoteenv{table}
\usepackage{caption}
\usepackage{fancyhdr}
\usepackage{amssymb}% http://ctan.org/pkg/amssymb
\usepackage{pifont}
\usepackage{xr}
\usepackage{wrapfig}
\usepackage{setspace}
\usepackage{xcolor}         % colors

 \usepackage{exscale}
 \usepackage{relsize}
 \usepackage{xr}

\newcommand{\comment}[1]{}

\newcommand{\et}{\emph{et al.}}
\newcommand{\eg}{\emph{e.g.}}

\newtheorem{thm}{Theorem}%[section]
\newtheorem{lem}{Lemma}
\newtheorem{defn}{Definition}

\newtheorem{assum}{Assumption}
\DeclareMathOperator*{\argmin}{argmin}

\newcommand{\led}[1]{\overset{\text{\ding{#1}}}{\leq}}

\newcommand{\lee}[1]{\overset{\text{\ding{#1}}}{=}}

\newcommand{\Rs}[1]{{\mathbb{R}^{#1}}}

\newcommand{\order}[1]{{\cal{O}}\left(#1\right)}
\newcommand{\Lt}[1]{\mathcal{L}_{#1}}
\newcommand{\Ltt}[1]{\widetilde{\mathcal{L}}_{#1}}

\newcommand{\alphamax}{\alphai{\mbox{\scriptsize{max}}}}

\newcommand{\Linfo}{\mathcal{L}_{\mbox{\scriptsize{n}}}}
\newcommand{\Lcon}{\mathcal{L}_{\mbox{\scriptsize{c}}}}

\newcommand{\nn}{\nonumber}
\newcommand{\betad}{\textsf{\footnotesize{Beta}}}
\newcommand{\xmfi}[1]{\bm{x}_{#1}'}
\newcommand{\ymfi}[1]{\bm{y}_{#1}'}

\newcommand{\bmbi}[1]{\bar{\bm{b}}_{#1}}
\newcommand{\yms}{\bm{y}^{*}}
\newcommand{\ymsi}[1]{\bm{y}_{#1}^{*}}
\newcommand{\yi}[1]{y_{#1}}

\newcommand{\Bmi}[1]{\bm{B}_{#1}}
\newcommand{\Bmb}{\bar{\bm{B}}}

\newcommand{\norm}[1]{\|{#1}\|_{2}}
\newcommand{\normin}[1]{\|{#1}\|_{\infty}}
\newcommand{\normf}[1]{\|{#1}\|_{F}}
\newcommand{\tin}[1]{\|{#1}\|_{\infty}}

\newcommand{\alphab}{\widetilde{\alpha}}

\newcommand{\bp}{\beta}
\newcommand{\bn}{\alpha}
\newcommand{\be}{\epsilon}

\newcommand{\pa}{{\partial}}
\newcommand{\distas}{\overset{\text{i.i.d.}}{\sim}}
\newcommand{\Sigm}{\bm{\Sigma}}

\newcommand{\Scp}{\mathcal{S}_+}
\newcommand{\Scn}{\mathcal{S}_-}
\newcommand{\wmi}[1]{\bm{w}_{#1}}

\newcommand{\eps}{\varepsilon}

\newcommand{\alphai}[1]{\alpha_{#1}}
\newcommand{\betai}[1]{\beta_{#1}}

\newcommand{\opnorm}[1]{\left\|#1\right\|}
\newcommand{\fronorm}[1]{\left\|#1\right\|_{F}}

\newcommand{\twonorm}[1]{\left\|#1\right\|_{\ell_2}}

\newcommand{\abs}[1]{\left|#1\right|}

\newcommand{\mtx}[1]{\bm{#1}}
\newcommand{\cb}{\mtx{c}}

\newcommand{\gradt}[2]{{\nabla\Lt{#1}(#2)}}

\newcommand{\xmi}[1]{\bm{x}_{#1}}

\newcommand{\xmti}[1]{\widetilde{\bm{x}}_{#1}}
\newcommand{\ymti}[1]{\widetilde{\bm{y}}_{#1}}
\newcommand{\ymt}{\widetilde{\bm{y}}}

\newcommand{\ymgi}[1]{\bm{y}_{#1}}
\newcommand{\ymai}[1]{\widetilde{\bm{y}}_{#1}}

\newcommand{\bmi}[1]{\bm{b}_{#1}}

\newcommand{\cmi}[1]{\bm{c}_{#1}}
\newcommand{\pmi}[1]{\bm{p}_{#1}}
\newcommand{\pmii}[2]{\bm{p}^{#1}_{#2}}

\newcommand{\qmi}[1]{\bm{q}_{#1}}
\newcommand{\qmii}[2]{\bm{q}^{#1}_{#2}}

\newcommand{\vmi}[1]{\bm{v}_{#1}}
\newcommand{\rmi}[1]{\bm{r}_{#1}}

\newcommand{\rmti}[1]{\widetilde{\bm{r}}_{#1}}
\newcommand{\rmbi}[1]{\widehat{\bm{r}}_{#1}}
\newcommand{\embi}[1]{\widehat{\bm{e}}_{#1}}
\newcommand{\emi}[1]{\bm{e}_{#1}}

\newcommand{\ymi}[1]{\bm{y}_{#1}}
\newcommand{\ymbi}[1]{\widehat{\bm{y}}_{#1}}
\newcommand{\ymbii}[2]{\bar{\bm{y}}^{#1}_{#2}}
\newcommand{\ymwii}[2]{\widetilde{\bm{y}}^{#1}_{#2}}

\newcommand{\LL}{\mathcal{L}}
\newcommand{\J}{\mathcal{J}}
\newcommand{\Ji}[1]{\mathcal{J}_{#1}}
\newcommand{\Jti}[1]{\widetilde{\mathcal{J}}_{#1}}

 \newcommand{\cm}{\bm{c}}

\newcommand{\emm}{\bm{e}}

\newcommand{\rmm}{\bm{r}}
\newcommand{\rmw}[1]{\widetilde{\bm{r}}}
\newcommand{\rmb}[1]{\widehat{\bm{r}}}
\newcommand{\wm}{\bm{w}}
\newcommand{\xm}{\bm{x}}
\newcommand{\xmt}[1]{\widetilde{\bm{x}}}
\newcommand{\ym}{\bm{y}}
\newcommand{\ymg}{\bm{y}}
\newcommand{\yma}{\widetilde{\bm{y}}}

\newcommand{\ymw}{\widetilde{\bm{y}}}

\newcommand{\um}{\bm{u}}
\newcommand{\vm}{\bm{v}}

\newcommand{\Bm}{\bm{B}}
\newcommand{\Cm}{\bm{C}}

\newcommand{\Imm}{\bm{I}}
\newcommand{\Wm}{\bm{W}}
\newcommand{\Wmt}{\widetilde{\bm{W}}}
\newcommand{\Gm}{\bm{G}}
\newcommand{\Xm}{\bm{X}}
\newcommand{\Xmt}{\widetilde{\bm{X}}}

\newcommand{\Um}{\bm{U}}

\newcommand{\Wmi}[1]{\bm{W}_{#1}}

\newcommand{\Wmti}[1]{\widetilde{\bm{W}}_{#1}}
\newcommand{\Xmti}[1]{\widetilde{\bm{X}}_{#1}}

\newcommand{\Oc}[1]{\mathcal{O}\left(#1\right)}

\newcommand{\R}{\mathbb{R}}
 
\newcommand{\W}{\mathcal{W}}

\newcommand{\EE}{\mathbb{E}}

\newcommand{\Pro}{\mathcal{P}}

\newcommand{\Ly}{L_y}

\newcommand{\xim}{\bm{\xi}}
\newcommand{\ft}{g}
\newcommand{\fo}{f}

\newcommand{\N}{\bm{\mathcal{N}}}

\newcommand{\Qge}{\bm{\mathcal{Q}}_{e}}
\newcommand{\Qgp}{\bm{\mathcal{Q}}}
\newcommand{\Qa}{\widetilde{\bm{\mathcal{Q}}}}
\newcommand{\X}{\bm{\mathcal{X}}}

\newcommand{\Dg}{\bm{\mathcal{D}}}
\newcommand{\Da}{\widetilde{\bm{\mathcal{D}}}}

\newcommand{\D}{\bm{\mathcal{D}}}

\newcommand{\SSS}{\bm{\mathcal{S}}}
\newcommand{\F}{\mathcal{F}}
\newcommand{\Y}{\bm{\mathcal{Y}}}

\newtheoremstyle{corstyle}
{3.2pt} % Space above
{0.2pt} % Space below
{} % Body font
{} % Indent amount
{\bfseries} % Theorem head font
{.} % Punctuation after theorem head
{.3em} % Space after theorem head
{} % Theorem head spec (can be left empty, meaning `normal')

\theoremstyle{corstyle} 
\DeclareMathSizes{10}{9}{7.5}{6.5}  
\title{A Theory-Driven Self-Labeling Refinement Method for Contrastive Representation Learning}
\author{ Pan Zhou$^{*}$,  \quad   \normalsize{Caiming Xiong$^{*}$},  \quad  \normalsize{Xiao-Tong Yuan$^{\dagger}$}, \quad   \normalsize{Steven  Hoi$^{*}$}\\
	{$^{*}$ Salesforce Research \qquad  $^{\dagger}$ Nanjing University of Information
		Science $\&$ Technology} \\ 
	{ \texttt panzhou3@gmail.com}    \ \   {\texttt\tt \{cxiong,  shoi\}@salesforce.com}  \ \  {\texttt \tt xtyuan@nuist.edu.cn} 
}
 
\begin{document}

\maketitle

\begin{abstract}
For an image  query, unsupervised contrastive learning  labels crops of  the same image as positives,  and other image crops as  negatives. Although intuitive, such a native label assignment strategy cannot reveal the underlying semantic similarity between a  query and  its positives and negatives, and impairs performance,  since some negatives are  semantically similar to  the query or even share the same semantic class as the query.  In this work, we first  prove that for  contrastive learning,  inaccurate label assignment heavily  impairs its generalization for semantic instance discrimination, while accurate labels  benefit its generalization.  Inspired by this theory, we  propose   a novel self-labeling refinement approach for contrastive learning. It improves the label quality via two complementary  modules:  (i)  self-labeling refinery (SLR) to  generate accurate labels and (ii)  momentum mixup (MM)  to enhance similarity between query and its positive. SLR uses a positive of a query to estimate  semantic similarity between  a query and its positive and negatives, and  combines estimated similarity with  vanilla label assignment in contrastive learning to  iteratively generate  more accurate and informative soft labels. We theoretically show that our SLR can exactly recover the true semantic  labels of  label-corrupted  data, and  supervises   networks to achieve zero prediction  error on classification tasks.  MM randomly  combines   queries and  positives to increase  semantic similarity between the generated virtual queries and their positives so as to improves label accuracy.  Experimental results on CIFAR10,  ImageNet, VOC and COCO show the effectiveness of our method. 
 Code  will be released online. 
\end{abstract}

\vspace{-1.0em}
\section{Introduction}\label{introduction}
\vspace{-0.5em}
Self-supervised learning (SSL)  
 is an effective approach to  learn features without  manual annotations, with great  success witnessed to many downstream tasks, e.g.  image classification and object detection~\cite{he2020momentum,chen2020improved,chen2020big,chen2020simple,caron2020unsupervised,grill2020bootstrap,li2020prototypical}.  The methodology of SSL is to construct a  pretext task that can obtain data labels via well designing the task  itself,  and then build a network to  learn from these tasks. For instance, by constructing  jigsaw puzzle~\cite{noroozi2016unsupervised},   spatial arrangement identification~\cite{doersch2017multi}, orientation~\cite{komodakis2018unsupervised}, or chromatic channels~\cite{zhang2016colorful} as a pretext task, SSL  learns high-qualified features  from the pretext task  that can  well transfer to downstream tasks. As it gets rid of  the manual annotation requirement  in  supervised deep learning, SSL has been widely attracted increasing researching interests~\cite{he2020momentum,aa2020}.

%\vspace{-0.05em} is a recently developed leading SSL  family. It 
As a leading approach in SSL, 
contrastive learning~\cite{hadsell2006dimensionality, he2020momentum,chen2020simple,hjelm2018learning,oord2018representation,bachman2019learning}  constructs a novel    instance discrimination pretext task to
train a network so that the representations of different crops  (augmentations)  of the same instance    are  close, while representations of different instances are far from each other.  Specifically, for an  image crop query, it  randomly augments the same image to obtain   a positive, and view    other  image crops as negatives. Then it constructs a one-hot label over the positive and negatives   to pull the query together with its positive  and push the query away its negatives  in the feature space.

\noindent{\textbf{Motivation.}} But the  one-hot labels in contrastive learning are indeed  inaccurate and  uninformative.  It is because for a query, it could be   semantically similar or even more similar   to some negatives than its positives.  Indeed,  some negatives  even belong to the same semantic class as the query~\cite{arora2019theoretical,chuang2020debiased,wei2020co2}.  It holds in practice, as (i) to achieve good performance,  one often uses sufficient  negatives that are  much more than the semantic class number, e.g. 
%tens of thousands of negatives for ImageNet~\cite{deng2009imagenet} 
in MoCo~\cite{he2020momentum},  unavoidably leading to the issue on negatives; (ii)   even for the same image,  especially for an  image  containing different objects which occurs in ImageNet, random augmentations, e.g. crop,  provide  crops with (slightly) different semantic information,  and thus some of the huge negatives could be  more similar to  query.   
So  the one-hot label cannot well reveal the semantic similarity between  query and its positives and ``negatives",  and  cannot guarantee the semantically similar samples to close each other,  leading to performance degradation.

\noindent{\textbf{Contributions.}}  In this work, we alleviate the above label 
issue,  and  derive some new results and alternatives   for contrastive learning. Particularly,  we theoretically  show that    inaccurate labels impair the performance of  contrastive learning. Then we propose a self-labeling refinement method  to  obtain more accurate  labels for contrastive learning.   Our  main contributions are highlighted below.

Our first contribution is proving that the generalization error of MoCo for instance discrimination linearly depends on the discrepancy between the estimated labels  (e.g. one-hot labels) in MoCo and the true labels that  really reflect  semantical similarity  between a query  and its positives and negatives.  Formally, given $n$ training queries $\Dg\!\!=$ $\{\xmi{i}\}_{i=1}^n$ with  estimated labels $\{\ymi{i}\}_{i=1}^n$ (e.g. one-hot labels in MoCo) and ground truth labels $\{\ymsi{i}\}_{i=1}^n$  on their corresponding positives and negatives,   the generalization error of MoCo for instance discrimination    is lower bounded by $\mathcal{O}\big( \EE_{\Dg} \! \left[\left\| \ymi{} \!-\! \ymsi{}\right\|_2 \right]\!\big)$ where $ \EE_{\Dg} [\left\| \ymi{} \!-\! \ymsi{}\right\|_2]\!=\!\frac{1}{n}\!\sum_{i=1}^n \!\| \ymi{i} \!-\! \ymsi{i}\|_2$, and  is upper bounded by  $\mathcal{O}\big(\!\sqrt{\!\ln(|\F|)/n} $ $+ \EE_{\Dg}  \left[\left\| \ymi{}\! -\! \ymsi{}\right\|_2 \right] \big)$, where $|\F|$ is    the covering number  of the network hypotheses  in MoCo. It means that the more accurate of the estimated labels $\{\ymi{i}\}_{i=1}^n$, the better generalization  of MoCo for instance discrimination.
% Note,   poor  performance on  instance discrimination task often gives unsatisfactory performance on the downstream tasks,  as instance discrimination aims to   learn invariant features under various augmentations which can guarantee strong generalizations to downstream tasks. So estimating accurate labels for query instances in instance discrimination task is necessary and important.

Inspired by our theory, we propose a Self-lAbeliNg rEfinement (SANE) method which iteratively  employs the network  and data themselves to   generate more accurate and  informative soft labels for contrastive learning.  SANE has  two  complementary modules:   (i) \textit{Self-Labeling Refinery (SLR)}  to explicitly generate accurate labels, and (ii) \textit{Momentum Mixup (MM)} to  increase similarity between query and its positive and implicitly improve label accuracy.  Given a query, SLR uses  its one positive to estimate  semantic similarity between the query and its   keys (i.e. its positive and negatives)  by  computing  their feature  similarity, since a query and its positive   come from the same image  and should have close   similarity on the same  keys.  Then  SLR linearly combines the  estimated  similarity of a query with its vanilla one-hot  label  in contrastive learning to  iteratively generate  more accurate and informative soft labels. Our  strategy is that at the early training stage,  one-hot label has heavy combination weight to provide relatively accurate labels;  along with more training, the estimated  similarity becomes more accurate and informative, and its combination  weight becomes larger as it explores useful underlying semantic information between the query and its  keys that is missing in the one-hot  labels.  % This strategy is  both empirically and theoretically effective. In theory,
Besides,  we prove that when the semantic labels in the instance discrimination task  are corrupted,  our  SLR    can exactly recover the true semantic  labels of training data, and  networks trained with  our SLR can exactly predict the true semantic labels of test samples.   %These results also demonstrate of effectiveness of our method.  %method achieves the state-of-the-art results:

Moreover, we introduce MM for contrastive learning to further reduce the possible label noise  and  also increase augmentation diversity.  Specifically, we randomly  combines queries $\{\xmi{i}\}_{i=1}^{n}$ and their positives $\{\xmti{i}\}_{i=1}^{n}$ as  $
\xmfi{i} \!=\! \theta \xmi{i} + (1 \!-\!  \theta)  \xmti{k}$ and  estimate their  labels as  $
\ymfi{i} \!= $ $ \theta \ymbii{}{i}  +(1  \!- \!  \theta)  \ymbii{}{k}$, where indexes $i$ and $k$ are randomly selected, $\ymbii{}{i}$ is the  label of both $\xmi{i}$ and $\xmti{i}$ estimated by our label refinery,  and $\theta\!\in\!(0,1)$ is a random variable.  In this way, the component $ \xmti{k}$ in the virtual query $\xmfi{i}$ directly increases the similarity between the query $\xmfi{i}$  and the positive key $\xmti{k}$. So the label weight $(1-\theta)$ of  label $\ymfi{i}$ on  positive key $\xmti{i}$ to bring $\xmfi{i}$ and $\xmti{k}$ together  is relatively accurate, as $\xmfi{i}$ really contains the semantic information of $\xmti{k}$. Meanwhile, the possible noise at the remaining positions of label  $\ymfi{i}$ is scaled by $\theta$ and becomes smaller. In this way, MM also improves the label quality.  %This is the reason why we  select one  positive $\xmti{k}$ that plays a contrastive key for instance discrimination instead of the query $\xmi{k}$ for mixup. 
%As $\xmti{k}$ is fed into momentum-updated network,    we call this strategy ``momentum mixup". 

%Finally, our method sets  new state-of-the-arts for unsupervised representation learning on multiple benchmarks. For instance,  our method
%%%%%achieves $96.4\%$ accuracy and
%improves  1$\%$ accuracy over supervised method on CIFAR10, and   achieves 76.4$\%$ on ImageNet which is comparable to 76.5$\%$ of  supervised baseline.

\textbf{Other Related Work.}  
To estimate similarity between a query and its negatives,  Wei~\et~\cite{wei2020co2} approximated the  similarity by computing  cosine  similarity  between a positive  and  its negatives, and  directly replaced the one-hot label  for instance discrimination. Wang \et~\cite{aa2020}  used   similar similarity estimated on weak augmentations to supervise the learning of strong augmentations.  In contrast, we  respectively estimate the  similarities of the query on all contrastive keys ( its  positive and  negatives) and on only negatives,  and  linearly combines  two estimated similarities with vanilla one-hot label  to obtain more accurate and informative label with  provable performance guarantee.  Learning from noisy label, e.g.~\cite{reed2014training,bagherinezhad2018label},   %berthelot2019mixmatch
also uses soft labels generalized by a network to  %denoise crop labels and
supervise representation learning, and often  focus on (semi-)supervised learning that differs from our self-supervised learning.   %our work further reduces the label noise via momentum mixup.

Two relevant  works   \cite{kim2020mixco,lee2020mix}   performed vanilla mixup on all query instances to increase data diversity.  Differently, our momentum mixup mainly aims to reduce  label noise, as it randomly combines one query with one positive  (instead of one query) of other instances to increase the similarity between the query and its  its positive.  Verma \et~\cite{verma2020towards} showed that mixup is a better domain-agnostic noise than Gaussian noise for positive pair construction. But they did not perform mixup on labels, which is contrast  to \cite{kim2020mixco,lee2020mix} and ours.  See more  discussion in Sec.~\ref{secmomentummixup} and empirical comparison in Sec.~\ref{AblationStudy}. 

\vspace{-0.8em}
\section{Inspiration: A Generalization Analysis of MoCo}\label{analysissect}
\vspace{-0.4em}
In this section, we first briefly review the MoCo~\cite{he2020momentum} method popularly studied for contrastive learning, and then analyze the impact of inaccurate label assignment on its generalization ability.

\textbf{Review of MoCo.} The MoCo method contains  an online network $\fo_{\wm}$ and a target network  \scalebox{1}{$\ft_{\xim}$} receptively parameterized by $\wm$ and $\xim$.  Both $\fo_{\wm}$ and  \scalebox{1}{$\ft_{\xim}$} consists of a  feature encoder and   a  projection head (\eg~3-layered MLP). Given a  minibatch $\{\cmi{i}\}_{i=1}^s$ at each iteration, it first randomly augments each vanilla image $\cmi{i}$  into two views $(\xmi{i},\xmti{i})$ and  optimizes the following  contrastive loss:  
\begin{equation}\label{infoNCE}
	\Linfo(\wm) = - \frac{1}{s}  \sum\nolimits_{i=1}^{s}   \log  \Big( \frac{\sigma( \xmi{i} ,\xmti{i})}{\sigma( \xmi{i},\xmti{i}) + \sum_{l=1}^b \sigma( \xmi{i},\bmi{l})}\Big),
\end{equation}
where \scalebox{1}{$\sigma$}$(\xmi{i},\xmti{i}) \!=\! \exp\big(-  \!\frac{\langle\fo(\xmi{i}),\ft(\xmti{i}) \rangle}{\tau \|\fo(\xmi{i})\|_2 \cdot \|\ft(\xmti{i}) \|_2}\big)$ with a temperature  \scalebox{1}{$\tau$}. The dictionary  $\Bm\!=\!\{\bmi{i}\}_{i=1}^b$ denotes the negative keys   of current minibatch queries $\{\xmi{i}\}_{i=1}^s$, and is often of huge size  to achieve satisfactory  performance, e.g. 65,536 in MoCo.  In practice, $\Bm$ in MoCo is updated  by the  minibatch features $\{\ft(\xmti{i})\}_{i=1}^s$ in a first-in and first-out order. By  fixing   \scalebox{1}{$\ft_{\xim}$}   and updating  $\fo_{\wm}$  in Eqn.~\eqref{infoNCE}, MoCo   pushes the query  $\xmi{i}$ away from its negative keys in dictionary $\Bm$ while pulling together its positive  key  $\xmti{i}$. For \scalebox{1}{$\ft_{\xim}$},  it is updated  via exponential moving average, i.e. $\xim \!=\! (1 - $\scalebox{1.1}{$\iota$}) $\xim + $\scalebox{1.1}{$\iota$}$\wm$ with a constant \scalebox{1}{$\iota$} $\in (0,1)$.

From Eqn.~\eqref{infoNCE}, one can observe that MoCo views each image as an individual class and uses one-hot label $\ym\!\in\!\Rs{b+1}$ (its nonzero position is at the position of its positive key)  to train  $\fo_{\wm}$.
However, as mentioned in Sec.~\ref{introduction},  the one-hot labels  cannot reveal the semantic similarity between a query $\xmi{i}$ and its positive and negatives and thus  impair  representation learning. In the following, we theoretically analyze  the   effect of inaccurate labels to the generalization of MoCo for instance discrimination.
%As a result, for one query, using its inaccurate  one-hot label  to push  itself away negative samples could lead to performance degradation.

\textbf{Generalization Analysis.} We focus on analyzing MoCo in the final training stage where  the sample (key) distribution in the dictionary $\Bm$ is almost fixed. This simplified setup is reasonable because  (i) in the final training stage, the target network \scalebox{1}{$\ft_{\xim}$} almost does not change  due to the very small  momentum updating parameter \scalebox{1}{$\iota$} in practice and the oncoming convergence of  the online network $\fo_{\wm}$;  (ii)  dictionary is sufficient large  to cover different patterns in the dataset.  This fixed sample distribution  simplifies the analysis, and also  provides valuable insights.

Let $\Dg=\{(\xmi{i}, \xmti{i})\}_{i=1}^n$ denote  the training positive  pairs in  MoCo sampled from an unknown distribution $\SSS$. Moreover, the query $\xmi{i}$ has ground truth soft label $\ymsi{i}\!\in\!\Rs{b+1}$ over the key set $\Bmi{i}\!=\!\{\xmti{i}\cup \Bm\}$, where  the $t$-th entry $\ymsi{it}$ measures the semantic similarity between $\xmi{i}$ and the $t$-th key $\bmi{t}'$ in $\Bmi{i}$.  In practice, given query $\xmi{i}$ and dictionary $\Bmi{i}$,  MoCo estimates an one-hot label of $\xmi{i}$ as  $\ymi{i} \in \Rs{b+1}$ whose first entry is one and remaining entries are zero. So $\ymi{i}$ ignores the semantic similarity between  $\xmi{i}$ and  keys in $\Bmi{i}$, and  differs from $\ymsi{i}$.  Then MoCo  minimizes an empirical risk: %optimizes  $\fo_{\wm}$  by minimizing empirical risk:
\begin{equation}\label{empiricalrisk}
	\Qa(\fo_{\wm}) = \frac{1}{n} \sum\nolimits_{i=1}^{n} \ell(h(\fo_{\wm}(\xmi{i}), \Bmi{i}),\ymgi{i}),
\end{equation}
where $h(\fo_{\wm}(\xmi{i}),\! \Bmi{i})\!=\![\sigma(\xmi{i}, \!\xmti{i}), \sigma(\xmi{i}, \!\bmi{1}),\cdots\!, \sigma(\xmi{i}, \!\bmi{b})]$  denotes the predicted class probability, and $\ell(\cdot,\cdot)$ is cross-entropy loss.
% measures the discrepancy between the prediction $h(\fo_{\wm}(\xmi{i}),\! \Bmi{i})$ and the label $\ymgi{i}$.
Ideally, one should  sample sufficient pairs $(\xmi{i},\xmti{i})$ from the distribution $\SSS$ and use the ground truth label $\ymsi{i}$ of  $\xmi{i}$  to minimize the  population risk:
\begin{equation}\label{populationrisk}
	\Qgp(\fo_{\wm}) = \EE_{(\xmi{i},\xmti{i})\sim\SSS}\left[ \ell(h(\fo_{\wm}(\xmi{i}), \Bmi{i}),\ymsi{i})\right].
\end{equation}
 Here we assume the ground truth label $\ymsi{i}$ is soft which is indeed  more reasonable and stricter  than  the one-hot label setting especially for contrastive learning~\cite{hinton2015distilling}.  It is because soft label  requires the networks to capture the semantic similarity between  query and the instances in $\Bmi{i}$ and  bring semantically similar instances together, greatly helping  downstream tasks (e.g. classification) where global semantic information is needed, while one-hot label only needs networks to distinguish each instance from others and does not consider the global semantic structures in the data.   %Actually, the semantic similarity among samples here is also known as ``dark knowledge" in knowledge distillation~\cite{hinton2015distilling} or ``Bayes class-probability", and is often used to replace the one-hot label for training network with remarkable performance improvement  in many tasks.
 %, e.g. classification~\cite{chen2017learning}.
 As both the data distribution $\SSS$ and the ground truth labels are unknown, MoCo   optimizes the empirical risk $\Qa(\fo_{\wm})$ in~\eqref{empiricalrisk} instead of the population risk $\Qgp(\fo_{\wm})$ in~\eqref{populationrisk}.
%In this way,  the optimal network $\fo_{\wm}$ by minimizing $\Qa(\fo_{\wm}) $ differs from that via optimizing $\Qgp(\fo_{\wm})$. 
It is natural to ask whether $\fo_{\wm}$ by minimizing  $\Qa(\fo_{\wm}) $ can well perform instance discrimination task in contrastive learning, i.e. whether $\fo_{\wm}$ can capture the semantic similarity  ($\ymsi{i}$) between any test sample  $(\xmi{i},\xmti{i})\! \sim \!\SSS$ and the keys (samples) in $\Bmi{i}$.  To solve this issue, 
Theorem~\ref{proximatetrisk} analyzes the generalization error of $\fo_{\wm}$  for instance discrimination.  
\begin{thm} \label{proximatetrisk}
\!\! Suppose $\ell(h(\fo_{\wm}(\xmi{}), \!\Bmi{\xm}),\ymi{}) \!\in\! [a_1, a_2]$,  $\ell(\cdot,\ymi{})$ is $\Ly$-Lipschitz w.r.t. $\ym$. 	Let $\F$ be a finite class of hypotheses $\ell(h(\fo_{\wm}(\xmi{}),\! \Bmi{\xm}),\ymi{})\!: \! \X \!\times \!\Y \!\rightarrow \!\Rs{}$ and $|\F|$ be its covering number under $\|\!\cdot\!\|_{\infty}$ metric.  \\
	(1) Let $\EE_{\Dg\sim\SSS} \! \left[\left\| \ymi{} \!-\! \ymsi{}\right\|_2\right]=\EE_{\Dg\sim\SSS}\! \left[ \frac{1}{n}\!\sum_{i=1}^n\!\left\| \ymi{i}\! -\! \ymsi{i}\right\|_2\right]$. For any $\nu\! \in\!(0,1)$, it holds 
	\begin{equation*}
		\begin{split}
			\left| \Qgp(\fo_{\wm}) -\Qa(\fo_{\wm}) \right| \leq  \Ly  \EE_{\Dg\sim\SSS}  \left[\left\| \ymi{} - \ymsi{}\right\|_2 \right]  + \sqrt{\frac{2 (a_2-a_1)^2 V_{\Dg} \ln(2|\F|/\nu)}{n}} + \frac{7(a_2-a_1)^2  \ln(2|\F|/\nu)}{3(n-1)},
		\end{split}
	\end{equation*}
with probability at least $1\!-\!\nu$,	where  $V_{\Dg} $ is the variance of $\ell(h(\fo(\xm),\! \Bmi{\xm}),\yms{})$ on the data  $\Dg$. \\
(2) There exists a contrastive classification problem,  a class of hypotheses   $\ell(h(\fo_{\wm}(\xmi{}),  \Bmi{\xm}),\ymi{}) :  \X  \times  \Y \rightarrow \Rs{}$ and a constant $c_0$ such that the generalization error of $\fo_{\wm}$ is lower bounded
	\begin{equation*}
		\big|\Qgp(\fo_{\wm})  - \Qa(\fo_{\wm}) \big| \geq c_0 \cdot \EE_{\Dg\sim\SSS}  \left[\left\| \ymi{} - \ymsi{}\right\|_2 \right]. 
	\end{equation*}
\end{thm}
See its proof in Appendix~\ref{proofofapproximation}.
Theorem~\ref{proximatetrisk} shows that for the task of learning semantic similarity between a query and its positive and negatives which is important for downstream tasks (e.g., classification), the  generalization error of $\fo_{\wm}$  trained with the one-hot labels $\ym$ is upper bounded by  $\mathcal{O}\big(\EE_{\Dg\sim\SSS}\!  \left[\left\| \ymi{} - \ymsi{}\right\|_2 \right] \!+\!\! \sqrt{V_{\Dg} \ln(|\F|)/n}\big)$. It means  that large training sample number $n$   gives small generalization error, as  intuitively, model sees sufficient samples and can generalize  better. The loss variance $V_{\Dg}$ on the dataset $\Dg$   measures data diversity: the larger data diversity  $V_{\Dg}$,  the more challenging to learn a  model with good generalization. Here we are particularly interested in the factor $\EE_{\Dg\sim\SSS} \! \left[\left\| \ymi{} - \ymsi{}\right\|_2 \right] $ which reveals an important property: the higher accuracy of  the training label $\ymi{}$ to the ground truth label $\ymsi{}$, the smaller generalization error.  Moreover, Theorem~\ref{proximatetrisk}   proves that  there exists a contrastive classification problem such that the lower bound of generalization error   depends on $\EE_{\Dg\sim\SSS} \! \left[\left\| \ymi{} - \ymsi{}\right\|_2 \right] $. So the upper bound of generalization error is tight in terms of   $\EE_{\Dg\sim\SSS} \! \left[\left\| \ymi{} - \ymsi{}\right\|_2 \right] $. %, as it also occurs in the lower bound and  is not an artificial factor.
Thus, to better capture the underlying semantic similarity between query  $\xmi{i}$ and  samples in dictionary $\Bmi{i}$  to bring semantically similar samples together and better solve downstream tasks, one should provide   accurate  label $\ymi{i}$ to the soft true label $\ymsi{i}$.
%Actually, the semantic similarity between samples here is also known as ``dark knowledge" in knowledge distillation which characterizes the similarity of one sample to different semantic  classes.
In the following, we introduce our solution to estimate more accurate and informative soft labels   for   contrastive learning.
 
\vspace{-0.8em}
\section{Self-Labeling Refinement for Contrastive  Learning}
\vspace{-0.4em}
Our  Self-lAbeliNg rEfinement (SANE) approach for contrastive learning   contains  (i) \textit{Self-Labeling Refinery} (SLR for short) and (ii) \textit{Momentum  Mixup} (MM) which complementally refine  noisy labels  respectively from label estimation and positive pair construction.  SLR uses current training model and data to estimate more accurate and  informative soft labels, while MM increases  similarity between  virtual query and its positive,  and thus improves label accuracy. 

We begin by slightly modifying the instance discrimination task  in MoCo. Specifically, for the query $\xmi{i}$ in the current minibatch $\{(\xmi{i},\xmti{i})\}_{i=1}^s$, we maximize its similarity to its positive sample $\xmti{i}$ in the key set $\Bmb\! = \!\{\xmti{i}\}_{i=1}^s\! \cup \!\{\bmi{i}\}_{i=1}^b$ and minimize  its similarity to the remaining samples in $\Bmb$:
\begin{equation}\label{infoNCE2}
	\Lcon\big(\wm,  \{(\xmi{i},\ymi{i})\}\big)  =   - \frac{1}{s}  \sum\nolimits_{i=1}^{s}   \sum\nolimits_{k=1}^{s+b}  \ymi{ik} \log  \bigg(  \frac{\sigma( \xmi{i} ,\bmbi{k})}{\sum_{l=1}^{s+b} \sigma( \xmi{i}, \bmbi{l})}\bigg),
\end{equation}
where $\bmbi{k}$ is the $k$-th sample in $\Bmb$, and $\ymi{i}$ is the one-hot  label  of query $\xmi{i}$ whose   $i$-th entry $\ymi{ii}$ is one. In this way, the labels of current queries $\{\xmi{i}\}_{i=1}^s$ are defined on a shared set $\Bmb$, and can be linearly combined  which is key for SLR \&  MM. Next, we aim to improve the quality of label $\ymi{i}$ in~\eqref{infoNCE2} below.

\vspace{-0.5em}
\subsection{Self-Labeling Refinery}\label{sectionlabelrefine}
\vspace{-0.3em}
\textbf{Methodology.} As analyzed in Sec.~\ref{introduction} and~\ref{analysissect},  the one-hot labels in Eqn.~\eqref{infoNCE2}  could not well reveal the semantic similarity between $\xmi{i}$ and the instance keys in the set $\Bmb$, and thus impairs good representation learning. To alleviate this issue,  we introduce Self-Labeling Refinery (SLR) which employs  network and data themselves to  generate more accurate and  informative labels, and improves the performance of contrastive learning. Specifically, to refine the one-hot label $\ymi{i} $ of query $\xmi{i}$, SLR uses its positive instance $\xmti{i}$ to estimate  the  underlying semantic similarity between $\xmi{i}$ and  instances in $\Bmb\! = \!\{\xmti{i}\}_{i=1}^s\! \cup \!\{\bmi{i}\}_{i=1}^b$, since $\xmi{i}$ and $\xmti{i}$  come from the same image and should have close semantic similarity with instances in $\Bmb$. Let $\bmbi{k}$ be the $k$-th sample in $\Bmb$. Then  at the $t$-th iteration, SLR first estimates the instance-class probability $\pmii{t}{i}\!\in\!\Rs{s+b}$ of $\xmi{i}$ on the set $\Bmb$ whose $k$-th entry is defined as
\begin{equation*}\label{probability}
	\pmii{t}{ik} = \sigma^{1/\tau'}( \xmti{i} ,\bmbi{k})\big/\sum\nolimits_{l=1}^{s+b} \sigma^{1/\tau'}( \xmti{i},\bmbi{l}),\ \quad (\tau'\in(0, 1]).
\end{equation*}
The constant $\tau'$ sharpens $\pmii{t}{i}$ and removes some possible small noise,  since  
smooth labels cannot well distillate their knowledge to a network~\cite{muller2019does}.  %When $\tau'\rightarrow 0$, $\pmi{i}$ will approach a Dirac (“one-hot”) distribution.
Then SLR uses $\pmii{t}{i}$ to approximate the  semantic similarity between $\xmi{i}$ and the  instances in $\Bmb$ and employs it as the soft label of $\xmi{i}$ for contrastive learning.
%As $\pmii{t}{ik}$ measures the semantic similarity between $\xmti{i}$ and $\bmbi{k}$ and $\xmti{i}$, it can serve  as an instance-class probability to measure whether   $\xmi{i}$ and $\bmbi{k}$ share the same semantic class. Then we use $\pmi{i}$ to approximate the  semantic similarity between $\xmi{i}$ and the  instances in $\Bmb$.

However, since $\xmti{i}$ is highly similar to itself, $\pmii{t}{ii}$ could be much larger than others and conceals the similarity of other semantically similar instances in $\Bmb$. To alleviate this artificial effect, SLR removes $\xmti{i}$ from the set $\Bmb$ and re-estimates the similarity between $\xmi{i}$ and the remaining instances in $\Bmb$:
\begin{equation*}\label{probability2}
	\qmii{t}{ik} \!=\!\sigma^{1/\tau'}( \xmti{i} ,\bmbi{k})\big/\!\sum\nolimits_{l=1,l\neq i}^{s+b}\! \sigma^{1/\tau'}( \xmti{i},\bmbi{l}),\ \ \qmii{t}{ii} = 0.
\end{equation*}
Finally, SLR linearly combines the one-hot label $\ymi{i}$ and two label estimations, i.e. $\pmi{i}$ and $\qmi{i}$, to obtain more accurate, robust and informative label $\ymbii{t}{i}$ of $\xmi{i}$ at the $t$-th iteration:
\begin{equation}\label{labelcombination}
	\ymbii{t}{i} = (1-\alphai{t}- \betai{t} )\ymi{i} + \alphai{t} \pmii{t}{i} + \betai{t} \qmii{t}{i},
\end{equation}
where $\alphai{t}$ and $\betai{t}$ are two constants. In our experiments, we set
$\alphai{t}\!=\!\mu \max_k \pmii{t}{ik}/z$ and $\betai{t}\!=\!\mu \max_k \qmii{t}{ik} /z$, where  $z\!=\!1\!+\!\mu\max_k \pmii{t}{ik} \!+\!\mu\max_k \qmii{t}{ik}$, the constants  $1$, $\max_k \pmii{t}{ik}$ and $\max_k \qmii{t}{ik} $ respectively denote the largest confidences of labels $\ymi{i}$, $\pmii{t}{i}$ and $\qmii{t}{i}$ on a certain class. Here hyperparameter $\mu$  controls the prior confidence of $\pmii{t}{}$ and $\qmii{t}{}$. So SLR only has  two parameters $\tau'$ and $\mu$ to tune.

\textbf{The Benefit Analysis of Label Refinery.} Now we  analyze the performance of our SLR on  label-corrupted data. We first describe the dataset.  Let  $\{\cb_i\}_{i=1}^K\!\!\subset\!\R^d$ be  $K$ vanilla samples belonging to $\bar{K}\!\le\! K$ semantic classes, and  $\{(\xmi{i},\ymi{i})\}_{i=1}^n\!\!\in\!\Rs{d}\!\times \!\Rs{}$ be  the random crops of $\{\cb_i\}_{i=1}^K$.   Since in practice, one often cares more the semantic class prediction performance of a model which often directly reflects the performance on the downstream tasks, we assume that the labels $\{\ymi{i}\}_{i=1}^n$ denote corrupted semantic-class labels. Accordingly, we will analyze whether SLR  can refine  the corrupted labels $\{\ymi{i}\}_{i=1}^n$ and whether it  helps a model learn the essential semantic-class knowledge of $\{\xmi{i}\}_{i=1}^n$. Finally, while allowing for multiple classes,  we assume the labels are scalars and take values in $[-1,1]$ interval for simplicity. We formally define our label-corrupted dataset  below. 
\begin{defn} [$(\rho, \eps,\delta)$-corrupted dataset] \label{cdata}  Let  $\{(\xmi{i},\ymsi{i})\}_{i=1}^n$ denote the pairs of crops (augmentations)  and ground-truth semantic label, where crop $\xmi{i}$ generated from  the $t$-th sample $\cmi{t}$ obeys $\norm{\xmi{i}-\cmi{t}}\!\le\! \eps$ with a constant $\eps$, and  $\ymsi{i}\!\in\!\{\gamma_t\}_{t=1}^{\bar{K}}$ of $\xmi{i}$ is the label of $\cmi{t}$.
	Moreover,   samples and the crops are normalized, i.e.  $\norm{\cmi{i}}\!=\!\norm{\xmi{k}}\!=\!1 (\forall i, k)$.  Each  $\cmi{i}$ has $n_i$ crops, where $c_{l}\frac{n}{K}\!\le\! n_i\!\le \!c_{u}\frac{n}{K}$ with two constants $c_{l}$ and $c_{u}$.
	%Each $\xm$ that belongs to the $i$-th $\cmi{i}$ obey $\norm{\xm-\cmi{\ell}}\le \eps$ with a constant $\eps>0$.
	%	The labels $\ymi{i}\in\{\alpha_1,\alpha_2,\ldots,\alpha_{\bar{K}}\}$ belong to one of $\bar{K}\le K$ classes, where   $\{\alpha_\ell\}_{\ell=1}^{\bar{K}}\in[-1,1]$ denotes the labels associated with each class.
	Besides,  different classes are separated with a label separation $\delta$: % such that
	\begin{equation*}
		|\gamma_i-\gamma_k|\geq \delta, \quad \norm{\cmi{i}-\cmi{k}}\ge 2\eps,\ \ (\forall i\neq k).\label{alpha eq}
	\end{equation*}
%	where  $\delta$ is  the label separation. 
	  A $(\rho,\eps,\delta)$-corrupted dataset $\{(\xmi{i},\ymi{i})\}_{i=1}^n$ obeys the above conditions but with corrupted label $\{\ymi{i}\}_{i=1}^n$. Spefically, for each sample $\cmi{i}$, at most $\rho n_i$ augmentations  are  assigned to  wrong labels in $\{\gamma_i\}_{i=1}^{\bar{K}}$.
\end{defn}

Then we study a network of one hidden layer as an example to investigate the label refining performance of our SLR. The network parameterized by $\Wm\!\!\in\!\R^{k\times d}$ and $ \vm\! \in\!\Rs{k}$ is defined as
\begin{equation}
	\xm \in \Rs{d} \mapsto f(\Wm,\xm)=\vm^\top \phi(\Wm\xm),\label{nn model}
\end{equation}
where  
%$\Wm\!\in\!\R^{k\times d}$ and $\vm\!\in\!\Rs{k}$ are   network parameters, and  
$\phi$ is an activation function.  Following~\cite{li2018learning,shortest,li2020gradient} which  analyze convergence of networks or  robust learning of network, we  fix  $\vm$ to be a unit vector where half the entries are $1/\!\sqrt{k}$ and other half are $-1/\!\sqrt{k}$ to simplify exposition. So we  only optimize over  $\Wm$ that contains most  network parameters and will be shown to be sufficient for label refinery. Then given a $(\rho, \eps,\delta)$-corrupted dataset $\{(\xmi{i},\ymi{i})\}_{i=1}^n$, at the $t$-iteration we  train the network via minimizing the  quadratic loss:
\begin{equation}
	\Lt{t}(\Wm)\!=\!\frac{1}{2}\sum\nolimits_{i=1}^n (\ymbii{t}{i}-\!f(\Wm\!,\xmi{i}))^2 \!=\! \frac{1}{2} \norm{\ymbii{t}{}-\!f(\Wm,\Xm)}^2.\nn%\label{q loss}
\end{equation}%To optimize the network, w
Here the label $\ymbii{t}{i}$ of sample $\xmi{i}$ is estimated by Eqn.~\eqref{labelcombination}  in which  $\pmii{t}{i}=f(\Wmi{t},\xmti{i})$ denotes predicted label by using the positive $\xmti{i}$ of $\xmi{i}$, i.e. $\|\xmti{i}- \cmi{l}\|_2\leq \eps$ if $\xmi{i}$ is augmented from vanilla sample $\cmi{l}$.  We set $\betai{t} \!=\!0$ and $\tau'\!=\!1$ for simplicity, as (i) performing nonlinear mapping on network output greatly increases analysis difficulty; (ii)
%we will show that % even though $\betai{t} \!=\!0$ and $\tau'=1$,
our refinery~\eqref{labelcombination} is  still provably sufficient to refine labels when $\betai{t} \!=\!0$ and $\tau'\!=\!1$. Then we  update $\Wm$ via gradient descent algorithm with a learning rate $\eta$:
\begin{equation}
	\Wmi{t+1}=\Wmi{t}-\eta \nabla \Lt{t}(\Wmi{t}). \label{grad it}
\end{equation}
%where $\eta$ is  a learning rate.
 Following most  works on network convergence analysis~\cite{li2018learning,shortest,li2020gradient},  we use gradient descent and quadratic loss,  since (i) gradient descent is expectation version of stochastic one and often reveals similar convergence behaviors; (ii) one can expect similar results for other losses, e.g.  cross entropy, but  quadratic loss gives  simpler gradient computation.   For analysis, we  impose mild assumptions on  network~\eqref{nn model} and our  SLR, which are widely used in network  analysis~\cite{allen2019convergence,xie2017diverse,du2019gradient,du2018gradient}.
\begin{assum} \label{assumption1}
For   network~\eqref{nn model},  assume  $\phi$ and its first- and second-order derivatives obey  $|\phi(0)|,|\phi'(z)|,|\phi''(z)|\!\leq\! \Gamma$ for  $\forall z$ and some $\Gamma\!\geq\! 1$, 
the  entries of initialization  $\Wmi{0}$ obey  i.i.d.~$\mathcal{N}(0,1)$.
\end{assum}
\begin{assum} \label{assumption2} 
Define   network covariance matrix $\Sigm(\Cm)\!=\!(\Cm\Cm^\top)\odot \EE_{\um}[\phi'(\Cm\um)\phi'(\Cm\um)^{\top} ]$
	where $\Cm\!=\![\cmi{1}\dots\cmi{K}]^{\top}$,  $\um\!\sim\N(\bm{0},\Imm)$, $\odot$ is the elementwise product.  Let $\lambda(\Cm)\!>\!0$ be  the minimum eigenvalue of $\Sigm(\Cm)$. For  label refinery, assume  $3 \sqrt{n}  \sum_{t=0}^{t_0-1} \!  |\alphai{t} -  \alphai{t+1} | \leq  \psi_1   \norm{ f(\Wmi{0},\Xm)-\yms } $ and $3 \sqrt{n}\sum_{t=0}^{t_0-1}   \big(1- \frac{\eta\alpha^2}{4} \big)^{t_0-t} |\alphai{t} - \alphai{t+1} |\!\leq\! \psi_2 \norm{f(\Wmi{0},\Xm)-\yms\!}^2  $,  %$\betai{t}\!=$ $0$, $\tau'\!=\!1$,
	where $t_0\!=\! \!\frac{c_1K}{\eta n\lambda(\Cm)}\!\log\!\big( \frac{\Gamma\sqrt{n\!\log\! K}}{(1-\alphamax) \rho} \big)$ with three constants $\psi_1$, $\psi_2$ and $c_1$.  Here	$\alphai{\mbox{\scriptsize{max}}}$ is defined as $\alphai{\mbox{\scriptsize{max}}}\!=\!\max_{1\leq t  \leq t_0} \! \alphai{t}$.
\end{assum}

Assumption~\ref{assumption1} is mild, as most differential activation functions, e.g. softplus and sigmoid, satisfy it, and  the Gaussian  initialization     is used in practice. We assume Gaussian variance to be one for notation simplicity, but our technique is applicable to any constant variance. Assumption~\ref{assumption2}  requires that the discrepancy  between $\alphai{t}$ and $\alphai{t+1}$ until some iteration number $t_0$ are bounded,  which  holds by setting proper  $\alphai{t}$. 
For  $\lambda(\Cm)$, many works~\cite{allen2019convergence,xie2017diverse,du2019gradient,du2018gradient,li2020gradient}  empirically and theoretically show $\lambda(\Cm)\!>\!0$.  Based on these assumptions, we state our results  in Theorem~\ref{mainresults} with  constants  $c_1\!\sim\! c_6$. 
\begin{thm}\label{mainresults} Assume $\{(\xmi{i},\ymi{i})\}_{i=1}^n$ is a  $(\rho,\eps,\delta)$-corrupted dataset with noiseless labels $\{\ymsi{i}\!\}_{i=1}^n$.   Let $\xi\!=\!\log\big(\frac{\Gamma\!\sqrt{\!n\!\log\! K}}{\rho}\!\big).$
	Suppose  $\eps$ and the number $k$ of hidden nodes satisfy $\eps \!\leq\!  c_2\min(\frac{\lambda(\Cm)}{K\Gamma^2\xi^6},\frac{\rho}{\alphamax})$, $ k\! \geq\!  {\frac{c_3 K^2  \Gamma^{10}\xi^6 \|\Cm\|^4  }{\alphamax^2\lambda(\Cm)^4}}.$ 	Let $\psi'=1 + \frac{\psi_1}{2} +\sqrt{\psi_2}$.
	If step size $\eta\leq \frac{K}{2c_{up}n\Gamma^2\opnorm{\Cm}^2}$, with probability $1\!-\!3/K^{100}\!-\!K\exp(-100d)$, after $t\!\geq\! t_0\!=\! \frac{c_4 K}{\eta n\lambda(\Cm)} \log\big(\frac{\Gamma\sqrt{n\log K}}{(1-\alphamax) \rho}\big)$ iterations,  the gradient descent~\eqref{grad it} satisfies:\\
	(1) By defining $\zeta\!=\! 4\rho\!+\! c_5 \eps\psi' K \Gamma^3 \xi  \sqrt{\log K}/\lambda(\Cm)$ and  $\yms{}=[\ymsi{1},\!\cdots,\ymsi{n}]$, the discrepancy between  the  label $\ymbii{t}{}$ estimated by our SLR~\eqref{labelcombination} and the true label $\ymsi{}$ of the augmentation data $\{\xmi{i}\}_{i=1}^n$ is  bounded:
	\begin{equation*}
		\frac{1}{\sqrt{n}} \norm{\ymbii{t}{}-\yms} \leq  \frac{1-\alphai{t}}{\sqrt{n}} \norm{\ym-\yms} + \alphai{t} \zeta.
	\end{equation*}
	where $\ymbii{t}{}=[\ymbii{t}{1},\!\cdots,\ymbii{t}{n}]$. 
	Moreover, if $\rho\!\leq\!\frac{\delta}{32}$,  $\eps\!\leq\! c_6\delta  \min \big(\!\frac{ \lambda(\Cm)^2}{{\psi' \Gamma^5K^2}\xi^3}, \frac{1}{\Gamma\sqrt{d}}\big)$,    $1\!-\!  \frac{3}{4}\delta \leq\!  \alphai{t}$,  the  estimated label $\ymbii{t}{}$   predicts  true  label $\ymsi{i}$ of any crop     $\xmi{i}$:
	\begin{equation*}
		\gamma_{k^*} =\ymsi{i} \quad \text{with}\quad  k^*= \argmin\nolimits_{1\leq k \leq \bar{K}}| \ymbii{t}{i}-\gamma_k|.
	\end{equation*}
	(2) By using the refined label $\ymbii{t}{}$ in~\eqref{labelcombination} to train network and letting $f(\Wmi{t},\Xm) \!=\! [f(\Wmi{t},\xmi{1}), \cdots, f(\Wmi{t},\xmi{n})]$, the error of network  prediction on  $\{\xmi{i}\}_{i=1}^n$ is upper bounded
	\begin{equation*}
		\frac{1}{\sqrt{n}} \norm{f(\Wmi{t},\Xm)-\yms} \leq \zeta.
	\end{equation*}
	If  assumptions on $\rho$ and $\eps$ in (1)  hold, for vanilla sample $\cmi{k}$  ($\forall k =\!1\cdots K$),    network $f(\Wmi{\!t}, \cdot)$  predicts the true semantic label  $\gamma_{k}$ of its any augmentation  $\xm$ that satisfies  $\norm{\xm\!-\!\cmi{k}}\!\leq \!\eps$:
	\begin{equation*}
		\gamma_{k^*} =\gamma_{k} \quad \text{with}\quad  k^*= \argmin\nolimits_{1\leq i \leq \bar{K}}|f(\Wmi{t},\xm)-\gamma_i|.
	\end{equation*}
\end{thm}
See its \textit{proof roadmap} and  proof  in   Appendix~\ref{proofofmain}. The first  result in Theorem \ref{mainresults}  shows that after  training iterations $t_0$,  the discrepancy between  the  label $\ymbii{t}{}$ estimated by our label refinery~\eqref{labelcombination}, i.e. SLR, and  ground truth label $\ymsi{}$ of cropped training data $\{\xmi{i}\}_{i=1}^n$ is  upper  bounded by $\mathcal{O}\big( \norm{\ym-\yms}  + \zeta \big)$. Both factors $ \norm{\ym-\yms}$ and $\rho$ in the factor $\zeta$ reflect  the  label error of the provided corrupted label $\ym$.  Another important factor  in $\zeta$ is  the smallest eigenvalue $\lambda(\Cm)$ of  network covariance matrix $\Sigm(\Cm)$  in Assumption~\ref{assumption2}.  Typically, the performance of  a network heavily relies on the data diversity even without label corruption. For instance, if two samples are nearly the same but have different  labels,  the learning of a network   is difficult.  $\lambda(\Cm)$ can  quantify this data diversity, as one can think of $\lambda(\mtx{C})$ as a condition number associated with the  network which measures the  diversity of the vanilla samples $\{\cmi{i}\}_{i=1}^n$.  Intuitively,   if  there are two similar vanilla samples,  $\Sigm(\Cm)$ is trivially rank deficient and has small  minimum eigenvalue, meaning more challenges to distinguish  the augmentations $\xm$ generated from $\cmi{i}$. 
% Otherwise, the more distinct the vanilla samples, the larger $\lambda(\Cm)$ is and the smaller label error $\zeta$ is. % Thus $\Sigm(\Cm)$ could quantify the ability of the neural network to distinguish between distinct cluster centers.
Moreover, when the label corruption ratio $ \rho$ and the augmentation  distance $\eps$ are small,  the   label $\ymbii{t}{i}$  estimated by our SLR  can   predict  the true semantic label $\ymsi{i}$ for any crop sample $\xmi{i}$, and thus can supervises a network to learn the essential  semantic-class knowledges from %of  augmentation  samples
  $\{\xmi{i}\}_{i=1}^n$.

The second   result in Theorem~\ref{mainresults} shows that by  using the refined label $\ymbii{t}{}$ in our SLR~\eqref{labelcombination} to train network $f(\Wm, \cdot)$, the  error of network prediction on augmentations $\{\xmi{i}\}_{i=1}^n$ can be  upper bounded by $\zeta$.  Similarly, the factor $\rho$ and $\lambda(\Cm)$ in $\zeta$ respectively reflect the initial label error and the data diversity, which both reflect the learning difficulty for a model on the augmentation data $\{(\xmi{i},\ymi{i})\}_{i=1}^n$.  More importantly,  our results also guarantee the test performance of the trained network  $f(\Wmi{\!t}, \cdot)$.  Specifically,  when the    label corruption ratio $ \rho$ and  sample augmentation distance $\eps$ are small,  for any vanilla sample $\cmi{k}$  ($\forall k\!=\!1\cdots K$),  the  network $f(\Wmi{\!t}, \cdot)$ trained by our SLR can exactly  predict  the true semantic label  $\gamma_{k}$ of its any augmentation $\xm$  (i.e.  $\norm{\xm\!-\!\cmi{k}}\!\leq \!\eps$). These results accord with   Theorem~\ref{proximatetrisk} that shows the more accurate of training labels, the better generalization of the trained network. These  results show the effectiveness of the refined labels by our method.

\vspace{-0.5em}
\subsection{Momentum Mixup}\label{secmomentummixup}
\vspace{-0.3em}
Now we propose  momentum mixup (MM) to further reduce the possible label noise in realistic data and increase the data diversity  as well.   Similar to vanilla mixup~\cite{zhang2017mixup}, we  construct virtual instance as
\begin{equation}\label{mixup}
	\xmfi{i} = \theta \xmi{i} + (1 -  \theta)  \xmti{k}, \quad
	\ymfi{i} = \theta \ymbii{}{i} + (1 -  \theta)  \ymbii{}{k}, \quad \theta\sim \betad(\kappa,\kappa) \in [0, 1], 
\end{equation}
where $\xmti{k}$ is randomly sampled from the  key set $\{\xmti{i}\}_{i=1}^s$,   $\ymbii{}{i}$ denotes the refined label  by Eqn.~\eqref{labelcombination},   $\betad(\kappa,\kappa)$ is a  beta distribution. Here  $ \xmi{i}$ and $\xmti{i}$ share the same label $\ymbii{}{i}$ on the set $\Bmb\! = \!\{\xmti{i}\}_{i=1}^s  \cup  \{\bmi{i}\}_{i=1}^b$, as they come  from the same instance.  We call the mixup~\eqref{mixup} as ``momentum mixup",  since the sample $\xmti{k}$ is fed into the momentum-updated  network  \scalebox{1}{$\ft_{\xim}$}, and plays a contrastive key for  instance discrimination.  So MM differs from the vanilla mixup used in~\cite{kim2020mixco,lee2020mix} where $\xmti{k} $ is replaced with $ \xmi{k}$ and both are fed into online network $\fo_{\wm}$, and enjoys the following advantages. %A detailed discussion on the benefit of momentum mixup is deferred to  Appendix~\ref{benefit_momentum_mixup}.

Firstly, MM can improve the accuracy of the label $\ymfi{i}$ compared with   vanilla mixup.
%the instance construction approach, i.e. $\xmfi{i} = \lambda \xmi{i} + (1 -  \lambda)  \xmti{j}$, also help obtain more accurate label $\ymfi{i}$.
For   explanation,  assume $\ymbii{}{i}$ in~\eqref{mixup}  is one-hot label. Then  $\xmfi{i}$ has two positive keys  $ \xmti{i}$ and $\xmti{k}$ in $\Bmb$ decided by its label $\ymfi{i}$.  Accordingly,  the component $ \xmti{k}$ in  $\xmfi{i} \!=\! \theta \xmi{i} \!+\! (1 \!- \! \theta)  \xmti{k}$ directly increases the similarity between the query   $\xmfi{i}$  and its positive key  $\xmti{k}$ in  $\Bmb$. So the label weight $(1\!-\!\theta)$ of  label $\ymfi{i}$ on the key $\xmti{k}$ to bring $\xmfi{i}$ and $\xmti{k}$ together  is relatively accurate,  as $\xmfi{i}$ really contains the semantic information of $\xmti{k}$. Meanwhile, the sum of  label weights in $\ymfi{i}$ on remaining instance in $\Bmb \backslash \xmti{k}$ is scaled by $\theta$, which  also scales the possible  label noise  on instances in $\Bmb \backslash \xmti{k}$ smaller due to $\theta\!<\!1$.  By comparison, for vanilla mixup,  the label weight $(1\!-\!\theta)$ of  label $\ymfi{i}$ on the key $\xmti{i}$  does not improve label accuracy. It is because  the positive  pair $\xmi{k}$ and $\xmti{k}$ are  obtained via random augmentation, e.g. crop, and  may not be semantically similar, meaning that the component $ \xmi{k}$ in  $\xmfi{i}$ could not increase   similarity with  $\xmti{k}$.
%This is also one reason that SwAV~\cite{caron2020unsupervised} requires one query instance to close to its two positive ones obtained by performing  random augmentation twice. This issue is alleviated, since the label wight $\lambda\ (<1)$ of $\xmfi{i}$ does not heavily requires $\xmfi{i}$ to close $\xmti{i}$.
So its label weight $(1-\!\theta)$ to push $\xmfi{i}$ close to the key $\xmti{k}$ is not as accurate  as the one in MM.
%Indeed,   Sec.~\ref{AblationStudy} empirically shows superiority of  momentum   mixup over vanilla mixup.

Another advantage of MM is that it  allows us to use strong augmentation. As observed in~\cite{aa2020}, directly using strong augmentation in contrastive learning, e.g. MoCo, leads to performance  degradation, since  the instance obtained by strong augmentation often heavily differs from the one with weak augmentation. As aforementioned, the component $ \xmti{k}$ in  $\xmfi{i}= \theta \xmi{i} + (1 -  \theta)  \xmti{k}$  increases the similarity between the query   $\xmfi{i}$  and the key  $\xmti{k}$ in $\Bmb$, even though $(\xmi{i}, \xmti{i})$ is obtained via  strong augmentation. So  MM could reduce the matching difficulty between positive instances.%, and promotes the learning of online network.
%Sec.~\ref{AblationStudy} shows that directly plugging strong augmentation into our method can work well and improves the result with only weak augmentation.

With all the components in place, we are ready to define our proposed \textit{SANE model} as follows:
\begin{equation}\label{finalmodel}
	\mathcal{L}(\wm)= (1-\lambda) \Lcon\big(\wm, \{(\xmi{i},\ymi{i})\}\big) + \lambda \Lcon\big(\wm,\  \{(\xmfi{i},\ymfi{i})\}\big),
\end{equation}
where $\Lcon\big(\wm, \{(\xmi{i},\ymi{i})\}\big)$ defined in Eqn.~\eqref{infoNCE2} denotes the vanilla contrastive loss with one-hot label $\ymi{i}$, $\Lcon\big(\wm, \{(\xmfi{i},\ymfi{i})\}\big)$ denotes the momentum mixup loss with label $\ymfi{i}$ estimated by our self-labeling refinery~\eqref{labelcombination}, and $\lambda$ is a constant. Experimental results in Sec.~\ref{experiments} show the effectiveness of  both loss terms. % can guide the model learning  in a bootstrap manner.
See algorithm details in Algorithm~\ref{Elisealgorithm} of Appendix~\ref{morexp}.

\textbf{Limitation Discussion.}  SANE follows MoCo-alike framework and hopes to obtain a more accurate soft label of a query over its positive and negatives for instance discrimination.  So one limitation of SANE is that it does not apply to BYOL-alike methods~\cite{grill2020bootstrap} that only pulls positive pair together and does not require any labels. However, momentum mixup in SANE which increases the similarity of positive pair may also benefit BYOL, which is left as our future work to thoroughly   test. %which is an interesting avenue of future research to thoroughly   test this.%  

%\textbf{Societal Impact Discussion.}  As an unsupervised learning method, SANE could benefit many applications of societal interest where only low-resource labeled data are available,  e.g. medical data~\cite{miotto2018deep}, as SANE can use a mass of unlabeled data for pretraining and  few labeled data for fine-tuning. But same to standard contrastive learning,  SANE may suffer from  learning bias  caused by the potential data bias, and may provide bias or worse performance on smaller classes or groups.

\vspace{-1.1em}
\section{Experiments}\label{experiments}
\vspace{-0.5em}
 
\subsection{Evaluation Results on CIFAR10 and ImageNet}\label{imagenetsec}
\begin{wraptable}{r}{0.365\linewidth}
%	\vspace{-1.85em}
	\caption{Classification accuracy ($\%$). \vspace{-1.5em}}
	%\vspace{-0.2em}
	\setlength{\tabcolsep}{1.5 pt}
	\renewcommand{\arraystretch}{0.86}
	\label{comparisontable3}  
	%	\footnotesize
	%\scriptsize
	\centering
	{ \fontsize{8.1}{3}\selectfont{
			%{\scriptsize{
			\begin{tabular}{c|cc}
				\toprule
				CIFAR10 dataset &   KNN &  linear evaluation \\%& average  \\
				\midrule
				MoCo v2~\cite{chen2020improved} & 92.5  & 93.9  \\
				SimCLR~\cite{chen2020big,chen2020simple}  & --- & 94.0 \\
				%SwAV~\cite{caron2020unsupervised} & \\ 
				BYOL~\cite{grill2020bootstrap}  &92.4 & 93.9 \\
				%PCL~\cite{li2020prototypical}  &  & \\
				DACL~\cite{verma2020towards} & --- & 94.4 \\
				CLSA (strong)~\cite{aa2020} &93.4 & 94.9\\
				i-Mix (+MoCo)~\cite{lee2020mix} & --- & 95.9 \\
				\midrule
				SANE & 95.2 & 96.1 \\
				SANE (strong)  &  95.5 & 96.5\\
				\midrule
				Supervised~\cite{lee2020mix}  & --- & 95.5  \\
				\toprule
			\end{tabular}
	}}
	%	\vspace{-8.1mm}
%	\vspace{-1.0em}
\end{wraptable}
\textbf{Settings.}  We use ResNet50~\cite{he2016deep} with a $3$-layered MLP  head  
%as a backbone~\cite{he2020momentum} 
for CIFAR10~\cite{krizhevsky2009learning} and ImageNet~\cite{deng2009imagenet}. We first pretrain SANE, and then train a linear classifier on top of $2048$-dimensional  frozen features  in  ResNet50. With  dictionary size $4,096$, we pretrain $2,000$ epochs on CIFAR10 instead of $4,000$ epochs of MoCo, BYOL, and i-Mix in~\cite{lee2020mix}. Dictionary size on ImageNet is   $65,536$. For linear classifier, we  train $200$/$100$ epochs on CIFAR10/ImageNet.  
%We use ADAM~\cite{kingma2014adam} to train   the classifier on CIFAR10 and adopt SGD for all other experiments. 
See  all optimizer settings in Appendix~\ref{morexp}. 
We use standard data augmentations  in~\cite{he2020momentum} for pretraining and test unless otherwise stated. E.g., for test, we  perform  normalization on CIFAR10,  and  use   center crop and normalization on ImageNet. 
For SANE, we set $\tau\!=\!0.2,  \tau'\!=\!0.8,  \kappa\!=\!2$ in $\betad(\kappa,\kappa)$ on CIFAR10, and  $\tau\!=\!0.2$, $\tau'\!=\!1, \kappa\!=\!0.1$ on ImageNet.  For confidence $\mu$, we increase it as $\mu_t \!= \!m_{2}$ $-  (m_2-m_1)(\cos(\pi t/T)\!+\!1)/2$ with current iteration $t$ and total training iteration $T$. We set $m_1\!=\!0,$   $m_2\!=\!1$ on CIFAR10, and $m_1\!=\!0.5,$   $m_2\!=\!10$ on ImageNet.  For KNN on CIFAR10, its   neighborhood number is $50$  and its temperature is $0.05$.

For CIFAR10, to fairly compare with~\cite{lee2020mix}, we crop each image into two views to construct the loss~\eqref{finalmodel}. For ImageNet,   we follow~CLSA~\cite{aa2020} and  train  SANE in two settings.   SANE-Single  uses a single  crop in momentum mixup loss $\Lcon\big(\wm,\! \{\!(\xmfi{i},\ymfi{i})\!\}\big)$ in~\eqref{finalmodel} that crops each  image to a smaller size of $96\!\times\! 96$, without    much extra computational cost to process these small  images.    SANE-Multi  crop each image into five sizes $224\!\times\! 224$,  $192\!\times\!192$, $160\!\times\!160$, $128\!\times\!128$, and $96\!\times\!96$ and averages their  momentum mixup losses. This ensures a fair comparison with CLSA and SwAV. Moreover, we  use strong augmentation strategy in  CLSA. Spefically, for the above small image, we randomly select an operation from 14 augmentations used in   CLSA, and apply it to the image with a probability of 0.5, which is repeated 5 times.  
%Then we  repeat this process  five times  to obtain one strong augmentation of the image.  
We use ``(strong)" to mark whether we use strong augmentations on the small images in momentum mixup loss. Thus, SANE has almost the same training cost with CLSA, i.e. about $75$ ($198$) hours with $8$  GPUs,  $200$ epochs,  batch size of $256$ for SANE-Single (-Multi).  For vanilla contrastive loss on ImageNet, we always use weak augmentations. See more details of the  augmentation, loss construction, and pretraining cost on CIFAR10 and ImageNet   in Appendix~\ref{morexp}.

\begin{wraptable}{r}{0.65\linewidth}%靠文字内容的左侧
	%	\vspace{-0.35em}
	\setlength{\tabcolsep}{1pt} % column spacing
	\renewcommand{\arraystretch}{1.1}%  row spacing
	\caption{Top-1 accuracy (\%) under linear evaluation on ImageNet.% with the ResNet50 backbone.  
		%The left table compared methods trained over 200 epochs, and the right table compared the methods with various numbers of epochs.
		\vspace{-0.4em}}
	\label{comparisontable4}  
	\centering
	{ \fontsize{7.8}{3}\selectfont{
			\begin{tabular}{c|c|c|| c|c}
				\toprule
				augmentation &method (200 epochs) & Top 1 & method ($\geq$800 epochs) & Top 1  \\%& average  \\
				\midrule
				&	MoCo~\cite{he2020momentum} &  60.8 & PIRL-800epochs~\cite{misra2020self}  & 63.6\\
				&	SimCLR~\cite{chen2020simple} & 61.9 & CMC~\cite{tian2019contrastive} &  66.2\\
				&	CPC v2~\cite{henaff2019data} & 63.8 & SimCLR-800epochs~\cite{chen2020simple} & 70.0\\
				&	PCL~\cite{li2020prototypical} & 65.9 & MoCo v2-800epochs~\cite{chen2020improved}  & 71.1\\
				\multirow{2}{*}{{weak}}	&	MoCo v2~\cite{chen2020improved} &  67.5 &BYOL-1000epochs~\cite{grill2020bootstrap} & 74.3  \\
				&  CO2~\cite{wei2020co2} & 68.0 & SimSiam-800epochs~\cite{chen2020exploring} & 71.3\\
				&MixCo~\cite{kim2020mixco} &  68.4 &  i-Mix-800epochs~\cite{lee2020mix} & 71.3  \\
				&	SWAV-Multi~\cite{caron2020unsupervised} & 72.7 & SWAV-Multi-800epochs~\cite{caron2020unsupervised} & 75.3\\
				&	 	SANE-Single & 70.6   &  SANE-Single-800epochs & 73.0	\\
				&	 	SANE-Multi &  73.5   &  SANE-Multi-800epochs	&  75.7   \\
				\midrule	
				
				&	CLSA-Single~\cite{aa2020} &  69.4  & CLSA-Single-800epochs~\cite{aa2020}  & 72.2 \\
				\multirow{2}{*}{{strong}} &	CLSA-Multi~\cite{aa2020}  & 73.3 & CLSA-Multi-800epochs~\cite{aa2020} &  76.2 \\

				&	SANE-Single  &  70.1 &  SANE-Single-800epochs  & 73.5 \\
				&	SANE-Multi  & 73.7 & SANE-Multi-800epochs  & 76.4 \\
				\midrule	
				strong + JigSaw &	InfoMin Aug~\cite{tian2020makes}  & 70.1& InfoMin Aug-800epochs~\cite{tian2020makes}&73.0\\
				\midrule
				\multirow{2}{*}{{others}}	 &	InstDisc~\cite{wu2018unsupervised} & 54.0 & BigBiGAN~\cite{donahue2019large}  & 56.6 \\
				&		LocalAgg~\cite{zhuang2019local} &  58.8 & SeLa-400epochs~\cite{asano2019self}&  61.5 \\
				\midrule	
				&	Supervised~\cite{chen2020simple}  & 76.5 & Supervised~\cite{chen2020simple} & 76.5\\
				\bottomrule
			\end{tabular}
	}}
	\vspace{-1.0em}
\end{wraptable}
\textbf{Results.}    Table~\ref{comparisontable3} shows that with weak or strong augmentations,  SANE  always surpasses  the  baselines on  CIFAR10. Moreover,  SANE with strong (weak) augmentation   improves supervised  baseline  by    $1.0\%$ ($0.6\%$).

Table~\ref{comparisontable4}  also shows that for ImageNet under weak augmentation setting, for  $200$ ($800$) epochs SANE-Multi respectively  brings  $0.8\%$ ($0.6\%$)  improvements over  SwAV;   with $200$ ($800$) epochs, SANE-Single  also beats the runner-up MixCo (i-Mix and SimSiam). Note, BYOL outperforms SANE-Single but was trained $1,000$ epochs. With strong   augmentation,  SANE-Single and SANE-Multi  also respectively outperform   CLSA-Single and CLSA-Multi.   
Moreover,    our self-supervised accuracy $76.4\%$ is very close to  the accuracy $76.5\%$ of   supervised baseline, and still  improves $0.2\%$ over CLEAN-Multi even for this challenging case.    These results show the superiority and robustness of SANE, thanks to its  self-labeling refinery and momentum mixup which both  improve  label quality and thus bring semantically similar samples together.

\begin{table}[]
	%	\vspace{-1.0cm} 
	\begin{minipage}[t]{0.45\textwidth}
		\centering
		
		\setlength{\tabcolsep}{0.1pt} % column spacing
		\renewcommand{\arraystretch}{0.96}%  row spacing
		\caption{Transfer learning results.  \vspace{-0.1em}}
		\label{objecttable}  
		%	\footnotesize on  downstream tasks.
		%	\centering
		%	{\scriptsize{
		\centering
		{ \fontsize{8.1}{3}\selectfont{
				\begin{tabular}{c|ccc}
					\toprule
					\multirow{3}{*}{ {method}} & {classification}  & \multicolumn{2}{c}{ {object detection}} \\
					&  {VOC07} &  {VOC07+12} & {  {COCO}}\\
					&  {Accuracy} &  {AP$_{50}$} &  {AP}  \\
					
					\midrule
					
					%RotNet~\cite{gidaris2018unsupervised} & 64.6& ---& ---\\
					NPID++~\cite{wu2018unsupervised} & 76.6 & 79.1& ---  \\
					MoCo~\cite{he2020momentum} &  79.8 & 81.5 & --- \\
					PIRL~\cite{misra2020self}  & 81.1 & 80.7 & --- \\
					%PCL~\cite{li2020prototypical} & 84.0 & --- & ---  \\
					BoWNet~\cite{gidaris2020learning} & 79.3 & 81.3 & --- \\
					
					SimCLR~\cite{chen2020simple}  & 86.4 & --- & --- \\
					CO2~\cite{wei2020co2} &  85.2 &  82.7 & --- \\
					i-Mix~\cite{lee2020mix} &  --- &  82.7 & ---  \\
					MoCo v2~\cite{chen2020improved} &  87.1& 82.5 & 42.0   \\ 
					SWAV-Multi~\cite{caron2020unsupervised} & 88.9 & 82.6 & 42.1  \\
					CLSA-Multi(strong)\cite{aa2020}  &  93.6 & 83.2 & 42.3   \\
					\midrule	
					SANE-Multi & 92.9 & 82.9 & 42.2    \\
					SANE-Multi (strong) & 94.0 & 83.4 & 42.4    \\
					\midrule	
					Supervised~\cite{aa2020}  &  87.5 & 81.3 & 40.8 \\
					\toprule
		\end{tabular}}}
		\vspace{-1mm}
	\end{minipage}
	\hspace{1.5em}
	\begin{minipage}[t]{0.5\textwidth}
		\centering
		\caption{Effects of the components in SANE with strong augmentation on CIFAR10. \vspace{-0.1em}}
		%\vspace{-0.2em}  to  classification accuracy ($\%$)
		\setlength{\tabcolsep}{1pt}
		\renewcommand{\arraystretch}{0.88}
		\label{comparisontable1}
		%	\footnotesize
		%\scriptsize
		%	\centering
		%{\scriptsize{
		\centering
		{ \fontsize{8.1}{3}\selectfont{
				\begin{tabular}{c|c|c|c}
					\toprule
					label $\pmi{}$ in~\eqref{labelcombination}  &  label $\qmi{}$ in~\eqref{labelcombination}   & momentum mixup & accuracy (\%) \\%& average  \\
					\midrule
					&  & &  93.7   \\
					\checkmark  & &   & 94.6  \\
					& \checkmark & &94.5 \\
					& & 	\checkmark &94.8  \\
					\checkmark & \checkmark &  & 94.9 \\
					\checkmark  &  & \checkmark &95.2  \\
					& \checkmark & \checkmark &95.1  \\
					\checkmark & \checkmark & \checkmark &  95.9  \\
					\toprule
				\end{tabular}
		}}
		\vspace{0.1mm}
		\caption{Effects of  parameter $\lambda$ in SANE with strong augmentation on  CIFAR10. \vspace{-0.8em}}
		%\vspace{-0.2em}  to classification accuracy ($\%$) 
		\setlength{\tabcolsep}{6pt}
		\renewcommand{\arraystretch}{0.88}
		\label{comparisonrobustlambda} 
		%	\footnotesize
		%\scriptsize
		%	\centering
		\centering
		{ \fontsize{8.1}{3}\selectfont{
				%{\scriptsize{
				\begin{tabular}{c | c cc cc}
					\toprule
					regularization $\lambda$	& 0 & 0.25 & 0.5   &  0.75 &  1\\%& average  \\
					\midrule
					% 							&   &    \
					accuracy (\%)	& 	94.3 & 95.8  &   95.9 &  95.5  &   94.5 \\
					\toprule
				\end{tabular}
		}}
		%	\vspace{-4mm}
		\vspace{-1mm}
	\end{minipage}
	
\end{table}

\vspace{-0.4em}
\subsection{Transfer Results on Downstream Tasks} 
\textbf{Settings.} We evaluate the pretrained SANE model   on  VOC~\cite{everingham2010pascal} and COCO~\cite{lin2014microsoft}. For classification,  we train a linear classifier upon ResNet50 100 epochs by SGD.   For object detection, we use the same protocol in~\cite{he2020momentum} to fine-tune the pretrained ResNet50 based on detectron2~\cite{wu2019detectron2} for fairness. On VOC, we train  detection head with VOC07+12 trainval data and tested on VOC07 test data. On COCO, we train the head on train2017 set  and evaluate on the val2017. See  optimization settings in Appendix~\ref{morexp}.
 
\textbf{Results.}   Table~\ref{objecttable}  shows that SANE consistently outperforms
the compared state-of-the-art approaches on  both classification  and object detection tasks, and  enjoys better performance than supervised method pretrained on ImageNet. These results show the superior transferability  of  SANE.% behind which the potential reasons have been discussed in Sec.~\ref{imagenetsec}. 

\vspace{-0.4em}
\subsection{Ablation Study} \label{AblationStudy}
We train SANE 1,000 epochs on CIFAR10 to  investigate the effects of each component in SANE using  strong augmentation. Table~\ref{comparisontable1}  shows the  benefits of each    component, i.e. the label estimations $\pmii{}{}$ and $\qmii{}{}$ in   self-labeling refinery, and   momentum mixup.  Table~\ref{comparisonrobustlambda}  shows the stable performance (robustness) of SANE on CIAFR10 when regularization parameter  $\lambda$ in~\eqref{finalmodel} varies in a  large range.%,  thus testifying the robustness of SANE.

\begin{wraptable}{r}{0.5\linewidth}%靠文字内容的左侧
	%\vspace{-0.4em}
	\caption{Effects of various mixups  on  ImageNet. \vspace{-0.6em}}
	\setlength{\tabcolsep}{0.4pt}
	\renewcommand{\arraystretch}{0.88}
	\label{comparisonmixup} 
	%	\footnotesi.3ze
	%\scriptsize
	%	\centering
	\centering
	{ \fontsize{8.1}{3}\selectfont{
			%{\scriptsize{
			\begin{tabular}{c|  c|c}
				\toprule
				Accuracy ($\%$)	&	MoCo+mixup     & MoCo+momentum mixup  \\%& average  \\
				\midrule
				% 							&   &    \
				CIFAR10 (weak) &		  93.7 & 94.2 \\	
				CIFAR10 (strong) &		  93.3 & 94.8  \\	
				ImageNet (weak) &		  68.4 {\tiny{\cite{kim2020mixco}}} & 69.0  \\
				\toprule
			\end{tabular}
	}}
%	\vspace{-0.7em}
\end{wraptable}	
Then we  compare our momentum mixup~\eqref{mixup} with vanilla mixup  in the  works~\cite{kim2020mixco,lee2020mix}. Specifically, we use one-hot label in MoCo and replace $\xmti{j}$ in  ~\eqref{mixup} with the query $\xmi{j}$ to obtain ``MoCo+ mixup",  and ours with one-hot label can be viewed as ``MoCo+momentum mixup". Then we  train them 1,000 epochs on CIFAR10 with weak/strong augmentation, and 200 epochs on ImageNet with weak augmentations.   Table~\ref{comparisonmixup} shows that with weak augmentation,  momentum mixup always outperforms  vanilla mixup in~\cite{kim2020mixco,lee2020mix}.  
Moreover,    momentum mixup using strong augmentation    
improves its weak augmentation version, while vanilla mixup with strong augmentation suffers from performance  degradation.  It is because as discussed in Sec.~\ref{secmomentummixup}, momentum mixup   well reduces the possible label noise, especially for strong augmentations, and   can enhance  the performance more. 

\vspace{-0.4em}
\section{Conclusion} 
\vspace{-0.3em}
In this work, we prove the benefits of accurate labels  to
the generalization of contrastive learning. Inspired by this theory, we propose SANE to improve   label quality in contrastive learning  via self-labeling refinery and momentum mixup.  The former   uses the positive of a query to generate   informative soft labels and  combines  with vanilla one-hot label to improve label quality. The latter randomly combines queries and positives to make   virtual queries more  similar to their corresponding positives,  improving label accuracy.   Experimental results testified the advantages of SANE.

{\small
{
	\bibliographystyle{unsrt}
	\bibliography{referen}
}}

 \newpage

\appendix

\section{Structure of This Document}
	This supplementary document contains more additional experimental details and the technical proofs of convergence results  of the manuscript entitled ``A Theory-Driven Self-Labeling Refinement Method for Contrastive Representation Learning''. It is structured as follows. In Appendix~\ref{morexp}, we provides more experimental    details, including training algorithm, network architecture, optimizer details, loss construction and training cost of SANE.  Appendix~\ref{proofofapproximation} presents  the proof and details of the main results, namely,  Theorem~\ref{proximatetrisk}, in Section~\ref{analysissect}, which analyzes the generalization performance of MoCo.
	
	Next, Appendix~\ref{proofoflabelrefine} introduces the  \textit{proof roadmap} and details  of the main results, i.e. Theorem~\ref{mainresults}, in Section~\ref{sectionlabelrefine}. Since the proof framework is relatively complex, we first
	introduce some necessary preliminaries, including notations, conceptions and assumptions that are verified in subsequent analysis in Appendix~\ref{ProofsofAuxiliaryLemmasandTheories}. Then we provide the proofs of Theorem~\ref{mainresults} in Appendix~\ref{proofofmain}. Specifically, we first introduce the \textit{proof roadmap of Theorem~\ref{mainresults}} in Appendix~\ref{Proofroadmap}. Then  we present several auxiliary theories in Appendix~\ref{AuxiliaryTheories}. Next, we prove our Theorem~\ref{mainresults} in Appendix~\ref{proofofmainust22}. Finally, we present all proof details of auxiliary theories in Appendix~\ref{ProofsofAuxiliaryLemmasandTheories}.

\section{More Experimental  Details}\label{morexp}
Due to space limitation, we defer more experimental   details   to this appendix. Here
we first introduce the training algorithm of SANE, and then present more setting details of optimizers,  architectures, loss construction for CIFAR10 and ImageNet.

\subsection{Algorithm Framework of SANE}
In this subsection, we introduce the training algorithm of SANE in details, which is summarized in Algorithm~\ref{Elisealgorithm}.  Same as  MoCo~\cite{he2020momentum} and CLSA~\cite{aa2020},  we  alternatively update  the online network $\fo_{\wm}$ and target network $\ft_{\xim}$ via SGD optimizer. Our codes are implemented based on MoCo and CLSA. The code of MoCo and CLSA satisfies ``Creative Commons Attribution-NonCommercial 4.0 International Public
License".
\begin{algorithm}[H]%[t]
	\caption{Algorithm Framework for SANE}
	\label{Elisealgorithm}
	%	\setlength{\tabcolsep}{2pt}
	%	\setstretch{0.8}
	%	\renewcommand{\baselinestretch}{0.6}
	\begin{algorithmic}
		\STATE {\bfseries Input: }   online network $\fo_{\wm}$, target network $\ft_{\xim}$, dictionary $\Bm$, temperature parameter $\tau$, momentum-update parameter $\iota$, sharpness parameter $\tau'$,  prior confidence $\mu$, regularization weight $\lambda$, parameter $\kappa$ for  $\betad(\kappa,\kappa)$, weak augmentation $T_1$, and weak or strong augmentation $T_2$
		\STATE {\bfseries Initialization:}  initialize online network $\fo_{\wm}$, target network $\ft_{\xim}$, dictionary $\Bm$ as MoCo.
		\FOR{$i=1 \cdots T$}
		\STATE 1.  sample a minibatch of vanilla samples $\{\cmi{i}\}_{i=1}^s$
		\STATE  2.  use $T_1$ to augment $\{\cmi{i}\}_{i=1}^s$ to obtain weak augmentations $\{(\xmi{i},
		\xmti{i})\}_{i=1}^s$, i.e. $\xmi{i}=T_1(\cmi{i})$ and $\xmti{i}=T_1(\cmi{i})$.
		\STATE  3. compute feature $\{f(\xmi{i})\}_{i=1}^s$ and  $\Bm'=\{\ft(\xmti{i})\}_{i=1}^s$
		\STATE  4. compute the  contrastive loss $ \Lcon\big(\wm,\! \{(\xmi{i},\ymi{i})\}\big)$ in Eqn.~\eqref{finalmodel}
		\STATE 5.   use $\xmti{i}$ to compute  the estimated labels $\ymbii{t}{i}$  of query $\xmi{i}$ by self-labeling refinery~\eqref{labelcombination} ($\forall i =1, \cdots, s$)
		\STATE  6.  if using strong augmentation for momentum mixup,  use $T_2$ to augment $\{\cmi{i}\}_{i=1}^s$ for obtaining strong augmentations $\{\xmti{i}\}_{i=1}^s$ to replace the previous  $\{\xmti{i}\}_{i=1}^s$  in  $\{(\xmi{i},
		\xmti{i})\}_{i=1}^s$
		\STATE 7. 	use  momentum mixup~\eqref{mixup} and samples $\{(\xmi{i},\xmti{i},\ymbii{t}{i})\}_{i=1}^s$ to obtain new virtual queries and labels $\{(\xmfi{i},\ymfi{i})\}_{i=1}^s$
		\STATE  8.  use $\{(\xmfi{i},\ymfi{i})\}_{i=1}^s$ to compute the  momentum mixup contrastive loss $ \Lcon\big(\wm,\! \{(\xmfi{i},\ymfi{i})\}\big)$ in Eqn.~\eqref{finalmodel}
		\STATE 9. 	update online network $\fo_{\wm}$ by minimizing $ (1\!-\!\lambda) \Lcon\big(\wm,\! \{(\xmi{i},\ymi{i})\}\big) \!+\! \lambda \Lcon\big(\wm,\! \{(\xmfi{i},\ymfi{i})\}\big)$
		\STATE  10.  update target network $\ft_{\xim}$ by exponential moving average
		\STATE  11.  update the dictionary $\Bm$ via minibatch feature $B'$ in a first-in first-out order.
		\ENDFOR
		\STATE {\bfseries Output: }
	\end{algorithmic}
\end{algorithm}

\subsection{Algorithm Parameter Settings}
%\subsection{Evaluation Results on CIFAR10 and ImageNet}\label{imagenetsec}
\textbf{Experimental Settings for Linear Evaluation on CIFAR10 and ImageNet.} For CIFAR10 and ImageNet, we follow  \cite{he2020momentum,chen2020simple} and use ResNet50~\cite{he2016deep} as a backbone.  Then we first pretrain SANE on the corresponding training data, and then train a linear classifier on top of 2048-dimensional  frozen features provided by ResNet50. For pretraining on both datasets, we use  SGD  with an initial learning rate 0.03 (annealed down to zero via cosine decay~\cite{loshchilov2016sgdr}), a momentum of 0.9, and a weight decay of $10^{-4}$. Such optimizer parameters are the same with MoCo and CLSA.

Next, we pretrain 2,000 epochs on CIFAR10 with minibatch size 256 and dictionary size 4,096. For pretraining on Imagenet, the dictionary size is always 65,536;  the batch size is often 256 on a cluster of 8 GPUs and is linearly scaled together with learning rate on multiple clusters. For linear classifier training, we use ADAM~\cite{kingma2014adam} with a learning rate of 0.01  and without weight decay
to train 200 epochs on CIFAR10, and adopt  SGD with an initial learning 10 (cosine decayed to zero) and a momentum of 0.9 to train 100 epochs on ImageNet. We use standard data augmentations  in~\cite{he2020momentum} for pretraining unless otherwise stated.  Specifically, for pretraining on CIFAR10 and ImageNet, we follow MoCo and use  RandomResizedCrop,  ColorJitter, RandomGrayscale, GaussianBlur, RandomHorizontalFlip, and Normalization. For CIFAR10, please find its pretraining augmentation in the example\footnote{\url{https://colab.research.google.com/github/facebookresearch/moco/blob/colab-notebook/colab/moco_cifar10_demo.ipynb}}. Except the above random augmentation, we also use the proposed momentum mixup to generate the virtual instances for constructing the momentum mixup loss.

For CIFAR10, to fairly compare with~\cite{lee2020mix}, we crop each image into two views to construct the loss~\eqref{finalmodel}.  Specifically, for a minibatch of vanilla samples $\{\cmi{i}\}_{i=1}^s$, we  use weak augmentation $T_1$ to augment $\{\cmi{i}\}_{i=1}^s$ to obtain weak augmentations $\{(\xmi{i},\xmti{i})\}_{i=1}^s$, i.e. $\xmi{i}=T_1(\cmi{i})$ and $\xmti{i}=T_1(\cmi{i})$. Then same as MoCo, we can compute the contrastive loss by using   $\{(\xmi{i},\xmti{i})\}_{i=1}^s$. Meanwhile, we use $\xmti{i}$ to compute the soft label $\ymbii{t}{i}$ of $\xmti{i}$ via~\eqref{labelcombination}. Next, we use  momentum mixup~\eqref{mixup} and samples $\{(\xmi{i},\xmti{i},\ymbii{t}{i})\}_{i=1}^s$ to obtain new virtual queries and labels $\{(\xmfi{i},\ymfi{i})\}_{i=1}^s$, and then  use $\{(\xmfi{i},\ymfi{i})\}_{i=1}^s$ to compute the  momentum mixup contrastive loss $ \Lcon\big(\wm,\! \{(\xmfi{i},\ymfi{i})\}\big)$ in Eqn.~\eqref{finalmodel}. For strong augmentation, after we compute the vanilla contrastive loss in MoCo, and then use strong augmentation to augment $\{\cmi{i}\}_{i=1}^s$ to replace $\xmti{i}$ in  $\{(\xmi{i},\xmti{i},\ymbii{t}{i})\}_{i=1}^s$. Then we can generate virtual query instances  and their labels ($\{(\xmfi{i},\ymfi{i})\}_{i=1}^s$ ) by using  $\{(\xmi{i},\xmti{i},\ymbii{t}{i})\}_{i=1}^s$.  The training cost on CIFAR10 for 2,000 epochs is about 11 days on single V100 GPU.

For ImageNet, we follow CLSA for fair comparison. For SANE-Single, we use the same way to construct the contrastive loss, and then use  augmentation $T_1$ to augment $\{\cmi{i}\}_{i=1}^s$ to replace $\xmti{i}$ in  $\{(\xmi{i},\xmti{i},\ymbii{t}{i})\}_{i=1}^s$ to construct the momentum mixup loss. Indeed, we also can do not replace $\xmti{i}$ in  $\{(\xmi{i},\xmti{i},\ymbii{t}{i})\}_{i=1}^s$ for  momentum mixup loss, which actually did not affect the performance.  We do it, since SANE-Multi crops each image into five different crops for constructing momentum mixup loss, and thus SANE-Single and SANE-Multi will be more consistent, i.e. SANE-Multi uses 5 crops while SANE-Single uses one crop.   For strong augmentation, we replace the augmentation $T_1$ in momentum mixup with strong augmentation, which is the same on CIFAR10. As mentioned above, to construct the momentum mixup loss,  SANE-Multi crops each image into five sizes $224\!\times\! 224$,  $192\!\times\!192$, $160\!\times\!160$, $128\!\times\!128$, and $96\!\times\!96$ and averages their  momentum mixup losses. For the vanilla contrastive loss, SANE-Multi uses the same way in SANE-Single to compute. In this way, SANE-Single and SANE-Multi respectively have the same settings with CLSA-Single and CLSA-Multi.  Thus,  ELSE has almost the same training cost with CLSA,  i.e. about 75 (188) hours with 8  GPUs,  200 epochs,  batch size of 256 for SANE-Single (-Multi).  It should be mentioned that for vanilla contrastive loss in both  CLSA-Single and CLSA-Multi, we always use weak augmentations.

\textbf{Transfer Evaluation Settings.} We evaluate the pretrained model on ImageNet on  VOC~\cite{everingham2010pascal} and COCO~\cite{lin2014microsoft}. For VOC, similar to linear evaluation, we train a linear classifier upon ResNet50 100 epochs by SGD with a learning rate 0.05, a momentum 0.9, batch size 256, and without weight and learning rate decay.  For COCO, we adopt the same protocol in~\cite{he2020momentum} to fine-tune the pretrained ResNet50 based on detectron2~\cite{wu2019detectron2} for fairness.  We  evaluate the transfer ability of the  cells selected on CIFAR10 by testing them on   ImageNet. Following DARTS,  we use   momentum  SGD  with an initial learning $0.025$ (cosine decayed to zero), a momentum of 0.9,  a weight decay of $3\!\times\! 10^{-4}$,  and  gradient norm clipping parameter 5.0.

\section{Proofs of The Results in Section~\ref{analysissect}}\label{proofofapproximation}
\begin{lem}\cite{maurer2009empirical}\label{perfectrisk}
	Suppose the loss $\ell$ is bounded by the range $[a,b]$, namely $\ell(f(\xm;\wm), \ymg) \in [a, b]$. Then let $\F$ be a finite class of hypotheses $\ell(f(\xm;\wm), \ymg): \X\rightarrow \Rs{}$. Let
	\begin{equation*}
		\Qge(f) = \frac{1}{n} \sum_{i=1}^{n} \ell(f(\xmi{i};\wm), \ymi{i}),\quad  \Qgp(f) = \EE_{(\xm,\ym)\in\SSS} \left[ \ell(f(\xm;\wm), \ym)\right]
	\end{equation*}
	respectively denote the empirical and population risk, where $\SSS$ denote the unknown data distribution and the sampled dataset $\Dg=\{(\xmi{i},\ymi{i})\}_{i=1}^n \sim \SSS$ is of size $n$.  Then for any $\delta \in(0,1)$, with probability at least $1-\delta$ we have
	\begin{equation}
		\Qgp(f)  \leq \Qge(f)  + \sqrt{\frac{2 (b-a)^2 V_{\Dg} \ln(2|\F|/\delta)}{n}} + \frac{7(b-a)^2  \ln(2|\F|/\delta)}{3(n-1)},
	\end{equation}
	where $V_{\Dg}$ denotes the variance of the loss $\ell(f(\xm;\wm), \ymg) $ on the dataset $\Dg$, and $|\F|$ denotes the covering number of $\F$ in the uniform norm $\|\cdot\|_{\infty}$.
\end{lem}

\begin{lem}~\cite{liang2016deep} \label{approximation}
	For any polynomials $f(x) = \sum_{i=0}^{p} a_i x^i$, $x\in[0,]$ and $\sum_{i=1}^p |a_i|<1$, there exists a multilayer
	neural network $\hat{f}(x)$ with $\Oc{p+\log\frac{p}{\epsilon}}$ layers, $O(\log\frac{p}{\epsilon})$ binary step units and $O(p\log\frac{p}{\epsilon})$ rectifier linear units such that $|f(x) - \hat{f}(x)|\leq \epsilon,\ \forall x\in[0,1]$.\\
	Assume that function $f$ is continuous on $[0,1]$ and $\lceil \log\frac{2}{\epsilon}\rceil +1$ times differential in
	$(0,1)$. Let $f^{(n)}$ denote the derivative of $f$ of $n-$th order and $\|f\|=\max_{x\in[0,1]} f(x)$. If $\|f^{(n)}\|\leq n!$ holds for all $n \in [\lceil \log\frac{2}{\epsilon}\rceil +1]$, then there exists a deep network $f$ with  $\Oc{\log\frac{1}{\epsilon}}$ layers, $O(\log\frac{1}{\epsilon})$ binary step units and $O(\log^2\frac{1}{\epsilon})$ rectifier linear units such that $|f(x) - \hat{f}(x)|\leq \epsilon,\ \forall x\in[0,1]$.
\end{lem}
For expression power analysis of deep network,  more stronger results can be found in  \cite{lu2017expressive,telgarsky2016benefits,cohen2016expressive,eldan2016power} and all show that any function can be approximately can be approximated by a deep network to arbitrary accuracy.

\subsection{Proof of Theorem~\ref{proximatetrisk}}
\begin{proof}
	Here we use two steps to prove our results in Theorem~\ref{proximatetrisk}.
	\begin{equation}\label{empiricalrisk}
		\setlength{\abovedisplayskip}{4.5pt}
		\setlength{\belowdisplayskip}{4.5pt}
		\setlength{\abovedisplayshortskip}{4.5pt}
		\setlength{\belowdisplayshortskip}{4.5pt}
		\Qa(\fo_{\wm}) = \frac{1}{n} \sum\nolimits_{i=1}^{n} \ell(h(\fo_{\wm}(\xmi{i}), \Bmi{i}),\ymgi{i}),
	\end{equation}
	\textbf{Step 1. proof for first part results.}  To begin with, we first define an empirical risk  $\Qge(f)$:
	\begin{equation*}
		\Qge(f) = \frac{1}{n} \sum_{i=1}^{n} \ell(h(\fo_{\wm}(\xmi{i}),  \Bmi{i}),\ymsi{i}),
	\end{equation*}
	where  $\Qge(f)$ uses the ground truth label $\ymsi{i}$ for training.   	From Lemma~\ref{perfectrisk}, with probability at least $1-\delta$, we have
	\begin{equation*}
		\Qgp(f)  \leq \Qge(f)  + \sqrt{\frac{2 (b-a)^2 V_{\Dg} \ln(2|\F|/\delta)}{n}} + \frac{7(b-a)^2  \ln(2|\F|/\delta)}{3(n-1)},
	\end{equation*}
	where $\Qgp(f) $ is the population risk, and $\Qge(f)$ is the empirical risk. Both are  trained with the ground truth $\ymsi{i}$. So the remaining work is to  upper bound $ \Qge(f) $ via $ \Qa(f) $.  Towards this end, we can bound it as follows
	\begin{equation*}
		\begin{split}
			\Qge(f)  - \Qa(f) = & \frac{1}{n} \sum_{i=1}^{n} \left( \ell(h(\fo_{\wm}(\xmi{i}),  \Bmi{i}),\ymsi{i})  - \ell(h(\fo_{\wm}(\xmi{i}),  \Bmi{i}),\ymi{i}) \right)   \\
			%= & \frac{1}{n} \sum_{i=1}^{n} \int_{0}^{1} \nabla_{\ym} \ell(f(\xmi{i};\wm),\ymai{i} + \theta(\ymgi{i} - \ymai{i}))) ( \ymgi{i} - \ymai{i}) d \theta \\
			\led{172} & \frac{1}{n} \sum_{i=1}^{n}  \| \nabla_{\ym} \ell(h(\fo_{\wm}(\xmi{i}),  \Bmi{i}),\ym)  \| \cdot  \left\| \ymsi{i} - \ymi{i}\right\|_2 \\
			\led{173} &  \Ly \EE_i \left\| \ymsi{i} - \ymi{i}\right\|_2 \\
			\led{173} &  \Ly  \EE_{\Dg\sim\SSS}   \left[\left\| \yms - \ym \right\|_2\right], \\
		\end{split}
	\end{equation*}
	where \ding{172} holds by using  $\ym = \ymi{i} + \theta(\ymsi{i} - \ymi{i})$ for certain $\theta \in (0,1)$; \ding{173} holds since we use the $\Ly$-Lipschitz property of $\ell(h(\fo_{\wm}(\xmi{i}),  \Bmi{i}),\ymi{i}) $.  Then combining these results together, we can obtain the desired results:
	\begin{equation*}
		| \Qgp(f) -  \Qa(f) | \leq +\Ly  \EE_{\Dg\sim\SSS}  \left\| \ymsi{i} - \ymi{i}\right\|_2 + \sqrt{\frac{2 (b-a)^2 V_{\Dg} \ln(2|\F|/\delta)}{n}} + \frac{7(b-a)^2  \ln(2|\F|/\delta)}{3(n-1)}.
	\end{equation*}

	\textbf{Step 2. proof for second part results.}   Here we can construct a simple two-classification problem for clarity.  Suppose we have two classes: class one with training data $\D_1 = \{(\xm_1, \xm_1,\yms_1)\}_{i=1}^{n/2}$ and class two with training data $\D_2= \{(\xm_2, \xm_2, \yms_2)\}_{i=1}^{n/2}$, where $\yms_1$ denotes the ground truth label of $\xm_1$ on the set $\Bmi{1}=\{\xm_1  \cup \Bm\}$, and $\yms_2$    denotes the ground truth label of $\xm_2$ on the set $\Bmi{2}=\{\xm_2 \cup  \Bm\}$.  Both training datasets $\D_1$ and $\D_2$ have $\frac{n}{2}$ samples. Here we assume there is no data augmentation which means $\xmi{i}=\xmti{i}$ in the manuscript.   In $\D_1 $, its samples are the same, namely $(\xm_1, \xm_1,\yms_1)$. Similarly, $\D_2 $ also has  the same samples, namely $(\xm_2, \xm_2,\yms_2)$.
	Then the predicted class probability $\ymi{ij}$ of sample $\xmi{i}$ on class $j$ is as follows:
	\begin{equation}\label{classprobablity}
		\ymi{i0}= \frac{e^{\delta( \xmi{i} ,\xmi{i})/t}}{e^{\delta( \xmi{i},\xmi{i})/\tau} + \sum_{j=1}^k e^{\delta( \xmi{i},\bmi{j})/\tau}}, \quad \ymi{ij}= \frac{e^{\delta( \xmi{i} ,\bmi{j})/\tau}}{e^{\delta( \xmi{i},\xmi{i})/\tau} + \sum_{j=1}^k e^{\delta( \xmi{i},\bmi{j})/\tau}}\ (j=1,\cdots, k),
	\end{equation}
	where $\delta(\xmi{i},\xmti{i}) = -  \frac{\langle\fo(\xmi{i}),\ft(\xmti{i}) \rangle}{\|\fo(\xmi{i})\|_2 \cdot \|\ft(\xmti{i}) \|_2} $, $\tau$ denotes a temperature. For simplicity, we let  dictionary $\Bm=\{\xmi{1}, \xmi{2}\}$. In this way, we have for both ground truth label $\ymsi{1}$ and $\ymsi{2}$ that satisfy  	$\ymsi{10} = \ymsi{11}$, $\ymsi{10} + \ymsi{11} + \ymsi{12} =1$,   	$\ymsi{20} = \ymsi{22}$, $\ymsi{20} + \ymsi{21} + \ymsi{22} =1$. For this setting, here we assume the training labels are denoted by $\ymi{1}$ and $\ymi{2}$. Moreover, they satisfy  	$\ymi{10} = \ymi{11}>0$, $\ymi{10} + \ymi{11} + \ymi{12} =1$,   	$\ymi{20} = \ymi{22}>0$, $\ymi{20} + \ymi{21} + \ymi{22} =1$. The reason that we do not use one-hot labels. This is because for dictionary $\Bm=\{\xmi{1}, \xmi{2}\}$, given a sample $\xmi{i} \ (i=1,2)$,    $\xmi{i} $ needs to predict  the labels on the set $\{\xmi{i} \cup \Bm\}=\{\xmi{i}, \xmi{1}, \xmi{2}\}$, where the labels are  not one-hot obviously and satisfy $\ymi{i1}=\ymi{ii}>0$.   In the following, we will train the model on the training data  $\Da=\Da_1\cup \Da_2$ where $\Da_1 =\{(\xm_1,\xm_1, \ym_1)\}$ and  $\Da_2 =\{(\xm_2,\xm_2,\ym_2)\}$.   We use $\ymti{i}$ to denote  the model predicted label  of $\xmi{i}$.

	Then for the test samples, we assume that  half of samples are $(\xm_1, \xm_1,\yms_1)$ and remaining samples are $(\xm_2, \xm_2,\yms_2)$. Then for any network $f$,  we always have
	\begin{equation*}
		\begin{split}
			\Qgp(f)  - \Qge(f)  =    \frac{1}{n} \sum_{i=1}^{n} \left( \EE [ \ell(h(\fo_{\wm}(\xmi{i}), \Bmi{i}),\ymsi{i})  ]-   \ell(h(\fo_{\wm}(\xmi{i}), \Bmi{i}),\ymsi{i}) \right) =0.
		\end{split}
	\end{equation*}
	% 	Then assume our estimated label for the first class is $\ym_1$ which does not equal to $\yms_1$, and our    estimated label for the first class is $\ym_2$ which does not equal to $\yms_2$.  In this way,
	Then we attempt to lower bound $\Qge(f)  - \Qa(f)$.  Our training dataset is  $\Da=\Da_1\cup \Da_2$ where $\Da_1 =\{(\xm_1,\xm_1, \ym_1)\}$ and  $\Da_2 =\{(\xm_2,\xm_2,\ym_2)\}$.   	Then we discuss whether the network  $f$ can perfectly fit the labels  \eqref{classprobablity} of data $\Da$.  For both cases, our results can hold.

	\textbf{Perfectly fitting.} Network $f$   has the capacity to perfectly fit the label $\yma_1$ in $\Da_1$ and the label $\yma_2$ in $\Da_2$ when $\xm_1$ are different $\xm_2$.  In this case, we have
	\begin{equation*}
		\begin{split}
			& \Qge(f)  - \Qa(f) \\
			= & \frac{1}{n} \sum_{i=1}^{n} \left(\ell(h(\fo_{\wm}(\xmi{i}), \Bmi{i}),\ymsi{i})- \ell(h(\fo_{\wm}(\xmi{i}), \Bmi{i}),\ymi{i})\right)   \\
			%= & \frac{1}{n} \sum_{i=1}^{n} \int_{0}^{1} \nabla_{\ym} \ell(f(\xmi{i};\wm),\ymai{i} + \theta(\ymgi{i} - \ymai{i}))) ( \ymgi{i} - \ymai{i}) d \theta \\
			= & \frac{1}{n} \sum_{i=1}^{n} \sum_{s=1}^{k}(\ymsi{i,s} \log(h(\fo_{\wm}(\xmi{i}), \Bmi{i})) - \ymi{i,s} \log(h(\fo_{\wm}(\xmi{i}), \Bmi{i})) )   \\
			= & \frac{1}{n} \sum_{i=1}^{n}  \sum_{s=1}^{k}(\ymsi{i,s}- \ymi{i,s}) \log(\yma_{i,s})\\
			\lee{172} & \frac{1}{n} \sum_{i=1}^{n}  \sum_{s=1}^{k}(\ymsi{i,s}- \ymi{i,s}) \log(\ymi{i,s})\\
			= & \frac{1}{6}  \left[(\ymsi{10}- \ymi{10}) \log(\ymi{10}) + (\ymsi{11}- \ymi{11}) \log(\ymi{11}) + (\ymsi{12}- \ymi{12}) \log(\ymi{12}) \right.\\
			& \left.+(\ymsi{20}- \ymi{20}) \log(\ymi{20}) + (\ymsi{21}- \ymi{21}) \log(\ymi{21}) + (\ymsi{22}- \ymi{22}) \log(\ymi{22}) \right]\\
			= & \frac{1}{6}  \left[2(\ymsi{10}- \ymi{10}) \log(\ymi{10}) + (\ymsi{12}- \ymi{12}) \log(\ymi{12})  +2(\ymsi{20}- \ymi{20}) \log(\ymi{20})   + (\ymsi{21}- \ymi{21}) \log(\ymi{21}) \right]\\
			\lee{173} & \frac{1}{3}  \left[(\ymsi{10}- \ymi{10}) \log \frac{\ymi{10}}{1-2 \ymi{10}}  +(\ymsi{20}- \ymi{20}) \log \frac{\ymi{20}}{1-2 \ymi{20}}  \right],
		\end{split}
	\end{equation*}
	where \ding{172} holds since $\yma_{i,s}=\ymi{i,s}$, and \ding{173} uses $\ymsi{10} = \ymsi{11}$, $\ymsi{10} + \ymsi{11} + \ymsi{12} =1$,   	$\ymsi{20} = \ymsi{22}$, $\ymsi{20} + \ymsi{21} + \ymsi{22} =1$,  	$\ymi{10} = \ymi{11}$, $\ymi{10} + \ymi{11} + \ymi{12} =1$,   	$\ymi{20} = \ymi{22}$, $\ymi{20} + \ymi{21} + \ymi{22} =1$.  Then we can choose proper values such that
	\begin{equation*}
		\begin{split}
			\ymsi{10}= \ymsi{11} > \ymi{10} = \ymi{11} >\frac{1}{3},  \ymsi{20}= \ymsi{22} > \ymi{20} = \ymi{22} >\frac{1}{3}.
		\end{split}
	\end{equation*}
	For example, we can let $\ymi{1}=(0.4, 0.4, 0.2)$,  $\ymsi{1}=(0.45, 0.45, 0.1)$, $\ymi{2}=(0.4,  0.2 , 0.4)$,  $\ymsi{2}=(0.45, 0.1, 0.45)$.  In this way,  we   have $(\ymsi{10}- \ymi{10}) \log \frac{\ymi{10}}{1-2 \ymi{10}}   \geq c_1(\ymsi{10}- \ymi{10})>0$ and  $(\ymsi{20}- \ymi{20}) \log \frac{\ymi{20}}{1-2 \ymi{20}}  \geq c_2 (\ymsi{20}- \ymi{20})>0$. So this means that there exists a constant $C$ such that
	\begin{equation*}
		\begin{split}
			\Qge(f)  - \Qa(f) \geq C \cdot  \EE_{i} \left[\left\| \ymsi{i} - \ymi{i}\right\|_2\right] =  C\cdot \EE_{\Dg\sim\SSS} \left[\left\| \yms - \ym\right\|_2\right].
		\end{split}
	\end{equation*}
	So combining the above results gives the following desired result:
	\begin{equation*}
		\begin{split}
			\Qgp(f)  - \Qa(f) \geq C\cdot \EE_{\Dg\sim\SSS} \left[\left\| \yms - \ym\right\|_2\right].
		\end{split}
	\end{equation*}

	\textbf{Non-perfectly fitting.} From Lemma~\ref{approximation}  (other more results in \cite{lu2017expressive,telgarsky2016benefits,cohen2016expressive,eldan2016power}), one can  approximate any function by a deep network to arbitrary accuracy. Specifically, for the polynomial function in Eqn.~\eqref{classprobablity},  there exists a multilayer neural network $\hat{f}(x)$ with proper width and depth such that $\|\ymi{1}-\yma_1\|_{1} \leq \epsilon$ and $\|\ymi{2}-\yma_2\|_{1} \leq \epsilon$, where $\yma_1$ and $\yma_2$ are the predicted labels of samples $\xmi{1}$ and $\xmi{2}$ by using \eqref{classprobablity}. The labels $\ym_1$ and $\ym_2$ are associated with our training dataset   $\Da=\Da_1\cup \Da_2$ where $\Da_1 =\{(\xm_1,\xm_1, \ym_1)\}$ and  $\Da_2 =\{(\xm_2,\xm_2,\ym_2)\}$.
	In this case, we have
	\begin{equation*}
		\begin{split}
			& \Qge(f)  - \Qa(f)\\
			= & \frac{1}{n} \sum_{i=1}^{n} \left( \ell(h(\fo_{\wm}(\xmi{i}), \Bmi{i}),\ymsi{i}) - \ell(h(\fo_{\wm}(\xmi{i}), \Bmi{i}),\ymi{i}) \right)   \\
			= & \frac{1}{n} \sum_{i=1}^{n} \sum_{s=1}^{k}(\ymsi{i,s} \log(h(\fo_{\wm}(\xmi{i}), \Bmi{i})) - \ymi{i,s} \log(h(\fo_{\wm}(\xmi{i}), \Bmi{i})) )   \\
			= & \frac{1}{n} \sum_{i=1}^{n}  \sum_{s=1}^{k}(\ymsi{i,s}- \ymi{i,s}) \log(\ymai{i,s})\\
			= & \frac{1}{6}  \left[(\ymsi{10}- \ymi{10}) \log(\ymai{10}) + (\ymsi{11}- \ymi{11}) \log(\ymai{11}) + (\ymsi{12}- \ymi{12}) \log(\ymai{12}) \right.\\
			& \left.+(\ymsi{20}- \ymi{20}) \log(\ymai{20}) + (\ymsi{21}- \ymai{21}) \log(\ymai{21}) + (\ymsi{22}- \ymi{22}) \log(\ymai{22}) \right]\\
			= & \frac{1}{6}  \left[2(\ymsi{10}- \ymi{10}) \log(\ymai{10}) + (\ymsi{12}- \ymi{12}) \log(\ymai{12})  +2(\ymsi{20}- \ymi{20}) \log(\ymai{20})   + (\ymsi{21}- \ymi{21}) \log(\ymai{21}) \right]\\
			\lee{172} & \frac{1}{3}  \left[(\ymsi{10}- \ymi{10}) \log \frac{\ymai{10}}{1-2 \ymai{10}}  +(\ymsi{20}- \ymi{20}) \log \frac{\ymai{20}}{1-2 \ymai{20}}  \right],
		\end{split}
	\end{equation*}
	where \ding{172} uses $\ymsi{10} = \ymsi{11}$, $\ymsi{10} + \ymsi{11} + \ymsi{12} =1$,   	$\ymsi{20} = \ymsi{21}$, $\ymsi{20} + \ymsi{21} + \ymsi{22} =1$,  	$\ymi{10} = \ymi{11}$, $\ymi{10} + \ymi{11} + \ymi{12} =1$,   	$\ymi{20} = \ymi{22}$, $\ymi{20} + \ymi{21} + \ymi{22} =1$.  Then we can choose proper values such that
	\begin{equation*}
		\begin{split}
			\ymsi{10}= \ymsi{11} > \ymi{10} = \ymi{11} >\frac{1}{3} + \epsilon,  \ymsi{20}= \ymsi{22} > \ymi{20} = \ymi{22} >\frac{1}{3} + \epsilon.
		\end{split}
	\end{equation*}
	For example, we can let $\ymi{1}=(0.4, 0.4, 0.2)$,  $\ymsi{1}=(0.45, 0.45, 0.1)$, $\ymi{2}=(0.4, 0.2, 0.4)$,  $\ymsi{2}=(0.45, 0.1, 0.45)$, and $ \epsilon=0.0001$.  In this way,  we   have $(\ymsi{10}- \ymi{10}) \log \frac{\ymai{10}}{1-2 \ymai{10}}   \geq c_1(\ymsi{10}- \ymi{10})>0$ and  $(\ymsi{20}- \ymi{20}) \log \frac{\ymai{20}}{1-2 \ymai{20}}  \geq c_2 (\ymsi{20}- \ymi{20})>0$. So this means that there exists a constant $C$ such that
	\begin{equation*}
		\begin{split}
			\Qge(f)  - \Qa(f) \geq C \cdot  \EE_{i} \left[\left\| \ymsi{i} - \ymi{i}\right\|_2\right] =  C\cdot \EE_{\Dg\sim\SSS} \left[\left\| \yms - \ym\right\|_2\right].
		\end{split}
	\end{equation*}
	So combining the above results gives the following desired result:
	\begin{equation*}
		\begin{split}
			\Qgp(f)  - \Qa(f) \geq C\cdot \EE_{\Dg\sim\SSS} \left[\left\| \yms - \ym\right\|_2\right].
		\end{split}
	\end{equation*}
	
	The proof is completed.
\end{proof}

\section{Proof of Results in Section~\ref{sectionlabelrefine}}\label{proofoflabelrefine}
In this section, we first introduce some necessary preliminaries, including notations, conceptions and assumptions that are verified in subseqent analysis in Appendix~\ref{ProofsofAuxiliaryLemmasandTheories}. Then we provide the proofs of Theorem~\ref{mainresults} in Appendix~\ref{proofofmain}. Specifically, we first introduce the proof roadmap in Appendix~\ref{Proofroadmap}. Then  we present several auxiliary theories in Appendix~\ref{AuxiliaryTheories}. Next, we prove our Theorem~\ref{mainresults} in Appendix~\ref{proofofmainust22}. Finally, we present all proof details of auxiliary theories in Appendix~\ref{AuxiliaryTheories}.

\subsection{Preliminaries}\label{sec gen thy}
\subsubsection{General Model Formulation}\label{generalmodel}
In this section, we outline our approach to proving robustness of overparameterized neural networks. Towards this goal, we consider a general formulation where we aim to fit a general nonlinear model of the form $\xm \mapsto f(\wm,\xm)$ with $\wm\in\R^p$ denoting the parameters of the model. For instance in the case of neural networks $\wm$ represents its weights. Given a data set of $n$ input/label pairs $\{(\xmi{i},\ymi{i})\}_{i=1}^n\subset\Rs{d}\times \R$, we fit to this data by minimizing a nonlinear least-squares loss of the form%as follows% and the problem
\begin{align*}
	%\min_{\tn{\bteta}\leq R}\sum_{i=1}^n w_i(y_i-\x_i^T\bteta)^2.%\tn{\y-\X\bteta}^2.
	\Lt{t}(\wm)=\frac{1}{2}\sum_{i=1}^n (\ymbii{t}{i} -f(\wm,\xmi{i}))^2.%+\la\tn{\bteta}^2.
\end{align*}
%\pz{where $\ymbii{t}{i} = \alphai{t} \ymi{i} + (1-\alphai{t}) f(\wmi{t},\xmi{i})$ }
where $\ymbii{t}{i} = (1-\alphai{t}) \ymi{i} + \alphai{t} \pmii{t}{}=  (1-\alphai{t}) \ymi{i} + \alphai{t}f(\wmi{t},\xmi{i})$
denotes the estimated label of sample $\xmi{i}$.  In Assumption~\ref{assumption2}   we assume $\betai{t} \!=\!0$ and $\tau'=1$ for simplicity, since performing nonlinear mapping on network output greatly increases analysis difficulty. But  we will show that   even though $\betai{t} \!=\!0$ and $\tau'=1$,  our refinery~\eqref{labelcombination} is  still sufficient to refine labels.   It can also be written in the more compact form
\begin{align}\label{generalproblems}
	\Lt{t}(\wm)=\frac{1}{2}\twonorm{f(\wm)-\ymbii{t}{}}^2\quad\text{with}\quad f(\wm):=
	\begin{bmatrix}
		f(\wm,\xmi{1})\\
		f(\wm,\xmi{2})\\
		\vdots\\
		f(\wm,\xmi{n})
	\end{bmatrix}.
\end{align}
%Here, $\bteta\in\R^p$ denotes our model.
To solve this problem we run gradient descent iterations with a constant learning rate $\eta$ starting from an initial point $\wmi{0}$. These iterations take the form
\begin{align}
	\label{GD}
	\wmi{t+1}=\wmi{t}-\eta\nabla\Lt{t}(\wmi{t})\quad\text{with}\quad\nabla\LL(\wm)=\J^T(\wm)\left(f(\wm)-\ymbii{t}{}\right).
\end{align}
Here,  $\J(\wm)$ is the $n\times p$ Jacobian matrix associated with the nonlinear mapping $f$ defined via
\begin{align}
	\J(\wm)=\begin{bmatrix}\frac{\pa f(\wm,\xmi{1})}{\pa \wm}~\dots~\frac{\pa f(\wm,\xmi{n})}{\pa \wm}\end{bmatrix}^T.\label{jacob eq}
\end{align}

Define the $n$-dimensional residual vector $\rm$ and corrupted  residual vector  $\emm$ where
\begin{equation*}
	\rmi{t}=\rmi{t}(\wm)=\begin{bmatrix}f(\xmi{1},\wmi{t})-\ymbii{t}{1} &\ldots&f(\xmi{n},\wmi{t})-\ymbii{t}{n}\end{bmatrix}^T\quad \text{and} \quad \emi{t} = \ymbii{t}{}-\yms .
\end{equation*}
A key idea in our approach is that we argue that (1) in the absence of any corruption $\rmm(\wm)$ approximately lies on the subspace $\Scp$ and (2) if the labels are corrupted by a vector $\emm$, then $\emm$ approximately lies on the complement space.

Throughout, $\sigma_{\min}(\cdot)$ denotes the smallest singular value of a given matrix. We first introduce helpful definitions that will be used in our proofs. Given a matrix $\Xm\in\R^{n\times d}$ and a subspace $\SSS \subset\R^n$, we define the minimum singular value of the matrix over this subspace by $\sigma_{\min}(\Xm,\SSS)$ which is defined as
\[
\sigma_{\min}(\Xm,\SSS)=\sup_{\norm{\vm}=1,\Um\Um^T=\Pro_{\SSS}} \norm{\vm^T \Um^T \Xm}.
\]
Here, $\Pro_{\SSS}\in\R^{n\times n}$ is the projection operator to the subspace. Hence, this definition essentially projects the matrix on $\SSS$ and then takes the minimum singular value over that projected subspace.

Since augmentations are produced by using the vanilla sample $\cmi{i}$ and the augmentation $\xm$ obeys $\norm{\xm -\cmi{i}}\leq \epsilon_0$. So in this sense, we often call the vanilla sample and its augmentations as cluster, and  call the vanilla sample as cluster center.

\subsubsection{Definitions and Assumptions}
To begin with, we define $(\eps,\delta)$-clusterable dataset. As aforementioned,  we often call the vanilla sample and its augmentations as cluster, and  call the vanilla sample as cluster center, because  augmentations are produced by using the vanilla sample $\cmi{i}$ and the augmentation $\xm$ obeys $\norm{\xm -\cmi{i}}\leq \epsilon_0$.

\begin{defn} [$(\eps,\delta)$-clusterable dataset] \label{clusterdata}  Suppose $\{(\xmi{i},\ymsi{i})\}_{i=1}^n$ denote the pairs of augmentation  and ground-truth label, where augmentation $\xmi{i}$ generated from  the $t$-th sample $\cmi{t}$ obeys $\norm{\xm-\cmi{t}}\!\le\! \eps$ with a constant $\eps$, and  $\ymsi{i}\!\in\!\{\gamma_1,\gamma_2,\ldots,\gamma_{\bar{K}}\}$ of $\xmi{i}$ is the label of $\cmi{t}$.
	Moreover,   samples and its augmentations are normalized, i.e.  $\norm{\cmi{i}}\!=\!\norm{\xmi{i}}\!=\!1$.  Each vanilla sample $\cmi{i}$ has $n_i$ augmentations, where $c_{l}\frac{n}{K}\le n_i\le c_{u}\frac{n}{K}$ with two constants $c_{l}$ and $c_{u}$.
	%Each $\xm$ that belongs to the $i$-th $\cmi{i}$ obey $\norm{\xm-\cmi{\ell}}\le \eps$ with a constant $\eps>0$.
	%	The labels $\ymi{i}\in\{\alpha_1,\alpha_2,\ldots,\alpha_{\bar{K}}\}$ belong to one of $\bar{K}\le K$ classes, where   $\{\alpha_\ell\}_{\ell=1}^{\bar{K}}\in[-1,1]$ denotes the labels associated with each class.
	Moreover, the classes are separated such that
	\begin{align}
		|\gamma_r-\gamma_s|\geq \delta, \quad \norm{\cmi{r}-\cmi{s}}\ge 2\eps,\ \ (\forall r\neq s),\nn
		%\\&\tn{\cb_i-\cb_j}\geq 2\eps.
	\end{align}
	where  $\delta$ is  the label separation.
\end{defn}\vspace{-6pt}

Our approach is based on the hypothesis that the nonlinear model has a Jacobian matrix with {\em{bimodal spectrum where few singular values are large and remaining singular values are small}}. This assumption is inspired by the fact that realistic datasets are clusterable in a proper, possibly nonlinear, representation space. Indeed, one may argue that one reason for using neural networks is to automate the learning of such a representation (essentially the input to the softmax layer). We formalize the notion of bimodal spectrum below.
\begin{assum}[Bimodal Jacobian] \label{lrank2} Let $\bp\geq\bn\geq \be>0$ be scalars. Let $f:\Rs{p}\rightarrow\Rs{n}$ be a nonlinear mapping and consider a set $\mathcal{D}\subset\Rs{p}$ containing the initial point $\wmi{0}$ (i.e.~$\wmi{0}\in\mathcal{D}$). Let $\Scp\subset\R^n$ be a subspace  and $\Scn$ be its complement. We say the mapping $f$ has a Bimodal Jacobian with respect to the complementary subpspaces $\Scp$ and $\Scn$ as long as the following two assumptions hold for all $\wm\in\mathcal{D}$.% we have% the conditionLipschitzness condition holds over any compact domain (for possibly large $\el$)
	\begin{itemize}
		\item {\bf{Spectrum over $\Scp$:}} For all $\vm\in\Scp$ with unit Euclidian norm we have
		\[
		\bn \leq \twonorm{\J^T(\wm)\vm}\leq \bp.
		\]
		\item {\bf{Spectrum over $\Scn$:}} For all $\vm\in\Scn$ with unit Euclidian norm we have
		\[
		\twonorm{\J^T(\wm)\vm}\leq \be.
		\]
	\end{itemize}
	We will refer to $\Scp$ as the {\em{signal subspace}} and $\Scn$ as the {\em{noise subspace}}.
\end{assum}
When $\epsilon<<\alpha$ the Jacobian is approximately low-rank. An extreme special case of this assumption is where $\epsilon=0$ so that the Jacobian matrix is exactly low-rank. We formalize this assumption below for later reference.

\begin{assum}[Low-rank Jacobian] \label{lrank} Let $\bp\geq\bn>0$ be scalars. Consider a set $\mathcal{D}\subset\R^p$ containing the initial point $\wmi{0}$ (i.e.~$\wmi{0}\in\mathcal{D}$). Let $\Scp\subset\R^n$ be a subspace and $\Scn$ be its complement. For all $\wm\in\mathcal{D}$, $\vm\in\Scp$ and $\vm'\in\Scn$ with unit Euclidian norm, we have that
	\begin{align}
		\bn\leq \twonorm{\J^T(\wm)\vm}\leq \bp\quad\text{and}\quad \twonorm{\J^T(\wm)\vm'}=0. \nn
	\end{align}
	%\footnote{}.
\end{assum}
In Theorem~\ref{mainthmrobust}, we verify that the Jacobian matrix of real datasets indeed have a bimodal structure i.e.~there are few large singular values and the remaining singular values are small which further motivate Assumption \ref{lrank}. This is inline with earlier papers which observed that Hessian matrices of deep networks have bimodal spectrum (approximately low-rank) \cite{sagun2017empirical} and is related to various results demonstrating that there are flat directions in the loss landscape \cite{hochreiter1997flat}.

Our dataset model in Definition \ref{cdata} naturally has a low-rank Jacobian when $\epsilon_0=0$ and each augmentation is equal to one of the $K$ centers (vanilla samples) $\{\cmi{\ell}\}_{\ell=1}^K$. In this case, the Jacobian will be at most rank $K$ since each row will be in the span of $\big\{\frac{\pa f(\cmi{\ell},\wm)}{\pa\wm}\big\}_{\ell=1}^K$. The subspace $\Scp$ is dictated by the {\em{membership}} of each cluster center (vanilla example) as follows: Let $\Lambda_{\ell}\subset\{1,\dots,n\}$ be the set of coordinates $i$ such that $\xmi{i}=\cmi{\ell}$. Then, subspace is characterized by $
\Scp=\{\vm\in\Rs{n} \big| \vmi{i_1}=\vmi{i_2}~~\text{for all}~~ i_1,i_2\in\Lambda_{\ell}~~\text{and}~~1\leq \ell\leq K\}.$ %n approximately low-rank
When $\epsilon_0>0$ and the augmentation points of each cluster (vanilla sample )  are not the same as the cluster we have the bimodal Jacobian structure of Assumption \ref{lrank2} where over $\Scn$ the spectral norm is small but nonzero.
\begin{defn}[Support subspace] \label{supp space}Let $\{\xmi{i}\}_{i=1}^n$ be an input dataset generated according to Definition \ref{cdata}. Also let $\{\xmti{i}\}_{i=1}^n$ be the associated vanilla samples, that is, $\xmti{i}=\cmi{\ell}$ iff $\xmi{i}$ is from the $\ell$th vanilla sample. We define the support subspace $\Scp$ as a subspace of dimension $K$, dictated by the cluster center membership as follows. Let $\Lambda_{\ell}\subset\{1,\dots,n\}$ be the set of coordinates $i$ such that $\xmti{i}=\cmi{\ell}$. Then, $\Scp$ is characterized by
	\[
	\Scp=\{\vm\in\Rs{n} \big| \vmi{i_1}=\vmi{i_2}\quad\text{for all}\quad i_1,i_2\in\Lambda_{\ell}\quad\text{and for all}~1\leq \ell\leq K\}.
	\]
\end{defn}

Before we state our general result we need to discuss another assumption and definition.
\begin{assum}[Smoothness] \label{spert}   The Jacobian mapping $\J(\wm)$ associated to a nonlinear mapping $f:\Rs{p}\rightarrow\Rs{n}$ is $L$-smooth if for all $\wmi{1}, \wmi{2}\in\Rs{p}$ we have $\opnorm{\J(\wmi{2})-\J(\wmi{1})}\le L \twonorm{\wmi{2}-\wmi{1}}$.
	%\footnote{}.
\end{assum}
In Theorem~\ref{mainthmrobust}, we verify this assumption.
Note that, if $\frac{\partial \J(\wm)}{\partial \wm}$ is continuous, the smoothness condition holds over any compact domain (albeit for a possibly large $L$.

Additionally, to connect our results to the number of corrupted labels, we introduce the notion of subspace diffusedness defined below.%of the elements of the signal space described next. These can be verified for our clusterable dataset model.% we used for neural network training.% These are Our second assumption is a smoothness condition on Jacobian as follows.
%[Smoothness]
\begin{defn}[Diffusedness] \label{diff scp} $\Scp$ is $\zeta$ diffused if for any vector $\vm\in\Scp$
	\begin{align*}
		\tin{\vm}\leq\sqrt{\zeta/n}\norm{\vm},
	\end{align*}
	holds for some $\zeta>0$.
\end{defn}
We begin by defining the average Jacobian which will be used throughout our analysis.
\begin{defn} [Average Jacobian] \label{avg jacob}We define the average Jacobian along the path connecting two points $\xm,\ym\in\Rs{p}$ as
	\begin{align}
		&\J(\ym,\xm):=\int_0^1 \J(\xm+\alpha(\ym-\xm))d\alpha.\nn
	\end{align}
\end{defn}

\begin{defn}[Neural Net Jacobian] \label{nnj def}Given input samples $(\xmi{i})_{i=1}^n$, form the input matrix $\Xm=[\xmi{1}~\dots~\xmi{n}]^T\in\R^{n\times d}$. The Jacobian of our learning problem, i.e. $
	\xm \mapsto f(\Wm,\xm)=\vm^T\phi(\Wm\xm)$  and $\Lt{t}(\Wm)=\frac{1}{2}\sum_{i=1}^n (\ymi{ti}-f(\Wm,\xmi{i}))^2$, at a matrix $\Wm$ is denoted by $\J(\Wm,\Xm)\in\R^{n\times kd}$ and is given by
	\[
	\J(\Wm,\Xm)^T=(\text{diag}(\vm)\phi'(\Wm\Xm^T))*\Xm^T.
	\]
	Here $*$ denotes the Khatri-Rao product.
\end{defn}

\subsubsection{Auxiliary Lemmas}
\begin{lem} [Linearization of the residual]\label{lin res} For the general problem~\eqref{generalproblems} in Appendix~\ref{generalmodel}, we define
	\[
	\Gm(\wmi{t})=\J(\wmi{t+1},\wmi{t})\J(\wmi{t})^T.
	\]
	where $\J(\wmi{t})$ denotes the Jacobian matrix defined in Eqn.~\eqref{jacob eq}, and $\J(\wmi{t+1},\wmi{t})=\int_0^1 \J(\wmi{t}+\alpha(\wmi{t+1}-\wmi{t}))d\alpha$ denotes the average Jacobian matrix defined in Definition~\eqref{avg jacob}.
	When using the  gradient descent iterate $\wmi{t+1}=\wmi{t}-\eta\gradt{t}{\wmi{t}}$,
	then residuals
	$$\rmi{t+1}=f(\wmi{t+1})-\ymbii{t+1}{}, \quad \rmi{t}=f(\wmi{t})-\ymbii{t}{}$$ obey the following equation
	\[
	\rmi{t+1}=(\Imm-\eta\Gm(\wmi{t}))\rmi{t} + \ymbii{t}{}-\ymbii{t+1}{}.
	\]
\end{lem}
\begin{proof} Here we follow~\cite{li2020gradient} to prove our result. 	Following Definition \ref{avg jacob}, denoting $\rmi{t+1}=f(\wmi{t+1})-\ymbii{t+1}{}$ and $\rmi{t}=f(\wmi{t})-\ymbii{t}{}$ , we find that
	\begin{align}
		\rmi{t+1}=&\rmi{t}-f(\wmi{t})+f(\wmi{t+1}) + \ymbii{t}{} -\ymbii{t+1}{}\nn\\
		\lee{172} &\rmi{t}+\J(\wmi{t+1},\wmi{t})(\wmi{t+1}-\wmi{t}) + \ymbii{t}{} -\ymbii{t+1}{}\nn\\
		\lee{173}&\rmi{t} -\eta\J(\wmi{t+1},\wmi{t})\J(\wmi{t})^T\rmi{t} +\ymbii{t}{} -\ymbii{t+1}{}\nn\\
		%&=(\Iden-\eta\Jc(\bteta_{i},\bteta_{i})\Jc(\bteta_{i})^T)\rb_{i}\\
		=&~(\Imm-\eta\Gm(\wmi{t}))\rmi{t} + \ymbii{t}{} -\ymbii{t+1}{}. \nn
	\end{align}
	where \ding{172} uses the fact that Jacobian is the derivative of $f$ and \ding{173} uses the fact that $\gradt{t}{\wm}=\J(\wm)^T\rmi{t}$.
\end{proof}
Using Assumption \ref{diff scp}, one can show that sparse vectors have small projection on $\Scp$.
\begin{lem}~\cite{li2020gradient}\label{lem sp proj} Suppose Assumption \ref{diff scp} holds. If $\rmm\in\R^n$ is a vector with $s$ nonzero entries, we have that
	\begin{align}
		\normin{\Pro_{\Scp}(\rmm)}\leq  \frac{\zeta\sqrt{s}}{n}\norm{\rmm}, \nn
	\end{align}
	where $\Pro_{\Scp}(\rmm)$ projects $\rmm$ onto the space $\Scp$.
\end{lem}
\begin{lem}\label{next inside} For the general problem~\eqref{generalproblems} in Appendix~\ref{generalmodel}, let $\rmi{t}=f(\wmi{t})-\ymbii{t}{} $ and $\rmbi{t}=\Pro_{\Scp}(\rmi{t})$. Suppose  Assumption \ref{lrank} holds  and $\eta\le \frac{1}{\bp^2}$. If $\norm{\wmi{t}-\wmi{0}}+\frac{\norm{\rmbi{t}}}{\bn}\leq \frac{4(1+\psi)\norm{\rmi{0}}}{\bn}$, then $$\wmi{t+1}\in\D=\big\{\wm\in\Rs{p} \ \big| \ \norm{\wm - \wmi{0}} \leq  \frac{4(1+\psi)\norm{\rmi{0}}}{\bn} \big\}.$$
\end{lem} %and $\eta\leq 1/\bp^2$
\begin{proof}
	Since range space of Jacobian is in $\Scp$ and $\eta\leq 1/\beta^2$, we can easily   obtain
	\begin{equation*}
		\label{lem85temp2}
		\begin{split}
			\norm{\wmi{t+1}-\wmi{t}}&=\eta \norm{\J^T(\wmi{t})\left(f(\wmi{t})-\ymbii{t}{}\right)}\\
			&\lee{172} \eta \norm{\J^T(\wmi{t})\left(\Pro_{\Scp}(f(\wmi{t})-\ymbii{t}{})\right)}\\
			&\lee{173}\eta \norm{\J^T(\wmi{t})\rmbi{t}}\\
			&\led{174} \eta \bp\norm{\rmbi{t}}\\
			&\led{175} \frac{\norm{\rmbi{t}}}{\bp}\\
			&\led{176} \frac{\norm{\rmbi{t}}}{\bn}
		\end{split}
	\end{equation*}
	In the above, \ding{172} follows from the fact that row range space of Jacobian is subset of $\Scp$ via Assumption \ref{lrank}. \ding{173} follows from the definition of $\rmbi{t}=\Pro_{\Scp}(f(\wmi{t})-\ymbii{t}{})$. \ding{174} follows from the upper bound on the spectral norm of the Jacobian over $\mathcal{D}$ per Assumption \ref{lrank}, \ding{175} from the fact that $\eta\le \frac{1}{\bp^2}$, \ding{176} from $\bn\le\bp$. The latter combined with the triangular inequality and the assumption
	\begin{align*}
		\norm{\wmi{t+1}-\wmi{0}}\leq \norm{\wmi{t+1}-\wmi{t}}+\norm{\wmi{0}-\wmi{t}}\le \norm{\wmi{t}-\wmi{0}}+\frac{\norm{\rmbi{t}}}{\bn}\leq \frac{4(1+\psi)\norm{\rmi{0}}}{\bn},
	\end{align*}
	concluding the proof of $\rmi{t+1}\in \D$.
\end{proof}

\begin{lem}\label{propertydefinitematrix}~\cite{li2020gradient} Let $\Pro_{\Scp}\in\R^{n\times n}$ be the projection matrix to $\Scp$ i.e.~it is a positive semi-definite matrix whose eigenvectors over $\Scp$ is $1$ and its complement is $0$. Let $\rmi{t}=f(\wmi{t})-\ymi{t}$, $\rmbi{t}=\Pro_{\Scp}(\rmi{t})$, and $\Gm(\wmi{t})=\J(\wmi{t+1},\wmi{t})\J(\wmi{t})^T$. Suppose Assumptions~\ref{lrank} and~\ref{spert} hold,  the learning rate $\eta$ satisfies $\eta \leq \frac{\alpha}{L\beta \norm{\rmi{0}}}$, $\norm{\rmbi{t}} \leq \norm{\rmbi{0}}$, then it holds
	\begin{align}
		{\bp^2} \Pro_{\Scp}\succeq \Gm(\wmi{t})\succeq \frac{1}{2}\J(\wmi{t})\J(\wmi{t})^T\succeq \frac{\bn^2}{2} \Pro_{\Scp}. \nn
	\end{align}
\end{lem}
%\begin{proof}
%	The $\beta^2$ upper bound directly follows from Assumption \ref{lrank} by again noticing range space of Jacobian is subset of $\Scp$.
%	\begin{lem} [Asymmetric PSD perturbation]\label{asym pert} Consider the matrices $\Am,\Gm\in\R^{n\times p}$ obeying $\|\Am-\Gm\|\leq \alpha/2$. Also suppose $\Gm\Gm^T\succeq \alpha^2\Pro_{\Scp}$. Furthermore, assume range spaces of $\Am,\Gm$ lies in $\Scp$. Then,% for all $\rb\in\Scp$,
%		\[
%		\Am\Gm^T\succeq \frac{\Gm\Gm^T}{2}\succeq \frac{\alpha^2}{2}\Pro_{\Scp}.
%		\]
%	\end{lem}
%\end{proof}
In the above context, we focus on introducing theoretical results for the general problem~\eqref{generalproblems} in Appendix~\ref{generalmodel}. Now we introduce lemmas and theories for our network learning problem, i.e. $
\xm \mapsto f(\Wm,\xm)=\vm^T\phi(\Wm\xm)$  and $\Lt{t}(\Wm)=\frac{1}{2}\sum_{i=1}^n (\ymbii{t}{i}-f(\Wm,\xmi{i}))^2$ used in our manuscript. Specifically,  we introduce some theoretical results in \cite{anon2019overparam} and characterizes three key properties of the neural network Jacobian. These are smoothness, spectral norm, and minimum singular value at initialization which correspond to Lemmas 6.6, 6.7, and 6.8 in that paper.
\begin{thm}[Jacobian Properties at Cluster Center]\cite{anon2019overparam}\label{JLlem} Suppose $\Xm=[\xmi{1}~\dots~\xmi{n}]^T\in\R^{n\times d}$ be an input dataset satisfying $\lambda(\Xm)>0$, where $\lambda(\Xm)$ denotes the smallest eigenvalue of matrix $\Xm$. Suppose $|\phi'|,|\phi''|\leq \Gamma$ where $\phi'$ and $\phi''$ respectively denotes the first and second order derivatives. The Jacobian mapping with respect to the input-to-hidden weights obey the following properties. Let $\J(\Wm,\Xm)$ denote the neural net Jacobian defined in Definition~\ref{nnj def}.
	
	\begin{itemize}
		\item[(1)] Smoothness is bounded by
		\begin{align*}
			\opnorm{\J(\widetilde{\Wm},\Xm)-\J(\Wm,\Xm)}\le \frac{\Gamma}{\sqrt{k}}\opnorm{\Xm}\fronorm{\Wmt-\Wm}\quad\text{for all}\quad \Wmt,\Wm\in\R^{k\times d}.
		\end{align*}
		\item[(2)] Top singular value is bounded by
		\begin{align*}
			\opnorm{\J(\Wm,\Xm)}\le \Gamma\opnorm{\Xm}.
		\end{align*}
		\item[(3)] Let $C>0$ be an absolute constant. As long as
		\begin{align*}
			k\ge \frac{C\Gamma^2{\log n\opnorm{\Xm}^2}}{\lambda(\Xm)}
		\end{align*}
		At random Gaussian initialization $\Wmi{0}\sim\N(0,1)^{k\times d}$, with probability at least $1-1/K^{100}$, we have
		\begin{align*}
			\sigma_{\min}\left(\J(\Wmi{0},\Xm)\right)\ge  \sqrt{\lambda(\Xm)/2}.%\frac{1}{\sqrt{2}}\mu_2(\phi)\sigma_{\min}\left(\X*\X\right),
		\end{align*}
	\end{itemize}
\end{thm}

%In our case, the Jacobian is {\bf{not}} well-conditioned. However, it is pretty well-structured as described previously.
% To proceed, given a matrix $\Xm\in\R^{n\times d}$ and a subspace $\SSS \subset\R^n$, we define the minimum singular value of the matrix over this subspace by $\sigma_{\min}(\Xm,\SSS)$ which is defined as
%\[
%\sigma_{\min}(\Xm,\SSS)=\sup_{\norm{\vm}=1,\Um\Um^T=\Pro_{\SSS}} \norm{\vm^T \Um^T \Xm}.
%\]
%Here, $\Pro_{\SSS}\in\R^{n\times n}$ is the projection operator to the subspace. Hence, this definition essentially projects the matrix on $\SSS$ and then takes the minimum singular value over that projected subspace.
The following theorem states the properties of the Jacobian at a  $(\epsilon_0,\delta)$ clusterable dataset defined in Definition~\ref{clusterdata}.  That is,  $(\xmi{i})_{i=1}^n$ are generated from $(\cmi{i})_{i=1}^K$, and their augmentation distance is at most $\epsilon_0$ and label separation is at least $\delta$.

\begin{thm}[Jacobian Properties at Cluster Center]\cite{li2020gradient}\label{JLcor} Let input samples $(\xmi{i})_{i=1}^n$ be generated according to $(\epsilon_0,\delta)$ clusterable dataset model of Definition \ref{clusterdata}.   Define $\Xm=[\xmi{1}~\dots~\xmi{n}]^T$ and $\Cm=[\cmi{1}~\dots~\cmi{k}]^T$. Let $\Scp$ be the support space and $(\xmti{i})_{i=1}^n$ be the associated clean dataset as described by Definition \ref{supp space}. Set $\Xmt=[\xmti{1}~\dots~\xmti{n}]^T$. Assume $|\phi'|,|\phi''|\leq \Gamma$ and $\lambda(\Cm)>0$. Let $\J(\Wm,\Xm)$ denote the neural net Jacobian defined in Definition~\ref{nnj def}.  The Jacobian mapping at $\Xmt$ with respect to the input-to-hidden weights obey the following properties.
	\begin{itemize}
		\item[(1)] Smoothness is bounded by
		\begin{align*}
			\opnorm{J(\Wmt,\Xmt)-\J(\Wm,\Xmt)}\le \Gamma\sqrt{\frac{c_{up} n}{{kK}}}\opnorm{\Cm}\fronorm{\Wmt-\Wm}\quad\text{for all}\quad \Wmt,\Wm\in\R^{k\times d}.
		\end{align*}
		\item[(2)] Top singular value is bounded by
		\begin{align*}
			\opnorm{\J(\Wm,\Xmt)}\le \sqrt{\frac{c_{up}n}{K}}\Gamma\opnorm{\Cm}.
		\end{align*}
		\item[(3)] As long as
		\begin{align*}
			k\ge \frac{C \Gamma^2{\log K\opnorm{\Cm}}^2}{\lambda(\Cm)}
		\end{align*}
		At random Gaussian initialization $\Wmi{0}\sim\N(0,1)^{k\times d}$, with probability at least $1-1/K^{100}$, we have
		\begin{align*}
			\sigma_{\min}\left(\J(\Wmi{0},\Xmt),\Scp\right)\ge   \sqrt{\frac{c_{low}n\lambda(\Cm)}{2K}}%\frac{1}{\sqrt{2}}\mu_2(\phi)\sigma_{\min}\left(\X*\X\right),
		\end{align*}
		\item[(4)] The range space obeys $\text{range}(\J(\Wmi{0},\Xmt))\subset\Scp$ where $\Scp$ is given by Definition \ref{supp space}.
	\end{itemize}
\end{thm}

\begin{lem}[Upper bound on initial misfit]\cite{li2020gradient}\label{upresz} Consider a one-hidden layer neural network model of the form $\xm\mapsto \vm^T\phi\left(\Wm\xm\right)$ where the activation $\phi$ has bounded derivatives obeying $|\phi(0)|,|\phi'(z)|\le \Gamma$. Suppose entries of $\vm\in\R^k$ are half $1/\sqrt{k}$ and half $-1/\sqrt{k}$ so that $\norm{\vm}=1$. Also assume we have $n$ data points $\xmi{1}, \xmi{2},\ldots,\xmi{n}\in\R^d$ with unit euclidean norm ($\norm{\xmi{i}}=1$) aggregated as rows of a matrix $\Xm\in\R^{n\times d}$ and the corresponding labels given by $\ym\in\R^n$ generated accoring to $(\rho,\eps=0,\delta)$ noisy dataset (Definition \ref{cdata}). Then for $\Wmi{0}\in\R^{k\times d}$ with i.i.d.~$\mathcal{N}(0,1)$ entries
	\begin{align*}
		\norm{\vm^T\phi\left(\Wmi{0}\Xm^T\right)-\ym}\le\order{\Gamma\sqrt{n\log K}},%{\log K}
	\end{align*}
	holds with probability at least $1-K^{-100}$.
\end{lem}

Then we introduce a lemma regarding the projection of label noise on the vanilla sample (cluster) induced subspace. Since augmentations are produced by using the vanilla sample $\cmi{i}$ and the augmentation $\xm$ obeys $\norm{\xm -\cmi{i}}\leq \epsilon_0$. So in this sense, we sometimes call the vanilla sample and its augmentations as cluster, and  call the vanilla sample as cluster center.
\begin{lem}\label{label noise lem}~\cite{li2020gradient} Let $\{(\xmi{i},\ymi{i})\}_{i=1}^n$ be an $(\rho,\eps=0,\delta)$ clusterable noisy dataset as described in Definition \ref{cdata}. Let $\{\ymsi{i}\}_{i=1}^n$ be the corresponding ground truth labels. Let $\J(\Wm,\Cm)$ be the Jacobian at the cluster center matrix which is rank $K$ and $\Scp$ be its column space. Then, the difference between noiseless and noisy labels satisfy the bound
	\[
	\normin{\Pro_{\Scp}(\ym-\yms)}\leq 2\rho.
	\]
\end{lem}

\begin{thm}\label{double pert2}~\cite{li2020gradient}
	Assume $\abs{\phi'},\abs{\phi''}\le \Gamma$ and $k \gtrsim d$. Suppose $\Wm_0 \sim \N(0,1)$. Let $\cmi{1},\dots,\cmi{K}$ be cluster centers. Then, with probability at least $1-2e^{-(k+d)}-Ke^{-100d}$ over $\Wmi{0}$, any matrix $\Wm$ satisfying $\fronorm{\Wm-\Wmi{0}}\lesssim\sqrt{k}$ satisfies the following. For all $1\leq i\leq K$,
	\begin{align*}
		\sup_{\norm{\xm-\cmi{i}},\norm{\xmt{} -\cmi{i}}\leq \eps}|f(\Wm,\xm)- f(\Wm,\xmt{})|\leq C\Gamma\eps(\opnorm{\Wm-\Wm_0}+\sqrt{d}).
	\end{align*}
\end{thm}

\begin{lem} [Perturbed Jacobian Distance]~\cite{li2020gradient} \label{pert dist lem}Let $\Xm=[\xmi{1}~\dots~\xmi{n}]^T$ be the input matrix obtained from Definition \ref{cdata}. Let $\Xmt$ be the noiseless inputs where $\xmti{i}$ is the cluster center corresponding to $\xmi{i}$. Let $\J(\Wm,\Xm)$ denote the neural net Jacobian defined in Definition~\ref{nnj def} and define $\J(\Wmi{1},\Wmi{2},\Xm)=\int_{0}^1\J(\alpha\Wmi{1}+(1-\alpha)\Wmi{2},\Xm)d\alpha$. Given weight matrices $\Wmi{1},\Wmi{2},\Wmti{1},\Wmti{2}$, we have that
	\[
	\|\J(\Wm, \Xm)-\J(\Wmt, \Xmt)\|\leq \Gamma\sqrt{n}\left(\frac{\normf{\Wmt-\Wm}}{\sqrt{k}}+\eps\right).
	\]
	and
	\[
	\|\J(\Wmi{1},\Wmi{2},\Xm)-\J(\Wmti{1},\Wmti{2},\Xmt)\|\leq \Gamma\sqrt{n}\left(\frac{\normf{\Wmti{1}-\Wmi{1}}+\normf{\Wmti{2}-\Wmi{2}}}{2\sqrt{k}}+\eps\right).
	\]
\end{lem}

\subsection{Proof of Theorem~\ref{mainresults}}\label{proofofmain}
The subsection has four parts. In the first part, we introduce the proof roadmap in Appendix~\ref{Proofroadmap}. Then in the second part, we present several auxiliary theories in Appendix~\ref{AuxiliaryTheories}. Next, we prove our Theorem~\ref{mainresults} in Appendix~\ref{proofofmainust22}. Finally, we present all proof details of auxiliary theories in Appendix~\ref{AuxiliaryTheories}.

\subsubsection{Proof roadmap}\label{Proofroadmap}
%\textbf{Proof roadmap.}
Before proving Theorem~\ref{mainresults}, we first briefly introduce our main idea. In the \textbf{first step}, we  analyze the general model introduced in Appendix~\ref{generalmodel}. For the solution $\wmi{t}$ at the $t$-th iteration, Theorem~\ref{gradnoise} proves that (1) the distance of $\norm{\wmi{t} - \wmi{0}}$ can be upper bounded; (2) both residual $\norm{\Pro_{\Scp}( f(\wmi{t})-\ymbii{t}{})}$ and $\normin{f(\wmi{t})-\yms}$ can be upper bound.  Result (1) means that the gradient descent algorithm gives solutions in a ball around the initialization $\wmi{0}$, and  helps us verify our assumptions, e.g.~Assumptions~\ref{spert} and~\ref{lrank} and upper bound some variables in our analysis. Results (2) directly bound the label estimation error which plays key role in subsequent analysis.

In the \textbf{second step}, we prove Theorem \ref{mainthmrobust} for the perfectly clustered data ($\epsilon_0=0$) by using Theorem~\ref{gradnoise}. We consider $\epsilon_0\rightarrow 0$ which means that  the input data set is perfectly clean. In this setting,   let $\Xmt=[\xmti{1},\cdots,\xmti{n}]$ be the clean input sample matrix obtained by mapping $\xmi{i}$ to its associated cluster center, i.e. $\xmti{i}=\cmi{\ell}$ if $\xmi{i}$ belongs to the $\ell$-th cluster. In this way, we update network parameter $\Wmti{t}$ as follows:
\begin{align*}
	\Wmti{t+1}=\Wmti{t}-\nabla  \Ltt{t}(\Wmti{t})\quad\text{where}\quad \Ltt{t}(\Wmt)=\frac{1}{2}\sum_{i=1}^n (\ymi{t i}-f(\Wmt,\xmti{i}))^2
\end{align*}
Theorem \ref{mainthmrobust} shows that for neural networks,  our method still can  upper bound the distance $	\|\Wmti{t}-\Wmti{0}\|_F$ and the residuals    $\normin{f(\Wmti{t})-\ymt}$ if the network,   learning rate, the weight $\alphai{i}$  for refining label  satisfy certain conditions.

In the \textbf{third step}, we consider the realistic setting, where we update the parameters on the corrupted data $\Xm=[\xmi{1},\cdots,\xmi{n}]$ as follows:
\begin{align}
	\Wmi{t+1}=\Wmi{t}-\eta\nabla \Lt{t}(\Wmi{t})\quad\text{where}\quad \Lt{t}(\Wm)=\frac{1}{2}\sum_{i=1}^n (\ymi{t i}-f(\Wm,\xmi{i}))^2.
\end{align}
Then to upper bound  $\normin{f(\Wmi{t})-\ymt}$ which measures the error between the predicted label $f(\Wmi{t})$ and the ground truth label $\ymt$, we upper bound $\norm{f(\Wmi{t},\Xm) -f(\Wmti{t},\Xmt)}$ and $\normf{\Wmi{t}-\Wmti{t}}$. These results are formally stated in Theorem~\ref{robust path}.

In the \textbf{fourth step}, we combine the above results together. Specifically,  Theorem \ref{mainthmrobust}   upper bounds  the residuals    $\normin{f(\Wmti{t}, \Xmt)-\ymt}$ and  Theorem~\ref{robust path} upper bounds $\norm{f(\Wmi{t},\Xm) -f(\Wmti{t},\Xmt)}$. So combining these two results and other results in Theorem \ref{mainthmrobust}  \& \ref{robust path}, we can upper bound  $\normin{f(\Wmi{t}, \Xm)-\ymt}$ which is our desired results. At the same time, by using similar method, we can also bound the label estimation error by our self-labeling refinery, since $\norm{\ymbii{t}{} -\yms} = \norm{ (1-\alphai{t} ) \ym + \alphai{t} f(\wm) -\yms}\leq  (1-\alphai{t} )\norm{  \ym -\yms} + \alphai{t} \norm{f(\wm) -\yms}$. The term $\norm{  \ym -\yms}$ denotes the initial label error and can be bounded by a factor related to $\rho$, while the second term is well upper bounded by the above results.

It should be note that our proof framework follows  the recent works \cite{shortest,li2020gradient} which shows that gradient descent is robust to label corruptions. The main difference is that this work uses the label estimation $\ymbii{t}{} = \alphai{t} \ym + (1-\alphai{t}) f(\wm)$ and minimizes the squared loss, while both works  \cite{shortest,li2020gradient} use the corrupted label  $\ym$ and then minimize the squared loss.  By comparison, our method is much more complicated  and gives different proofs.

\subsubsection{Auxiliary Theories} \label{AuxiliaryTheories}
%\textbf{Auxiliary Theories.}
The following theorem is to analyze the general model introduced in Appendix~\ref{generalmodel}. It guarantees  that  the estimated label by our method is close to the ground truth label when the Jacobian mapping is exactly low-rank. By using this results, one can obtain Theorem \ref{mainthmrobust} for the perfectly clustered data ($\epsilon_0=0$) which will be  stated later.

\begin{thm} [Gradient descent with label corruption] \label{gradnoise} Consider a nonlinear least squares problem of the form $\Lt{t}(\wm)=\frac{1}{2} \twonorm{f(\wm)-\ymbii{t}{})}^2$ with the nonlinear mapping $f:\Rs{p}\rightarrow\Rs{n}$ obeying assumptions \ref{lrank} and \ref{spert} over a unit Euclidian ball of radius $\frac{4(1+\psi_1)\norm{f(\wmi{0})-\ym}}{\bn}$ around an initial point $\wmi{0}$ and $\ym=[\yi{1}~\dots~\yi{n}]\in\Rs{n}$ denoting the corrupted labels.  We also assume $\alphai{t} \geq 1- \frac{\alpha^2}{4\beta^2}$ and  $2 \sqrt{n} \lim_{t\rightarrow +\infty}\sum_{t=0}^{t} |\alphai{t} - \alphai{t+1} | \leq   \psi_1 \norm{f(\wmi{0})-\ymbii{0}{}}$. Also let $\yms=[\ymsi{1}~\dots~\ymsi{n}]\in\R^n$ denote the ground truth  labels and $\emm=\ym-\yms$ the corruption. Furthermore, suppose the initial residual $f(\wmi{0})-\ymt$ with respect to the uncorrupted labels obey $f(\wmi{0}) - \yms \in \Scp$.  Then, running gradient descent updates of the from \eqref{GD} with a learning rate $\eta\leq \frac{1}{2\bp^2}\min\left(1,\frac{\bn\beta}{L\norm{f(\wmi{0})-\ymbii{0}{}}}\right)$, all iterates obey
	\[
	\norm{\wmi{t}-\wmi{0}}\leq \frac{4 \norm{\rmi{0}}}{\bn} + 2 \sqrt{n} \lim_{t\rightarrow +\infty}\sum_{t=0}^{t} |\alphai{t} - \alphai{t+1} | \leq \frac{4 (1+\psi) \norm{f(\wmi{0}) - \ymbii{0}{}}}{\bn}.
	\]
	and
	\[
	\norm{\rmbi{t}}^2\le \left(1-\frac{\eta\alpha^2}{4}\right)^t\norm{\rmbi{0}}^2 +    2 \sqrt{n} \sum_{i=0}^{t-1}   \left(1- \frac{\eta\alpha^2}{4} \right)^{t-i} |\alphai{i} - \alphai{i+1} |,\]
	where $\rmi{t}=f(\wmi{t})-\ymbii{t}{}$ and let $\rmi{0}=f(\wmi{0})-\ymbii{0}{}$ be the initial residual, and  $\rmbi{t}=\Pro_{\Scp}(\rmi{t}).$
	Furthermore, assume $\nu>0$ is a precision level obeying $\nu\geq \tin{\Pro_{\Scp}(\emm)}$. Then, after $t\geq \frac{5}{\eta\bn^2}\log \left(\frac{\norm{f(\wmi{0})-\ymbii{0}{}}}{ (1-\alphamax) \nu}\right)$ iterations where $\alphamax=\max_t \alphai{t}$, $\wmi{t}$ achieves the following error bound with respect to the true labels
	\[
	\normin{f(\wmi{t})-\yms}\leq  2\nu + \frac{2 \sqrt{n}}{1-\alphai{t}} \sum_{i=0}^{t-1}   \left(1- \frac{\eta\alpha^2}{4} \right)^{t-i} |\alphai{i} - \alphai{i+1} |.
	\]
	Furthermore, if $\emm$ has at most $s$ nonzeros and $\Scp$ is $\zeta$ diffused per Definition \ref{diff scp}, then using $\nu=\tin{\Pro_{\Scp}(\emm)}$
	\begin{equation*}
		\begin{split}
			\normin{f(\wmi{t})-\yms} \!\leq &  2\tin{\Pro_{\Scp}(\emm)} + \frac{2 \sqrt{n}}{1-\alphai{t}}   \sum_{i=0}^{t-1}  \!  \left(1- \frac{\eta\alpha^2}{4} \right)^{t-i}  \!\!\!|\alphai{i} - \alphai{i+1} | \\
			\le & \frac{\zeta\sqrt{s}}{n}\norm{\emm} +  \frac{2 \sqrt{n}}{1-\alphai{t}} \sum_{t=0}^{t-1} \!\!  \left(1- \frac{\eta\alpha^2}{4} \right)^{t-i}  |\alphai{i} - \alphai{i+1} |,%\tin{\Pro_{\Scp}(\eb)}.%\leq \frac{2\zeta\sqrt{s}}{n}\norm{\s}.
		\end{split}
	\end{equation*}
	where $\Pro_{\Scp}(\emm)$ denotes projection of $\emm$ on $\Scp$.
\end{thm}

See its proof in Appendix~\ref{proofofgradnoise}. This result shows that when the Jacobian of the nonlinear mapping is low-rank, our method  enjoys two good  properties.

For the solution $\wmi{t}$ at the $t$-th iteration, (1) the distance of $\norm{\wmi{t} - \wmi{0}}$ can be upper bounded; (2) both residual $\norm{\Pro_{\Scp}( f(\wmi{t})-\ymi{t})}$ and $\normin{f(\wmi{t})-\ymw}$ can be upper bound.  Result (1) means that the gradient descent algorithm gives solutions in a ball around the initialization $\wmi{0}$, and  helps us verify our assumptions, e.g.~Assumptions~\ref{spert} and~\ref{lrank} and upper bound some variables in our analysis. Results (2) directly bound the label estimation error which plays key role in subsequent analysis. This theorem is the key result that allows us to prove Theorem \ref{mainthmrobust} when the data points are perfectly clustered ($\epsilon_0=0$). Furthermore, this theorem when combined with a perturbation analysis allows us to deal with data that is not perfectly clustered ($\epsilon_0>0$) and to conclude the recovery ability of our method  (Theorem \ref{mainresults}).

When $\epsilon_0\rightarrow 0$ which means that  the input data set is perfectly clustered,  our method can be expected to exactly recover the ground truth label by using neural networks.

\begin{thm} [Training with perfectly clustered data] \label{mainthmrobust} Consider the setting and assumptions of Theorem \ref{gradnoise} with $\epsilon_0=0$. Starting from an initial weight matrix $\wmi{0}$ selected at random with i.i.d.~$\mathcal{N}(0,1)$ entries we run gradient descent updates of the form $\Wmi{t+1}=\Wmi{t}-\eta\nabla \Lt{t}(\Wmi{t})$ on the least-squares loss in the manuscript with step size $\eta\leq \frac{K}{2c_{up}n\Gamma^2\opnorm{\Cm}^2}$. Furthermore, assume the number of hidden nodes obey
	\begin{align*}
		k\ge C(1+\psi_1)^2\Gamma^4\frac{K\log(K)\|\Cm\|^2}{\lambda(\Cm)^2},
	\end{align*}
	with $\lambda(\Cm)$ is the minimum eigenvalue of $\Sigma(\Cm)$ in  Assumption~\ref{assumption2}. Then, with probability at least $1-2/K^{100}$ over randomly initialized $\Wmi{0}\distas\N(0,1)$, the iterates $\Wmi{t}$ obey the following properties.
	\begin{itemize}
		\item[(1)] The distance to initial point $\W_0$ is upper bounded by
		\[
		\|\Wmi{t}-\Wmi{0}\|_F\leq c\Gamma\sqrt{\frac{K\log K}{\lambda(\Cm)}}.
		\]
		\item[(2)] After $t\geq t_0:=  {\frac{c K}{\eta n\lambda(\Cm)}}\log \left(\frac{\Gamma\sqrt{n\log K}}{ (1-\alphamax) \rho}\right)$ iterations where $\alphamax=\max_{0\leq t \leq  t_0} \alphai{t}$, the entrywise predictions of the learned network with respect to the  {\em{ground truth labels}} $\{\ymsi{i}\}_{i=1}^n$ satisfy
		\[
		|f(\Wmi{t},\xmi{i})-\ymsi{i}|\leq 4\rho,%\order{\frac{Ks}{n}},
		\]
		for all $1\leq i\leq n$. Furthermore, if the noise level $\rho$ obeys $\rho\leq \delta/8$ the network predicts the correct label for all samples i.e.~%This implies that if $s\leq \order{\delta n}$, all labels (including noisy ones) will be correctly classified.
		\begin{align}
			\arg\min_{i:1\leq i\leq \bar{K}}|f(\Wmi{t},\xmi{i})-\gamma_i|=\ymsi{i}\quad\text{for}\quad i=1,2,\ldots,n.\label{pls satisfy this eq}
		\end{align}
	\end{itemize}
\end{thm}
See its proof in Appendix~\ref{proofmainthmrobust}. This result shows that in the limit $\epsilon_0\rightarrow 0$ where the data points are perfectly clustered, if the width of network and the iterations satisfy $k\ge C(1+\psi_1)^2\Gamma^4\frac{K\log(K)\|\Cm\|^2}{\lambda(\Cm)^2}$ and $t\geq t_0:=  {\frac{c K}{\eta n\lambda(\Cm)}}\log \left(\frac{\Gamma\sqrt{n\log K}}{ \alphab \rho}\right)$, then our method can exactly recover the ground truth label. This result can be interpreted as ensuring that the network has enough capacity to fit the cluster centers $\{\cb_\ell\}_{\ell=1}^K$ and the associated true labels.

Then we consider the perturbed data $\Xm=[\xmi{1},\cdots,\xmi{n}]$ instead of the perfectly clustered data $\Xmt=[\xmti{1},\cdots,\xmti{n}]$   obtained by mapping $\xmi{i}$ to its associated cluster center, i.e. $\xmti{i}=\cmi{\ell}$ if $\xmi{i}$ belongs to the $\ell$-th cluster. In Theorem~\ref{robust path}, we upper bound the parameter distance and output distance under the two kinds of data $\Xm$ and $\Xmt$.

\begin{thm} [Robustness of gradient path to perturbation] \label{robust path} Generate samples $(\xmi{i},\ymi{i})_{i=1}^n$ according to $(\rho,\eps,\delta)$ corrupted dataset   and form the concatenated input/labels $\Xm\in\R^{d\times n},\ym\in\R^n$. Let $\Xmt$ be the clean input sample matrix obtained by mapping $\xmi{i}$ to its associated cluster center. Set learning rate $\eta\leq \frac{K}{2c_{up}n\Gamma^2\opnorm{\Cm}^2}$ and maximum iterations $t_0$ satisfying
	\[
	\eta t_0=C_1\frac{K}{ n\lambda(\Cm)}\log(\frac{\Gamma\sqrt{n\log K}}{\rho}).
	\]
	where $C_1\geq 1$ is a constant of our choice. Suppose input noise level $\eps$ and number of hidden nodes obey
	\[
	\eps\leq \order{\frac{\lambda(\Cm)}{\Gamma^2K\log(\frac{\Gamma\sqrt{n\log K}}{\rho})}}\quad\text{and}\quad k\geq \order{\Gamma^{10}{\frac{K^2\|\Cm\|^4}{\alphamax^2\lambda(\Cm)^4}}\log(\frac{\Gamma\sqrt{n\log K}}{\rho})^6}.
	\]
	where $\alphamax=\max_{1\leq t \leq t_0} \alphai{t}$.
	Assume $2 \sqrt{n} \sum_{i=0}^{t-1}   \left(1- \frac{\eta\alpha^2}{4} \right)^{t-i} |\alphai{i} - \alphai{i+1} |\leq \psi_2 \norm{\rmi{0}}^2 $ and $2 \sqrt{n} \sum_{i=0}^{t-1}   |\alphai{i} - \alphai{i+1} |\leq \psi_1  \norm{\rmi{0}} $.
	Set $\Wm_0\sim \N(0,1)$. Starting from $\Wm_0=\Wmt_0$ consider the gradient descent iterations over the losses
	\begin{align}
		&\Wmi{t+1}=\Wmi{t}-\eta\nabla \Lt{t}(\Wmi{t})\quad\text{where}\quad \Lt{t}(\Wm)=\frac{1}{2}\sum_{i=1}^n (\ymi{t i}-f(\Wm,\xmi{i}))^2\\
		&\Wmti{t+1}=\Wmti{t}-\nabla  \Ltt{t}(\Wmti{t})\quad\text{where}\quad \Ltt{t}(\Wmt)=\frac{1}{2}\sum_{i=1}^n (\ymi{t i}-f(\Wmt,\xmti{i}))^2
	\end{align}
	Then, for all gradient descent iterations satisfying $t\leq t_0$, we have that
	\[
	\norm{f(\Wmi{t},\Xm) -f(\Wmti{t},\Xmt)}\leq c_0\psi't\eta\eps \Gamma^3n^{3/2}\sqrt{\log K},%\leq \order{\frac{\eps\Gamma^3K\sqrt{n}}{\la(\Cb)}\log(\frac{\Gamma\sqrt{n\log K}}{\rho})^2}.
	\]
	and
	\[
	\normf{\Wmi{t}-\Wmti{t}} \leq \order{t\psi'\eta\eps \frac{\Gamma^4Kn}{\lambda(\Cm)}\log\left(\frac{\Gamma\sqrt{n\log K}}{\rho}\right)^2}.
	\]
	where $\psi'=1 + \frac{\psi_1}{2} +\sqrt{\psi_2}$.
\end{thm}

See its proof in Appendix~\ref{proofofrobustpath}.  Theorem \ref{mainresults} is obtained by combining the above results together.

\subsubsection{Proof of Theorem~\ref{mainresults}}\label{proofofmainust22}
\begin{proof}[Proof of Theorem~\ref{mainresults}]
	Here we prove our results by three steps. In these steps, each step proves one of the three results in our theory.  To begin with, we consider two parameter update settings with initialization as $\Wmi{0}$:
	\begin{align*}
		\Wmti{t+1}=&\Wmti{t}-\nabla  \Ltt{t}(\Wmti{t})\quad\text{where}\quad \Ltt{t}(\Wmt)=\frac{1}{2}\sum_{i=1}^n (\ymwii{t}{i}-f(\Wmt,\xmti{i}))^2,\\
		\Wmi{t+1}=&\Wmi{t}-\eta\nabla \Lt{t}(\Wmi{t})\quad\text{where}\quad \Lt{t}(\Wm)=\frac{1}{2}\sum_{i=1}^n (\ymbii{t}{i}-f(\Wm,\xmi{i}))^2,
	\end{align*}
	where $\ymwii{t}{i} = (1-\alphai{t}) \ym + \alphai{t} f(\Wmti{t},\xmti{i})$,  $\ymbii{t}{i} = (1-\alphai{t}) \ym + \alphai{t} f(\Wmi{t},\xmi{i})$, $\Xmt=[\xmti{1},\cdots,\xmti{n}]$ denotes the clean input sample matrix obtained by mapping $\xmi{i}$ to its associated cluster center, i.e. $\xmti{i}=\cmi{\ell}$ if $\xmi{i}$ belongs to the $\ell$-th cluster, and     $\Xm=[\xmi{1},\cdots,\xmi{n}]$  denotes corrupted data matrix.  Denote the prediction residual vectors of the noiseless and original problems  with respect true ground truth labels $\yms$ by $\rmti{t}=f(\Wmti{t},\Xmt)-\yms$ and $\rmi{t}=f(\Wmi{t},\Xm)-\yms$ respectively.
	
	Theorem \ref{mainthmrobust} shows that if number of iterations $t$ and network width receptively satisfy $t\geq t_0:=  {\frac{c K}{\eta n\lambda(\Cm)}}\log \left(\frac{\Gamma\sqrt{n\log K}}{ \alphab \rho}\right)$ and $k\ge C(1+\psi_1)^2\Gamma^4\frac{K\log(K)\|\Cm\|^2}{\lambda(\Cm)^2}$, then it holds
	\begin{equation*}
		\normin{\rmti{t}}=\normin{f(\Wmti{t},\Xmt)-\yms}  \leq 4\rho \quad \text{and} \quad 	\|\Wmti{t}-\Wmi{0}\|_F\leq c\Gamma\sqrt{\frac{K\log K}{\lambda(\Cm)}}.
	\end{equation*}

	Meanwhile, Theorems \ref{robust path} proves that if $	\eps\leq \order{\frac{\lambda(\Cm)}{\Gamma^2K\log(\frac{\Gamma\sqrt{n\log K}}{\rho})}}$ and $k\geq \order{\Gamma^{10}{\frac{K^2\|\Cm\|^4}{\alphamax^2\lambda(\Cm)^4}}\log(\frac{\Gamma\sqrt{n\log K}}{\rho})^6},$ then it holds
	\[
	\norm{\rmti{t}-\rmi{t}}
	\leq c\eps\frac{\psi'  K}{ n\lambda(\Cm)}\log(\frac{\Gamma\sqrt{n\log K}}{\rho})\Gamma^3n^{3/2}\sqrt{\log K}
	=c\frac{\psi' \eps\Gamma^3K\sqrt{n\log K}}{\lambda(\Cm)}\log(\frac{\Gamma\sqrt{n\log K}}{\rho})
	\]
	and
	\[
	\normf{\Wmi{t}-\Wmti{t}} \leq \order{t\psi'\eta\eps \frac{\Gamma^4Kn}{\lambda(\Cm)}\log\left(\frac{\Gamma\sqrt{n\log K}}{\rho}\right)^2}.
	\]
	where $\psi'=1 + \frac{\psi_1}{2} +\sqrt{\psi_2}$.
	
	\noindent {\bf{Step 1.}} By using the above two results, we have
	\begin{equation*}
		\frac{\norm{f(\Wmi{t},\Xm)-\ymt}}{\sqrt{n}} =\frac{1}{\sqrt{n}}\left(  \norm{\rmti{t}} + \norm{\rmi{t}-\rmti{t}}\right) \leq 4\rho+c\frac{\eps\psi'\Gamma^3K\sqrt{\log K}}{\lambda(\Cm)}\log\left(\frac{\Gamma\sqrt{n\log K}}{\rho}\right).
	\end{equation*}
	Moreover, we can also upper bound
	\begin{equation*}
		\begin{split}
			\frac{\norm{\ymbii{t}{}-\yms}}{\sqrt{n}} \leq & \frac{(1-\alphai{t})\norm{\ym-\yms}}{\sqrt{n}} +  \frac{\alphai{t}\norm{f(\Wmi{t},\Xm)-\yms}}{\sqrt{n}}  \\
			=& \frac{(1-\alphai{t})\norm{\ym-\yms}}{\sqrt{n}} + 4\alphai{t} \rho+c\alphai{t}\frac{\eps\psi'\Gamma^3K\sqrt{\log K}}{\lambda(\Cm)}\log\left(\frac{\Gamma\sqrt{n\log K}}{\rho}\right).
		\end{split}
	\end{equation*}

	\noindent{\bf{Step 2.}} Now we consider what cases that our method can exactly recover the ground truth label. Assume an input $\xm$ is within $\eps$-neighborhood of one of the cluster centers $\cm\in (\cmi{\ell})_{\ell=1}^K$. Then we try to upper bound $|f(\Wmi{t},\xm)-f(\Wmti{t},\cm)|$ where $f(\Wmti{t},\cm)$ corresponds to $f(\Wmti{t},\xmt{})$.  To begin with, we have
	%We will argue that $f(\Wmi{t},\xm)-f(\Wmti{t},\cm)$ is smaller than $\delta/4$ when $\eps$ is small enough. We again write
	\[
	|f(\Wmi{t},\xm)-f(\Wmti{t},\cm)|\leq |f(\Wmi{t},\xm)-f(\Wmti{t},{\xm})|+|f(\Wmti{t},{\xm})-f(\Wmti{t},\cm)|
	\]
	We upper bound the first term as follows:
	\begin{equation*}
		\begin{split}
			|f(\Wmi{t},\xm)-f(\Wmti{t},{\xm})|&=|\vm^T\phi(\Wmi{t}\xm)-\vm^T\phi(\Wmti{t}\xm)|\leq \norm{\vm}\norm{\phi(\Wmi{t}\xm)-\phi(\Wmti{t}\xm)}\\
			&\leq \Gamma \normf{\Wmi{t}-\Wmti{t}}\\
			&\leq \order{\eps\psi' \frac{\Gamma^5K^2}{\lambda(\Cm)^2}\log(\frac{\Gamma\sqrt{n\log K}}{\rho})^3}
		\end{split}
	\end{equation*}
	where we use the results $	\normf{\Wmi{t}-\Wmti{t}} \leq \order{t\psi'\eta\eps \frac{\Gamma^4Kn}{\lambda(\Cm)}\log\left(\frac{\Gamma\sqrt{n\log K}}{\rho}\right)^2}$
	with $\psi'=1 + \frac{\psi_1}{2} +\sqrt{\psi_2}$  in Theorem~\ref{robust path}, and $t=t_0$.
	Next, we need to bound
	\begin{align*}
		|f(\Wmti{t},{\xm})-f(\Wmti{t},\cm)|&\leq |\vm^T\phi(\Wmti{t}\xm)-\vm^T\phi(\Wmti{t}\cm)|.
	\end{align*}
	On the other hand, we have  $\normf{\Wmti{t}-\Wmi{0}}\leq \order{\Gamma\sqrt{\frac{K\log K}{\lambda(\Cm)}}}$ in Theorem~\ref{mainthmrobust}, $\norm{\xm-\cm}\leq \eps$ and $\Wmi{0}\sim \N(0,\Imm)$ in assumption. Moreover, using by assumption we have
	\[
	k\geq \order{\normf{\Wmti{t}-\Wmi{0}}^2}=\order{\Gamma^2\frac{K\log K}{\lambda(\Cm)}}.
	\]
	By using the above results, Theorem \ref{double pert2} guarantees that  with probability at $1-K\exp(-100d)$, for all inputs $\xm$ lying $\eps$ neighborhood of cluster centers, it holds that
	\begin{align}
		|f(\Wmi{t},{\xm})-f(\Wmti{t},\cm)|&\leq C'\Gamma \eps(\normf{\Wmti{t}-\Wmi{0}}+\sqrt{d}) \leq
		C\Gamma \eps\left(\Gamma\sqrt{\frac{K \log K}{\lambda(\Cm)}}+\sqrt{d}\right).
		%&\leq \Gamma \tn{\W_0(\x-\cb)}+\Gamma\eps\|\tilde{\W}_t-\W_0\|\\
		%&\leq \Gamma \tn{\W_0(\x-\cb)}+\Gamma\eps\order{\sqrt{\frac{K\log K}{\la(\Cb)}}}\\
		%&\leq \Gamma \eps ???
	\end{align}
	Combining the two bounds above we get
	\begin{align}
		|f(\Wmi{t},\xm)-f(\Wmti{t},\cm)|&\leq \eps\order{ \frac{\psi'\Gamma^5K^2}{\lambda(\Cm)^2}\log(\frac{\Gamma\sqrt{n\log K}}{\rho})^3+\Gamma (\Gamma\sqrt{\frac{K \log K}{\lambda(\Cm)}}+\sqrt{d})} \nn\\
		&\leq \eps\order{ \frac{\psi'\Gamma^5K^2}{\lambda(\Cm)^2}\log(\frac{\Gamma\sqrt{n\log K}}{\rho})^3}. \nn
	\end{align}
	Hence, if $\eps\leq c'\delta\min\left(\frac{ \lambda(\Cm)^2}{{\psi' \Gamma^5K^2}\log(\frac{\Gamma\sqrt{n\log K}}{\rho})^3},\frac{1}{\Gamma\sqrt{d}}\right)$, we obtain that, for all $\xm$, the associated cluster $\cm$ and true label assigned to cluster $\yms=\yms(\cm)$, we have that
	\[
	|f(\Wmi{t},\xm)-\yms|< |f(\Wmti{t},\cm)-f(\Wmi{t},\xm)|+|f(\Wmti{t},\cm)-\yms|\leq 4\rho + \frac{\delta}{8}.
	\]
	Meanwhile, we can upper bound
	\begin{equation*}
		\begin{split}
			|\ymbii{t}{\xm} - \ymsi{\xm}  |\leq  (1-\alphai{t}) | \ymi{\xm}-\ymsi{\xm}| + \alphai{t} 	|f(\Wmi{t},\xm)-\yms| \leq (1-\alphai{t}) | \ymi{\xm}-\ymsi{\xm}| + \alphai{t}(4\rho + \frac{\delta}{8}).
		\end{split}
	\end{equation*}
	where $\ymbii{t}{\xm}= (1-\alphai{t}) \ymi{\xm} + \alphai{t} f(\Wmi{t},\xm)$ and $ \ymsi{\xm}$ receptively denote the estimated label by our label refinery and the ground truth label of sample $\xm$. Since $| \ymi{\xm}-\ymsi{\xm}| < 1$, by setting $1\geq \alphai{t}\geq 1-  \frac{3}{4}\delta$ and $\rho\leq \delta/32$, we have
	\begin{equation*}
		\begin{split}
			|\ymbii{t}{\xm} - \ymsi{\xm}  |< \frac{\delta}{2}
		\end{split}
	\end{equation*}
	This means that for any sample $\xmi{i}$, we have $|\ymbii{t}{i}-\ymsi{i}|< \delta/2$. Therefore, our label refinery gives the correct estimated labels for all samples.
	By using the same setting,  we obtain
	\[
	|f(\Wmi{t},\xm)-\yms|< \delta/2.
	\]
	This means that for any sample $\xmi{i}$, we have $|f(\Wmi{t},\xmi{i})-\ymsi{i}|< \delta/2$. Therefore, $\Wmi{t}$ gives the correct estimated labels for all samples. This competes all proofs.
	
	%	\noindent{\bf{Step 4.}} This follows from the triangle inequality
	%	\[
	%	\normf{\Wmi{t}-\Wmi{0}}\leq \normf{\Wmi{t}-\Wmti{t}}+\normf{\Wmti{t}-\Wmi{0}}
	%	\]
	%	We have that right hand side terms are at most $ \order{\Gamma\sqrt{\frac{K\log K}{\lambda(\Cm)}}}$ and $\order{t\eta\eps \frac{\Gamma^4Kn}{\lambda(\Cm)}\log(\frac{\Gamma\sqrt{n\log K}}{\rho})^2}$ from Theorems \ref{robust path} and \ref{mainthmrobust} respectively. This implies \eqref{dst W bound}.
\end{proof}

\subsubsection{Proofs of Auxiliary  Theories in Appendix~\ref{proofofmain}}~\label{ProofsofAuxiliaryLemmasandTheories}

\subsubsection{Proof of Theorem~\ref{gradnoise}}\label{proofofgradnoise}

\begin{proof}
	The proof will be done inductively over the properties of gradient descent iterates and is inspired from the recent work \cite{shortest,li2020gradient}. The main difference is that this work uses the label estimation $\ymbii{t}{} = (1-\alphai{t} )\ym + \alphai{t} f(\wmi{t})$ and minimizes the squared loss, while both  \cite{shortest,li2020gradient} use the corrupted label  $\ym$ and then minimize the squared loss.  By comparison, our method is much more complicated  and gives different proofs. Let us introduce the notation related to the residual. Set $\rmi{t}=f(\wmi{t})-\ymbii{t}{}$ and let $\rmi{0}=f(\wmi{0})-\ymbii{0}{}$ be the initial residual. We keep track of the growth of the residual by partitioning the residual as $\rmi{t}=\rmbi{t}+\embi{t}$ where
	\[
	\embi{t}=\Pro_{\Scn}(\rmi{t})\quad,\quad \rmbi{t}=\Pro_{\Scp}(\rmi{t}).
	\]
	We claim that for all iterations $t\geq 0$, the following conditions hold.
	\begin{align}
		\norm{\embi{t}}\leq& \norm{\embi{0} }+ \sqrt{n} \sum_{i=0}^{t} |\alphai{i} - \alphai{i+1} | \leq  \norm{\embi{0} } +  \frac{\psi_1}{2} \norm{\rmi{0} } ,\\
		\norm{\rmbi{t}}^2\le&\left(1-\frac{\eta\alpha^2}{4}\right)^t\norm{\rmbi{0}}^2 +    2 \sqrt{n} \sum_{t=0}^{t-1}   \left(1- \frac{\eta\alpha^2}{4} \right)^{t-t} |\alphai{t} - \alphai{t+1} | ,\label{err}\\
		\frac{\alpha}{4}\norm{\wmi{t}-\wmi{0}}+\norm{\rmbi{t}}\le&\norm{\rmbi{0}} +   2 \sqrt{n} \sum_{i=0}^{t} |\alphai{i} - \alphai{i+1} | \leq \norm{\rmi{0}} +   2 \sqrt{n} \sum_{i=0}^{t} |\alphai{i} - \alphai{i+1} |\notag\\
		\leq& (1+ \phi) \norm{\rmi{0}},\label{close}
	\end{align}
	where the last line uses the assumption that $2 \sqrt{n} \lim_{t\rightarrow +\infty}\sum_{i=0}^{t} |\alphai{i} - \alphai{i+1} | \leq   \psi_1 \norm{\rmi{0}}$.
	Assuming these conditions hold till some $t>0$, inductively, we focus on iteration $t+1$. First, note that these conditions imply that for all $t\geq i\geq 0$, $\wmi{i}\in \D$ where $\D=\big\{\wm\in\Rs{p} \ \big| \ \norm{\wm - \wmi{0}} \leq  \frac{4(1+\psi_1)\norm{\rmi{0}}}{\bn} \big\}$ is the Euclidian ball around $\wmi{0}$ of radius $\frac{4(1+\psi_1)\norm{\rmi{0}}}{\bn}$. This directly follows from \eqref{close} induction hypothesis. Next, we claim that $\wmi{t+1}$ is still within the set $\D$. From Lemma~\ref{next inside}, we have that if the results in Eqn.~\eqref{close} holds, then it holds that
	\begin{equation*}
		\wmi{t+1}\in\D=\Big\{\wm\in\Rs{p} \ \Big| \ \norm{\wm - \wmi{0}} \leq  \frac{4(1+\psi_1)\norm{\rmi{0}}}{\bn} \Big\}.
	\end{equation*}
	In this way, we can directly use the results in previous lemmas and assumptions.  Then we will prove that \eqref{err} and \eqref{close} hold for $t+1$ as well. Note that, following Lemma \ref{lin res}, gradient descent iterate can be written as
	\[
	\rmi{t+1}=(\Imm-\eta\Gm(\wmi{t}))\rmi{t} + \ymbii{t}{} - \ymbii{t+1}{} .
	\]
	Since both column and row space of $\Gm(\wmi{t})$ is subset of $\Scp$, we have that
	\begin{align}
		\embi{t+1}&=\Pro_{\Scn}((\Imm-\eta\Gm(\wmi{t}))\rmi{t} +\ymbii{t}{} - \ymbii{t+1}{})\\
		&=\Pro_{\Scn}(\rmi{t} ) +\Pro_{\Scn}( \ymbii{t}{} - \ymbii{t+1}{})\\
		&=\embi{t} + \Pro_{\Scn}( \ymbii{t}{} - \ymbii{t+1}{}) \\
		&=\embi{t} + \Pro_{\Scn}( (\alphai{t+1} - \alphai{t}) \ym)\\
		&=\embi{0} + \sum_{t=0}^{t}\Pro_{\Scn}( (\alphai{t+1} - \alphai{t}) \ym)\\
		%&=\embi{0} + \Pro_{\Scn}( (\alphai{t+1} - \alphai{0}) \ym).
	\end{align}
	So we can upper bound
	\begin{align}
		\norm{\embi{t}}\leq \norm{\embi{0} }+ 2\sqrt{n} \sum_{i=0}^{t} |\alphai{i} - \alphai{i+1} | \leq  \norm{\embi{0} } + \psi_1  \norm{\rmi{0} }.
	\end{align}
	
	This shows the first statement of the induction. Next, over $\Scp$, we have
	\begin{align}
		\rmbi{t+1}&=\Pro_{\Scp}((\Imm-\eta \Gm(\wmi{t}))\rmi{t} +  \ymbii{t}{} - \ymbii{t+1}{})\\
		&=\Pro_{\Scp}((\Imm-\eta\Gm(\wmi{t}))\rmbi{t})+\Pro_{\Scp}((\Imm-\eta\Gm(\wmi{t}))\embi{t}) + \Pro_{\Scp}(\ymbii{t}{} - \ymbii{t+1}{})\\
		&=\Pro_{\Scp}((\Imm-\eta\Gm(\wmi{t}))\rmbi{t}) + \Pro_{\Scp}(\ymbii{t}{} - \ymbii{t+1}{})\\
		&=(\Imm-\eta\Gm(\wmi{t}))\rmbi{t} + \ymbii{t}{} - \ymbii{t+1}{}\label{r tau+1}
	\end{align}
	where the second line uses the fact that $\embi{t}\in\Scn$ and last line uses the fact that $\rmbi{t}\in\Scp$, in the last line, we let $\ymbi{t}= \Pro_{\Scp}(\ymbii{t}{})$.  Then we can rewrite $\ymbii{t}{} -\ymbii{t+1}{}$ as
	\begin{equation*}
		\begin{split}
			\ymbi{t} -\ymbi{t+1} = &  (1-\alphai{t} )\ym + \alphai{t}  f(\wmi{t}) -  (1-\alphai{t+1}) \ym + \alphai{t+1}  f(\wmi{t+1})\\
			= &  (\alphai{t+1} - \alphai{t})  \ym +  \alphai{t} (f(\wmi{t}) - f(\wmi{t+1})) - (  \alphai{t+1}  -  \alphai{t} )   f(\wmi{t+1}).
		\end{split}
	\end{equation*}
	%
	%\begin{equation*}
	%\begin{split}
	%\ymbi{t} -\ymbi{t+1} = &  \alphai{t} \ym + (1-\alphai{t} ) f(\wmi{t}) -  \alphai{t+1} \ym - (1-\alphai{t+1} ) f(\wmi{t+1})\\
	%= &  (\alphai{t} - \alphai{t+1})  \ym + (1-\alphai{t} ) (f(\wmi{t}) - f(\wmi{t+1})) - (  \alphai{t+1}  -  \alphai{t} )   f(\wmi{t+1}).
	%\end{split}
	%\end{equation*}
	At the same time, we can upper bound
	\begin{equation*}
		\begin{split}
			\|\wmi{t+1} - \wmi{t}\|_F = \eta \norm{ \J(\wmi{t})^T \rmi{t}} \led{172} \eta \norm{ \J(\wmi{t})^T \rmbi{t}}\leq \eta \beta \norm{\rmbi{t}}.
		\end{split}
	\end{equation*}
	In this way, we can obtain
	\begin{equation*}
		\begin{split}
			& \norm{\rmbi{t+1}}\\
			\leq & \norm{(\Imm-\eta\Gm(\wmi{t}))\rmbi{t}} + \norm{(\alphai{t} - \alphai{t+1})  \ym} + \alphai{t}   \norm{f(\wmi{t}) - f(\wmi{t+1})} + \norm{ (  \alphai{t+1}  -  \alphai{t} )   f(\wmi{t+1})}\\
			\led{172} & \left(1- \frac{\eta\alpha^2}{2} \right) \norm{\rmbi{t}} + \alphai{t} \beta\norm{ \wmi{t} -\wmi{t+1}}  +  2 \sqrt{n} \cdot |\alphai{t} - \alphai{t+1} | \\
			\leq & \left(1- \frac{\eta\alpha^2}{2} \right) \norm{\rmbi{t}} +\alphai{t} \beta^2 \eta  \norm{ \rmbi{t}}  +  2 \sqrt{n} \cdot |\alphai{t} - \alphai{t+1} | \\
			\led{173} & \left(1- \frac{\eta\alpha^2}{4} \right) \norm{\rmbi{t}} +    2 \sqrt{n} \cdot |\alphai{t} - \alphai{t+1} | \\
		\end{split}
	\end{equation*}
	where \ding{172} uses  in Lemma~\ref{propertydefinitematrix}, $\norm{\ym}\leq \sqrt{n}$ and $\norm{ f(\wmi{t+1})}\leq \sqrt{n}$, \ding{173} uses $\alphai{t} \leq \frac{\alpha^2}{4\beta^2}$.  This result further yields
	\begin{equation*}
		\begin{split}
			\norm{\rmbi{t}} \leq  \left(1- \frac{\eta\alpha^2}{4} \right)^{t} \norm{\rmbi{0}} +    2 \sqrt{n} \sum_{t=0}^{t-1}   \left(1- \frac{\eta\alpha^2}{4} \right)^{t-t} |\alphai{t} - \alphai{t+1} | \\
		\end{split}
	\end{equation*}
	On the other hand, we have
	\begin{equation*}
		\begin{split}
			\norm{(\Imm-\eta\Gm(\wmi{t}))\rmbi{t}}^2 \leq &  \norm{\rmbi{t}}^2 - 2 \eta \rmbi{t}^T\Gm(\wmi{t})\rmbi{t} + \eta^2 \rmbi{t}^T\Gm^T(\wmi{t})\Gm(\wmi{t})\rmbi{t}\\
			\leq &  \norm{\rmbi{t}}^2 - 2 \eta \rmbi{t}^T \J(\wmi{t})  \J^T(\wmi{t})\rmbi{t} + \eta^2\beta^2  \rmbi{t}^T\J(\wmi{t})\J^T(\wmi{t})\rmbi{t}\\
			= &  \norm{\rmbi{t}}^2 -  \eta (2- \eta\beta^2)    \norm{ \J^T(\wmi{t})\rmbi{t} }^2\\
			\leq &  \norm{\rmbi{t}}^2 -  \eta    \norm{ \J^T(\wmi{t})\rmbi{t} }^2,
		\end{split}
	\end{equation*}
	where the last line use $\eta \leq \frac{1}{\beta^2}$. This further gives
	\begin{equation*}
		\begin{split}
			\norm{(\Imm-\eta\Gm(\wmi{t}))\rmbi{t}}  \leq  \sqrt{ \norm{\rmbi{t}}^2 -  \eta    \norm{ \J^T(\wmi{t})\rmbi{t} }^2} \leq \norm{\rmbi{t}}  -  \frac{\eta}{2}  \frac{  \norm{ \J^T(\wmi{t})\rmbi{t} }^2}{\norm{\rmbi{t}}} .
		\end{split}
	\end{equation*}
	Therefore, we can upper bound $\norm{\rmbi{t}} $ in another way which can help to bound $\norm{\wmi{t+1} -\wmi{0}}$:
	\begin{equation*}
		\begin{split}
			& \norm{\rmbi{t+1}} \\
			\leq & \norm{(\Imm-\eta\Gm(\wmi{t}))\rmbi{t}} + \norm{(\alphai{t} - \alphai{t+1})  \ym} + (1-\alphai{t} ) \norm{f(\wmi{t}) - f(\wmi{t+1})} + \norm{ (  \alphai{t+1}  -  \alphai{t} )   f(\wmi{t+1})}\\
			\leq &  \norm{(\Imm-\eta\Gm(\wmi{t}))\rmbi{t}}+ (1-\alphai{t} )\beta\norm{ \wmi{t} -\wmi{t+1}}  +  2 \sqrt{n} \cdot |\alphai{t} - \alphai{t+1} | \\
			= &  \norm{(\Imm-\eta\Gm(\wmi{t}))\rmbi{t}}+ (1-\alphai{t} )\beta \eta \norm{ \J^T(\wmi{t}) \rmi{t}}  +  2 \sqrt{n} \cdot |\alphai{t} - \alphai{t+1} | \\
			\leq & \norm{\rmbi{t}}  -  \frac{\eta}{2}  \frac{  \norm{ \J^T(\wmi{t})\rmbi{t} }^2}{\norm{\rmbi{t}}} + (1-\alphai{t} )\beta \eta \norm{ \J^T(\wmi{t}) \rmi{t}}  +  2 \sqrt{n} \cdot |\alphai{t} - \alphai{t+1} |.
		\end{split}
	\end{equation*}
	Since the distance of $\wmi{t+1}$ to initial point satisfies :
	\begin{equation*}
		\begin{split}
			\norm{\wmi{t+1} - \wmi{0}} \leq \norm{\wmi{t+1} - \wmi{t}} + \norm{\wmi{t}-\wmi{0}} \leq \norm{\wmi{t}-\wmi{0}}  +  \eta \norm{ \J^T(\wmi{t}) \rmi{t}} ,
		\end{split}
	\end{equation*}
	we can further bound
	\begin{equation*}
		\begin{split}
			& \frac{\alpha}{4}\norm{\wmi{t+1} - \wmi{0}} + \norm{\rmbi{t+1}} \\
			\leq & \frac{\alpha}{4}\left( \norm{\wmi{t}-\wmi{0}}  +  \eta \norm{ \J^T(\wmi{t}) \rmi{t}} \right) +   \norm{\rmbi{t}}  -  \frac{\eta}{2}  \frac{  \norm{ \J^T(\wmi{t})\rmbi{t} }^2}{\norm{\rmbi{t}}} \\
			&+ (1-\alphai{t} )\beta \eta \norm{ \J^T(\wmi{t}) \rmi{t}}  +  2 \sqrt{n} \cdot |\alphai{t} - \alphai{t+1} |\\
			\leq & \frac{\alpha}{4} \norm{\wmi{t}-\wmi{0}}  +  \norm{\rmbi{t}}  +  \frac{\eta}{4} \norm{ \J^T(\wmi{t}) \rmi{t}} \left( \alpha +   4(1-\alphai{t} )\beta -  2\frac{  \norm{ \J^T(\wmi{t})\rmbi{t} }}{\norm{\rmbi{t}}} \right)+  2 \sqrt{n} \cdot |\alphai{t} - \alphai{t+1} |\\
			\led{172} & \frac{\alpha}{4} \norm{\wmi{t}-\wmi{0}}  +  \norm{\rmbi{t}}  +   2 \sqrt{n} \cdot |\alphai{t} - \alphai{t+1} |\\
			\leq& \norm{\rmbi{0}}  +   2 \sqrt{n} \sum_{i=0}^{t} |\alphai{i} - \alphai{i+1} | \leq \norm{\rmi{0}}  +   2 \sqrt{n} \sum_{i=0}^{t} |\alphai{i} - \alphai{i+1} |,
		\end{split}
	\end{equation*}
	where \ding{172} uses $\frac{  \norm{ \J^T(\wmi{t})\rmbi{t} }}{\norm{\rmbi{t}}} \geq \alpha$ and $\alphai{t} \leq  \frac{\alpha}{4\beta}$.
	
	By setting $t\geq \frac{5}{\eta \alpha^2} \log\left( \frac{\norm{\rmi{0}}}{(1-\alphamax) \nu}\right)$ and $\frac{\eta\alpha^2}{4} \leq \frac{\eta\beta^2}{4} \leq \frac{1}{8}$ where $\alphamax=\max_t \alphai{t}$, then we have $\log\frac{1}{1-\frac{\eta\alpha^2}{4}} \geq \log\left(1+ \frac{\eta\alpha^2}{4} \right) \geq  \frac{\eta\alpha^2}{5}$ and thus
	\begin{equation*}
		\begin{split}
			\left(1- \frac{\eta\alpha^2}{4} \right)^{t} \norm{\rmbi{0}} \leq & \left(1- \frac{\eta\alpha^2}{4} \right)^{t} \norm{\rmi{0}} \leq  \left(1- \frac{\eta\alpha^2}{4} \right)^{t} \norm{\rmi{0}} \leq (1-\alphamax) \nu.
		\end{split}
	\end{equation*}
	In this way, we can further obtain
	\begin{equation*}
		\begin{split}
			\normin{\rmbi{t}}  \leq  \norm{\rmbi{t}} \leq (1-\alphamax) \nu + 2 \sqrt{n} \sum_{i=0}^{t-1}   \left(1- \frac{\eta\alpha^2}{4} \right)^{t-i} |\alphai{i} - \alphai{i+1} |
		\end{split}
	\end{equation*}
	and
	\begin{equation*}
		\begin{split}
			(1-\alphai{t}) \normin{\Pro_{\Scp}(f(\wmi{t})- \ym ) }=  & \normin{\Pro_{\Scp}(f(\wmi{t})- (1-\alphai{t} )\ym -  \alphai{t} f(\wmi{t})) }  =  \normin{\rmbi{t}}  \leq  \norm{\rmbi{t}}\\
			\leq &  (1-\alphamax) \nu + 2 \sqrt{n} \sum_{t=0}^{t-1}   \left(1- \frac{\eta\alpha^2}{4} \right)^{t-t} |\alphai{t} - \alphai{t+1} |
		\end{split}
	\end{equation*}
	Finally, we can obtain the desired results:
	\begin{equation*}
		\begin{split}
			\normin{f(\wmi{t})- \yms} \lee{172}  & \normin{\Pro_{\Scp}(f(\wmi{t}))-\Pro_{\Scp}(\yms)} \\
			\leq &  \normin{\Pro_{\Scp}(f(\wmi{t})- \ym)}  + \normin{\Pro_{\Scp}(\ym- \yms)} \\
			\leq & 2\nu + \frac{2 \sqrt{n}}{1-\alphai{t}} \sum_{i=0}^{t-1}   \left(1- \frac{\eta\alpha^2}{4} \right)^{t-i} |\alphai{i} - \alphai{i+1} |,
		\end{split}
	\end{equation*}
	where \ding{172} holds since $f(\wmi{t}) - \yms \in \Scp$ and  $\normin{\Pro_{\Scp}(f(\wmi{t})- \ym)}  = \normin{\Pro_{\Scp}(f(\wmi{t})- \ym)}$.
	If $\emm$ is $s$ sparse and $\Scp$ is diffused, applying Definition~\ref{diff scp} we have
	\[
	\tin{\Pro_{\Scp}(\emm)}\leq \frac{\gamma\sqrt{s}}{n}\normin{\emm}.
	\]
	The proof is completed.
\end{proof}

%\subsection{Proofs for Neural Networks}

\subsubsection{Proof of Theorem~\ref{mainthmrobust}} \label{proofmainthmrobust}

\begin{proof} The proof is based on the meta Theorem \ref{gradnoise}, hence we need to verify its Assumptions \ref{lrank} and \ref{spert} with proper values and apply Lemma \ref{label noise lem} to get $\tin{\Pro_{\Scp}(\emm)}$. We will also make significant use of Corollary \ref{JLcor}.
	
	%Let us first verify Assumption \ref{diff scp}. Let $\Scp$ be the support space of the dataset per Definition \ref{supp space} which is the range space of Jacobian via Corollary \ref{JLcor}. It follows from Lemma \ref{diff subspace} that $\Scp$ is ${\sqrt{\frac{K}{c_{low}n}}}$ diffused by construction. Since Jacobian range space is subset of $\Scp$ (as $\eps=0$), this also implies $\be$ can be chosen to be $0$ in Assumption \ref{lrank}.
	
	Using Corollary \ref{JLcor}, Assumption \ref{spert} holds with $L=\Gamma\sqrt{\frac{c_{up} n}{{kK}}}\opnorm{\Cm}$ where $L$ is the Lipschitz constant of Jacobian spectrum. Denote %$\rmi{t}=f(\Wmi{t})-\ymi{t}$.
	Using Lemma \ref{upresz} with probability $1-K^{-100}$, we have that $\norm{\rmi{0}}=\norm{\ymbii{0}{}-f(\Wmi{0})}=\norm{\ym-f(\Wmi{0})}\le \Gamma{\sqrt{c_0n\log K/128}}$ for some $c_0>0$. Corollary \ref{JLcor} guarantees a uniform bound for $\bp$, hence in Assumption \ref{lrank}, we pick
	\[
	\bp\leq \sqrt{\frac{c_{up}n}{K}}\Gamma\opnorm{\Cm}.
	\]
	We shall also pick the minimum singular value over $\Scp$ to be
	\[
	\bn=\frac{\alpha'}{2}\quad\text{where}\quad \alpha'=  \sqrt{\frac{c_{low}n\lambda(\Cm)}{2K}},
	\]
	We wish to verify Assumption \ref{lrank} over the radius of
	\[
	R=\frac{4\norm{f(\Wmi{0})-\ym}}{\bn}\leq\frac{ \Gamma{\sqrt{c_0n\log K/8}}}{\bn}=\Gamma\sqrt{\frac{{{c_0n\log K/2}}}{{\frac{c_{low}n\lambda(\Cm)}{2K}}}}=\Gamma\sqrt{\frac{c_0K\log K}{c_{low}\lambda(\Cm)}},%=\frac{4c_0}{c_1\sqrt{\la(\Cb)}}.
	\]
	%Since input dataset is noiseless (i.e.~$\eps=0$), we already have $\eps=0$ in Assumption \ref{lrank}.
	neighborhood of $\Wmi{0}$. What remains is ensuring that Jacobian over $\Scp$ is lower bounded by $\bn$. Our choice of $k$ guarantees that at the initialization, with probability $1-K^{-100}$, we have
	\[
	\sigma(\J(\Wmi{0},\Xm),\Scp)\geq \bn'.
	\]
	Suppose $LR\leq \bn=\bn'/2$ which can be achieved by using large $k$. Using triangle inequality on Jacobian spectrum, for any $\Wm\in\D$, using $\normf{\Wm-\Wmi{0}}\leq R$, we would have
	\[
	\sigma(\J(\Wm,\Xm),\Scp)\geq \sigma({\J(\Wmi{0},\Xm),\Scp})-LR\geq \bn'-\bn=\bn.
	\]
	Now, observe that
	\begin{align}
		LR= & (1+\psi_1) \Gamma\sqrt{\frac{c_{up} n}{{kK}}}\opnorm{\Cm} \Gamma\sqrt{\frac{c_0K\log(K)}{c_{low}\lambda(\Cm)}}=(1+\psi_1)\Gamma^2\|\Cm\|\sqrt{\frac{c_{up}c_0n\log K}{c_{low}k\lambda(\Cm)}}\\
		\leq & \frac{\bn'}{2}= \sqrt{\frac{c_{low}n\lambda(\Cm)}{8K}},
		%\\
		%\frac{4c_0}{c_1\sqrt{\la(\Cb)}}\frac{\Gamma\sqrt{n}}{\sqrt{k}}\opnorm{\Cb}\leq \alpha=c_1\sqrt{n\la(\Cb)}.
	\end{align}
	as $k$ satisfies
	\[
	k\geq \order{(1+\psi_1)^2\Gamma^4\|\Cm\|^2\frac{c_{up}K\log(K)}{c_{low}^2\lambda(\Cm)^2}}\ge  \order{\frac{(1+\psi_1)^2\Gamma^4{K\log (K)\opnorm{\Cm}^2}}{\lambda(\Cm)^2}}.
	\]
	Finally, since $LR=4(1+\psi_1)L\norm{\rmi{0}}/\alpha\leq \bn$, the learning rate is
	\[
	\eta\leq \frac{1}{2\bp^2}\min(1,\frac{\bn\beta}{L\norm{\rmi{0}}})=\frac{1}{2\bp^2}=\frac{K}{2c_{up}n\Gamma^2\opnorm{\Cm}^2}.
	\]
	%Denote the minimum singular of a matrix $\M$ over a left singular subspace $S$ via $\smn{\X,S}$. Applying Corollary \ref{JLcor}, with desired probability
	%\[
	%\smn{\Jc(\W_0,\X),\Scp}\geq 2c_1\sqrt{n\la(\Cb)}.
	%\]
	Overall, the assumptions of Theorem \ref{gradnoise} holds with stated $\bn,\bp,L$ with probability $1-2K^{-100}$ (union bounding initial residual and minimum singular value events). This implies for all $t>0$ the distance of current iterate to initial obeys
	\[
	\normf{\Wmi{t}-\Wmi{0}}\leq R.
	\]
	The final step is the properties of the label corruption. Using Lemma \ref{label noise lem}, we find that
	\[
	\tin{\Pro_{\Scp}(\yms-\ym)}\leq 2\rho.
	\]
	%We are given that there are $s$ noisy labels and each label noise moves the label by $\delta$. Hence, the label noise obeys $\tn{\y-\tilde{\y}}\leq  2\sqrt{s}$.
	Substituting the values corresponding to $\bn,\bp,L$ yields that, for all gradient iterations with
	\[% \Gamma{\sqrt{c_0n\log K/32}}
	\frac{5}{\eta\bn^2}\log \left(\frac{\norm{\rmi{0}}}{2 (1-\alphamax) \rho}\right)\leq\frac{5}{\eta\bn^2}\log \left(\frac{ \Gamma{\sqrt{c_0n\log K/32}}}{2  (1-\alphamax) \rho}\right)=\order{{\frac{K}{\eta n\lambda(\Cm)}}\log \left(\frac{\Gamma\sqrt{n\log K}}{ (1-\alphamax) \rho}\right)}\leq t,
	\]
	denoting the clean labels by $\ymt$  and applying Theorem \ref{gradnoise}, we have that, the infinity norm of the residual obeys (using $\normin{\Pro_{\Scp}(\emm)}=\normin{\Pro_{\Scp}(\ym-\yms)}\leq 2\rho$)
	\[
	\normin{f(\Wm)-\yms}\leq4\rho.
	\]
	%\begin{itemize}
	%\item the distance of current iterate to initial obeys
	%\[
	%\tf{\W_t-\W_0}\leq R.
	%\]
	%\item and % $\Scp$ is $\order{\sqrt{K/n}}$ diffused)
	This implies that if $\rho\leq \delta/8$, the network will miss the correct label by at most $\delta/2$, hence all labels (including noisy ones) will be correctly classified.
	%\end{itemize}
\end{proof}

\subsubsection{Proof of Theorem~\ref{robust path}} \label{proofofrobustpath}

\begin{proof} Since $\Wmti{t}$ are the noiseless iterations, with probability $1-2K^{-100}$, the statements of Theorem \ref{mainthmrobust} hold on $\Wmti{t}$.  To proceed with proof, we first introduce short hand notations. We use
	\begin{align}
		& \rmi{i}=f(\Wmi{i},\Xm)-\ymbii{i}{},~\rmti{i}=f(\Wmti{i},\Xmti{i})-\ymwii{i}{}\\
		&\Ji{i}=\J(\Wmi{i},\Xm),~\Ji{i+1,i}=\J(\Wmi{i+1},\Wmi{i},\Xm),~\Jti{i}=\J(\Wmti{i},\Xmt),~\Jti{i+1,i}=\J(\Wmti{i+1},\Wmti{i},\Xmt)\\
		&d_i=\normf{\Wmi{i}-\Wmti{i}},~p_i=\norm{\rmi{i}-\rmti{i}},~\beta=\Gamma\|\Cm\|\sqrt{c_{up}n/K},~L=\Gamma\|\Cm\|\sqrt{c_{up}n/Kk}.
	\end{align}
	Here $\beta$ is the upper bound on the Jacobian spectrum and $L$ is the spectral norm Lipschitz constant as in Theorem \ref{JLcor}. Applying Lemma \ref{pert dist lem}, note that
	\begin{align}
		&\|\J(\Wmi{t},\Xm)-\J(\Wmti{t},\Xmt)\|\leq L{\norm{\Wmti{t}-\Wmi{t}}}+\Gamma\sqrt{n}\eps\leq Ld_t+\Gamma\sqrt{n}\eps\\
		&\|\J(\Wmi{t+1},\Wmi{t},\Xm)-\J(\Wmti{t+1}, \Wmti{t},\Xmt)\|\leq L(d_t+d_{t+1})/2+\Gamma\sqrt{n}\eps.
	\end{align}
	By defining
	\[
	\embi{t}=\Pro_{\Scn}(\rmti{t})\quad,\quad \rmbi{t}=\Pro_{\Scp}(\rmti{t}),
	\]
	then we can use Theorem~\ref{gradnoise} and the assumption that $2 \sqrt{n} \sum_{i=0}^{t-1}   \left(1- \frac{\eta\alpha^2}{4} \right)^{t-i} |\alphai{i} - \alphai{i+1} |\leq \psi_2 \norm{\rmi{0}}^2 $ to obtain
	\begin{align}
		\norm{\embi{t}}\leq& \norm{\embi{0} }+ \sqrt{n} \sum_{i=0}^{t} |\alphai{i} - \alphai{i+1} | \leq  \norm{\embi{0} } +  \frac{\psi}{2} \norm{\rmi{0} } ,\\
		\norm{\rmbi{t}}^2\le&\left(1-\frac{\eta\alpha^2}{4}\right)^t\norm{\rmbi{0}}^2 +    2 \sqrt{n} \sum_{i=0}^{t-1}   \left(1- \frac{\eta\alpha^2}{4} \right)^{t-i} |\alphai{i} - \alphai{i+1} | \leq    \norm{\rmbi{0}}^2+ \psi_2  \norm{\rmi{0}}^2.
	\end{align}
	Therefore, we can upper bound
	\begin{align}
		\norm{\rmti{t}} = \norm{\embi{t}} + \norm{\rmbi{t}}\leq    \norm{\embi{0} } +  \frac{\psi}{2} \norm{\rmi{0} }  + \norm{\rmbi{0}}+ \sqrt{\psi_2}  \norm{\rmi{0}} = \left(1 + \frac{\psi}{2} +\sqrt{\psi_2} \right) \norm{\rmi{0}}.
	\end{align}
	Following this and setting $\norm{\rmti{t}}\leq \psi' \norm{\rmi{0}}$, note that parameter satisfies
	\begin{align}
		&\Wmi{i+1}=\Wmi{i}-\eta \Ji{i}\rmi{i}\quad,\quad \Wmti{i+1}=\Wmti{i} - \eta \Jti{i}^T\rmti{i}\\
		&\normf{\Wmi{i+1}-\Wmti{i+1}}\leq \normf{\Wmi{i}-\Wmti{i}}+\eta \|\Ji{i} - \Jti{i}\|\normf{\rmi{i}}+\eta \|\Ji{i}\|\norm{\rmi{i}-\rmti{i}}\\
		&d_{i+1}\leq d_i+\eta (\psi' (L d_i+\Gamma\sqrt{n}\eps)\norm{\rmi{0}}+\beta p_i),\label{dd rec}
		%\quad,\quad \tilde{\W}_{i+1}=\tilde{\W}_i-\eta \tilde{\Jc}_i\tilde{\rb}_i\\
	\end{align}
	and residual satisfies (using $\Imm \succeq \Jti{i+1,i} \Jti{i}^T/\beta^2\succeq0$)
	\begin{align}
		\rmi{i+1}&=\rmi{i}-\eta \Ji{i+1,i}\Ji{i}^T\rmi{i}\implies\\%\eta \Jc_{i+1,i}\grad{\W_i}=
		\rmi{i+1}-\rmti{i+1}&\\
		=(\rmi{i}-\rmti{i})&-\eta( \Ji{i+1,i}-\Jti{i+1,i})\Ji{i}^T\rmi{i}-\eta\Jti{i+1,i}(\Ji{i}^T-\Jti{i}^T)\rmi{i}-\eta \Jti{i+1,i} \Jti{i}^T(\rmi{i}-\rmti{i}).\\
		\rmi{i+1}-\rmti{i+1}&=(\Imm-\eta\Jti{i+1,i} \Jti{i}^T)(\rmi{i}-\rmti{i})-\eta( \Ji{i+1,i}-\Jti{i+1,i})\Ji{i}^T\rmi{i} - \eta \Jti{i+1,i}(\Ji{i}^T-\Jti{i}^T)\rmi{i}.\\
		\norm{\rmi{i+1}-\rmti{i+1}}&\leq \norm{\rmi{i}-\rmti{i}}+\eta\beta\norm{\rmi{i}}( L(3d_t+d_{t+1})/2+2\Gamma\sqrt{n}\eps).\\
		\norm{\rmi{i+1}-\rmti{i+1}}&\leq \norm{\rmi{i}-\rmti{i}}+\eta\beta(\norm{\rmti{0}}+p_i)( L(3d_t+d_{t+1})/2+2\Gamma\sqrt{n}\eps).\label{pp rec}
	\end{align}
	where we used $\norm{\rmi{i}}\leq p_i+\psi'\norm{\rmi{0}}$ and $\norm{(\Imm-\eta\Jti{i+1,i}\Jti{i}^T)\vm}\leq \norm{\vm}$ which follows from Lemma~\ref{propertydefinitematrix}. This implies
	\begin{align}
		p_{i+1}\leq p_i+\eta\beta(\psi'\norm{\rmi{0}}+p_i)( L(3d_t+d_{t+1})/2+2\Gamma\sqrt{n}\eps).\label{perturbme}
	\end{align}
	\noindent {\bf{Finalizing proof:}} Next, using Lemma \ref{upresz}, we have $\norm{\rmi{0}}\leq \Theta:=C_0\Gamma \sqrt{n\log K}$. We claim that if
	\begin{align}
		\boxed{\eps\leq  \order{\frac{1}{t_0\eta\Gamma^2n}}\leq\frac{1}{8t_0\eta\beta\Gamma\sqrt{n}}\quad\text{and}\quad L\leq\frac{2}{5t_0\eta\Theta(1+8\eta t_0\beta^2)}\leq\frac{1}{30(t_0\eta\beta)^2\Theta},}% \min(\frac{1}{6t_0\eta(\Theta+4\etat_0\beta)},\frac{1}{2\etat_0\Theta})=
	\end{align}
	%and setting $\Phi=\Gamma\sqrt{n}\eps\Theta$,
	(where we used $\eta t_0\beta^2\geq 1$), for all $t\leq t_0$, we have that
	\begin{align}
		p_t\leq 8t(1+\psi')\eta\Gamma\sqrt{n}\eps\Theta\beta \leq \Theta\quad,\quad d_t\leq 2t\eta \Gamma\sqrt{n}\eps\Theta( \psi'+8\eta t_0\beta^2).\label{induct induct}
	\end{align}
	The proof is by induction. Suppose it holds until $t\leq t_0-1$. At $t+1$, via \eqref{dd rec} we have that
	\[
	\frac{d_{t+1}-d_t}{\eta}\leq \psi'(Ld_t\Theta+\Gamma\sqrt{n}\eps\Theta)+ 8t_0\eta\beta^2\Gamma\sqrt{n}\eps\Theta\overset{?}{\leq}2 \Gamma\sqrt{n}\eps\Theta( \psi'+8\eta t_0\beta^2).
	\]
	Right hand side holds since $L\leq \frac{1}{2\eta t_0\Theta}$. This establishes the induction for $d_{t+1}$.
	%\begin{align}
	%L\leq \frac{1}{2\etat_0\Theta} \implies Ld_t\tn{\tilde{\rb}_0}\leq Ld_t\Theta\leq \Gamma\sqrt{n}\eps\Theta(1+4\etat_0\beta^2)
	%\end{align}
	%which establishes
	%\begin{align}
	%d_{t+1}&\leq2(t+1)\eta \Gamma\sqrt{n}\eps(\Theta+4\etat_0\beta)
	%\end{align}
	
	Next, we show the induction on $p_t$. Observe that $3d_t+d_{t+1}\leq 10t_0\eta \Gamma\sqrt{n}\eps\Theta( \psi'+8\eta t_0\beta^2) $. Following \eqref{perturbme} and using $p_t\leq \Theta$, we need
	\begin{align}
		\frac{p_{t+1}-p_t}{\eta}\leq \beta(1+\psi')\Theta( L(3d_t+d_{t+1})+4\Gamma\sqrt{n}\eps)& \overset{?}{\leq}\frac{8}{\alphamax}(1+\psi')\Gamma\sqrt{n}\eps\Theta\beta\iff\\
		L(3d_t+d_{t+1})+4\Gamma\sqrt{n}\eps&\overset{?}{\leq}\frac{8}{\alphamax}\Gamma\sqrt{n}\eps\iff\\
		L(3d_t+d_{t+1})&\overset{?}{\leq}\frac{4}{\alphamax}\Gamma\sqrt{n}\eps\iff\\
		10\alphamax Lt_0\eta(1+8\eta t_0\beta^2)\Theta&\overset{?}{\leq}4\iff\\
		L&\overset{?}{\leq}\frac{2}{5t_0\alphamax\eta(1+8\eta t_0\beta^2)\Theta},
	\end{align}
	where $\alphamax=\max_{1\leq t \leq t_0} \alphai{t}$.  Concluding the induction since $L$ satisfies the final line. Consequently, for all $0\leq t\leq t_0$, we have that
	\begin{equation*}
		\begin{split}
			p_t =& \norm{\rmi{i} -\rmti{i}}  = \norm{f(\Wmi{i},\Xm)-\ymbii{i}{}- f(\Wmti{i},\Xmti{i}) + \ymwii{i}{}}\\
			\led{172}& \alphamax\norm{f(\Wmi{i},\Xm) -  f(\Wmti{i},\Xmti{i}) }\\
			\leq&  8t (1+\psi')\eta\Gamma\sqrt{n}\eps\Theta\beta=c_0t(1+\psi')\eta\eps \Gamma^3n^{3/2}\sqrt{\log K}.
		\end{split}
	\end{equation*}
	where \ding{172} uses the definition of  $\ymbii{i}{}= (1-\alphai{i})  \ym + \alphai{i} f(\Wmi{i},\Xm)$
	and $\ymwii{i}=(1-\alphai{i})  \ym +\alphai{i}   f(\Wmti{i},\Xmt)$. In this way, we can obtain
	\begin{equation*}
		\begin{split}
			\norm{f(\Wmi{i},\Xm) -  f(\Wmti{i},\Xmti{i}) } \leq c_0t(1+\psi')\eta\eps \Gamma^3n^{3/2}\sqrt{\log K}.
		\end{split}
	\end{equation*}
	
	Next, note that, condition on $L$ is implied by
	\begin{align}
		k&\geq 1000\Gamma^2n(t_0\eta\beta)^4\Theta^2/\alphamax^2\\
		&=\order{\Gamma^4n{\frac{K^4}{\alphamax^2 n^4\lambda(\Cm)^4}}\log(\frac{\Gamma\sqrt{n\log K}}{\rho})^4(\|\Cm\|\Gamma\sqrt{n/K})^4(\Gamma \sqrt{n\log K})^2}\\
		&=\order{\Gamma^{10}{\frac{K^2\|\Cm\|^4}{\alphamax^2\lambda(\Cm)^4}}\log(\frac{\Gamma\sqrt{n\log K}}{\rho})^4\log^2(K)}
	\end{align}
	which is implied by $k\geq \order{\Gamma^{10}{\frac{K^2\|\Cm\|^4}{\alphamax^2 \lambda(\Cm)^4}}\log(\frac{\Gamma\sqrt{n\log K}}{\rho})^6}$.
	
	Finally, following \eqref{induct induct}, distance satisfies
	\[
	d_t\leq 20t \psi'\eta^2t_0 \Gamma\sqrt{n}\eps\Theta\beta^2\leq \order{t\psi'\eta\eps \frac{\Gamma^4Kn}{\lambda(\Cm)}\log(\frac{\Gamma\sqrt{n\log K}}{\rho})^2}.
	\]
	The proof is completed.
\end{proof}

\end{document}